\documentclass[wcp]{jmlr}

\usepackage{longtable}% for long tables

\usepackage{booktabs}

\usepackage{graphicx}%
\usepackage{multirow}%
\usepackage{amsmath,amssymb,amsfonts}%

\usepackage{amsthm}
\newtheorem{assumption}{Assumption}
\usepackage{mathrsfs}%
\usepackage{xcolor}%
\usepackage{textcomp}%
\usepackage{manyfoot}%
\usepackage{booktabs}%
\usepackage{algorithm}
\usepackage{algorithmic}
\usepackage{listings}%
\usepackage{hyperref}
\usepackage{bm}
\usepackage{subcaption}
\usepackage{placeins}
\usepackage{float}

\theoremstyle{thmstyleone}%
\newtheorem{mytheorem}{Theorem}[section]
\theoremstyle{thmstyletwo}%
\theoremstyle{thmstylethree}%
\newtheorem{myproposition}{Proposition}[section]% to get separate numbers for theorem and proposition etc.
\newtheorem{mylemma}{Lemma}[section]
\makeatletter
\let\Ginclude@graphics\@org@Ginclude@graphics 
\makeatother

\jmlrvolume{}
\jmlryear{2025}
\jmlrworkshop{ACML 2025}

\title[SGDM with Increasing Batch Size]{Increasing Batch Size Improves Convergence of Stochastic Gradient Descent with Momentum}

\author{
  \Name{Keisuke Kamo}\thanks{Equal contribution} \Email{ce245016@meiji.ac.jp}\and
  \Name{Hideaki Iiduka}\footnotemark[1] \Email{iiduka@cs.meiji.ac.jp}\\
  \addr Department of Computer Science, Meiji University, Japan
}

\editors{Hung-yi Lee and Tongliang Liu}

\begin{document}

\maketitle

\begin{abstract}
Stochastic gradient descent with momentum (SGDM), in which a momentum term is added to SGD, has been well studied in both theory and practice. The theoretical studies show that the settings of the learning rate and momentum weight affect the convergence of SGDM. Meanwhile, the practical studies have shown that the batch-size setting strongly affects the performance of SGDM. In this paper, we focus on mini-batch SGDM with a constant learning rate and constant momentum weight, which is frequently used to train deep neural networks. We show theoretically that using a constant batch size does not always minimize the expectation of the full gradient norm of the empirical loss in training a deep neural network, whereas using an increasing batch size definitely minimizes it; that is, an increasing batch size improves the convergence of mini-batch SGDM. We also provide numerical results supporting our analyses, indicating specifically that mini-batch SGDM with an increasing batch size converges to stationary points faster than with a constant batch size, while also reducing computational cost. Python implementations of the optimizers used in the numerical experiments are available at \url{https://github.com/iiduka-researches/NSHB_increasing_batchsize_acml25/}.
\end{abstract}

\begin{keywords}
convergence rate; increasing batch size; nonconvex optimization; stochastic gradient descent with momentum
\end{keywords}

\section{Introduction}
\label{sec1}
Stochastic gradient descent (SGD) and its variants, such as SGD with momentum (SGDM) and adaptive methods, are useful optimizers for minimizing the empirical loss defined by the mean of nonconvex loss functions in training a deep neural network (DNN). In the present paper, we focus on SGDM optimizers, in which a momentum term is added to SGD. Various types of SGDM have been proposed, including stochastic heavy ball (SHB) \citep{polyak1964}, normalized-SHB (NSHB) \citep{nshb}, Nesterov's accelerated gradient method \citep{nest1983,sut2013}, synthesized Nesterov variants \citep{doi:10.1137/15M1009597}, Triple Momentum \citep{7967721}, Robust Momentum \citep{8430824}, PID control-based methods \citep{8578987}, stochastic unified momentum (SUM) \citep{ijcai2018p410}, accelerated SGD \citep{DBLP:conf/colt/0002KKNS18,DBLP:conf/iclr/KidambiN0K18,pmlr-v178-varre22a,li2024risk}, quasi-hyperbolic momentum (QHM) \citep{ma2018quasihyperbolic}, and proximal-type SHB (PSHB) \citep{pmlr-v119-mai20b}.

Since the empirical loss is nonconvex with respect to a parameter $\bm{\theta} \in \mathbb{R}^d$ of {a} DNN, we are interested in nonconvex optimization for SGDM. Let $\bm{\theta}_t \in \mathbb{R}^d$ be the $t$-th approximation of SGDM to minimize the nonconvex empirical loss function $f \colon \mathbb{R}^d \to \mathbb{R}$. SGDM is defined as $\bm{\theta}_{t+1} = \bm{\theta}_t - \eta_t \bm{m}_t$, where $\eta_t > 0$ is a learning rate and $\bm{m}_t$ is a momentum buffer. For example, SGDM with $\bm{m}_t := \beta \bm{m}_{t-1} + \nabla f_{B_t}(\bm{\theta}_t)$ is SHB, where $\nabla f_{B_t} \colon \mathbb{R}^d \to \mathbb{R}^d$ denotes the stochastic gradient of $f$ and $\beta \in [0,1)$ is a momentum weight. SGDM with $\bm{m}_t := \beta \bm{m}_{t-1} + (1 - \beta) \nabla f_{B_t}(\bm{\theta}_t)$ is NSHB. Since SHB with $\beta = 0$ (NSHB with $\beta = 0$) coincides with SGD, which defined by $\bm{\theta}_{t+1} = \bm{\theta}_t - \eta_t \nabla f_{B_t}(\bm{\theta}_t)$, SGDM is defined as SGD with an added momentum term (e.g., $\beta \bm{m}_{t-1}$ in the case of SHB).

\begin{table*}[h]
\tiny
\caption{Convergence of SGDM optimizers to minimize $L$-smooth $f$ over the number of steps $T$. ``Noise" in the Gradient column means that the optimizer uses noisy observations, i.e., $\bm{g}(\bm{\theta}) = \nabla f(\bm{\theta}) + \text{(Noise)}$, of the full gradient $\nabla f(\bm{\theta})$, where $\sigma^2$ is an upper bound of Noise, while ``Increasing (resp. Constant) Mini-batch" in the Gradient column means that the optimizer uses a mini-batch gradient $\nabla f_{B_t} (\bm{\theta}) = \frac{1}{b_t} \sum_{i=1}^{b_t} \nabla f_{\xi_{t,i}} (\bm{\theta})$ with a batch size $b_t$ such that $b_t \leq b_{t+1}$ (resp. $b_t = b$). ``Bounded Gradient" in the Additional Assumption column means that there exists $G > 0$ such that, for all $t \in \mathbb{N}$, $\|\nabla f (\bm{\theta}_t)\| \leq G$, where $(\bm{\theta}_t)_{t=0}^{T-1}$ is the sequence generated by the optimizer. ``Polyak-\L ojasiewicz" in the Additional Assumption column means that there exists $\rho > 0$ such that, for all $t \in \mathbb{N}$, $\|\nabla f (\bm{\theta}_t)\|^2 \geq 2 \rho (f(\bm{\theta}_t) - f^\star)$, where $f^\star$ is the optimal value of $f$ over $\mathbb{R}^d$. Here, we let $\mathbb{E}\|\nabla f_T \| := \min_{t\in [0:T-1]} \mathbb{E}[\|\nabla f (\bm{\theta}_{t})\|]$. Results (1)--(7) were presented in (1) \citep[Theorem 1]{ijcai2018p410}, (2) \citep[Theorem 1]{NEURIPS2019_4eff0720}, (3) \citep[Theorem 2]{NEURIPS2019_4eff0720}, (4) \citep[Theorem 1]{pmlr-v119-mai20b}, (5) \citep[Corollary 1]{pmlr-v97-yu19d}, (6) \citep[Theorem 1]{NEURIPS2020_d3f5d4de}, and (7) \citep[Theorem 4.1]{DBLP:journals/nn/LiangLX23}.}
\label{table:1}
\centering
\begin{tabular}{llllll}
\toprule
Optimizer 
& Gradient 
& Additional Assumption 
& Learning Rate $\eta_t$ 
& Weight $\beta_t$ 
& Convergence Analysis \\
\midrule
(1) SUM 
& Noise 
& Bounded Gradient 
& $\eta = O(\frac{1}{\sqrt{T}})$                 
& $\beta_t = \beta$                   
& $\mathbb{E}\|\nabla f_T \| = O(\frac{1}{T^{1/4}})$\\
\midrule
(2) QHM 
& Noise
& Bounded Gradient 
& $\eta_t \to 0$                  
& $\beta_t \to 0$                  
& $\exists (\bm{\theta}_{t_i}): \nabla f(\bm{\theta}_{t_i}) \to 0$\\
\midrule
(3) QHM 
& Noise 
& Bounded Gradient 
& $\eta_t \to 0$                  
& $\beta_t \to 1$                  
& $\exists (\bm{\theta}_{t_i}): \nabla f(\bm{\theta}_{t_i}) \to 0$\\
\midrule
(4) PSHB 
& Noise 
& Bounded Gradient 
& $\eta = O(\frac{1}{\sqrt{T}})$                  
& $\beta_t = \beta$                  
& $\mathbb{E}\|\nabla f_T \| = O(\frac{1}{T^{1/4}})$\\
\midrule
(5) SHB 
& Noise 
& ---------                  
& $\eta = O(\frac{1}{\sqrt{T}})$                  
& $\beta_t = \beta$                  
& $\mathbb{E}\|\nabla f_T \| = O(\frac{1}{T^{1/4}})$\\
\midrule
(6) NSHB 
& Noise 
& ---------                  
& $\eta = O(\frac{1}{L})$                  
& $\beta_t = \beta$                  
& $\mathbb{E}\|\nabla f_T \| = O(\sqrt{\frac{1}{T} + \sigma^2})$\\
\midrule
(7) SUM 
& Noise 
& Polyak-\L ojasiewicz 
& $\eta_t \to 0$                  
& $\beta_t = \beta$                  
& $\mathbb{E}[f(\bm{\theta}_t)] \to f^\star$\\
\midrule
\textbf{NSHB}
& Constant
& ---------
& $\eta = O(\frac{1}{L})$ 
& $\beta_t = \beta$ 
& $\mathbb{E}\|\nabla f_T \| = O(\sqrt{\frac{1}{T} + \frac{\sigma^2}{b}})$\\
\textbf{[Theorem \ref{thm:1_1}]}
& Mini-batch 
&  
&  
& 
& \\
\midrule
\textbf{SHB}
& Constant 
& ---------
& $\eta = O(\frac{1}{L})$ 
& $\beta_t = \beta$ 
& $\mathbb{E}\|\nabla f_T \| = O(\sqrt{\frac{1}{T} + \frac{\sigma^2}{b}})$\\
\textbf{[Theorem \ref{thm:2_1}]]}
& Mini-batch 
&  
&  
& 
& \\
\midrule
\textbf{NSHB}
& Increasing 
& ---------
& $\eta = O(\frac{1}{L})$ 
& $\beta_t = \beta$ 
& $\mathbb{E}\|\nabla f_T \| = O(\frac{1}{T^{1/2}})$\\
\textbf{[Theorem \ref{thm:1}]}
& Mini-batch 
&  
&  
& 
& \\
\midrule
\textbf{SHB}
& Increasing 
& ---------
& $\eta = O(\frac{1}{L})$ 
& $\beta_t = \beta$ 
& $\mathbb{E}\|\nabla f_T \| = O(\frac{1}{T^{1/2}})$\\
\textbf{[Theorem \ref{thm:2}]]}
& Mini-batch 
&  
&  
& 
& \\
\bottomrule
\end{tabular}
\end{table*}

Table \ref{table:1} summarizes the convergence analyses of SGDM for nonconvex optimization. For example, NSHB ((6) in Table \ref{table:1}) using a constant learning rate $\eta_t = \eta > 0$ and a constant momentum weight $\beta_t = \beta$ satisfies $\min_{t \in [0:T-1]} \mathbb{E}[\|\nabla f(\bm{\theta}_t)\|]= O(\sqrt{\frac{1}{T} + \sigma^2})$ \citep[Theorem 1]{NEURIPS2020_d3f5d4de}, where $T$ is the number of steps, $\nabla f \colon \mathbb{R}^d \to \mathbb{R}^d$ is the gradient of $f$, $\sigma^2$ is the upper bound of the variance of the stochastic gradient of $f$, and $\mathbb{E}[X]$ denotes the expectation of a random variable $X$. In comparison, QHM ((2) in Table \ref{table:1}), which is a generalization of NSHB, using a decaying learning rate $\eta_t$ and a decaying momentum weight $\beta_t$ satisfies $\liminf_{t \to + \infty} \|\nabla f(\bm{\theta}_t) \| = 0$ \citep[Theorem 1]{NEURIPS2019_4eff0720}. As can be seen from these convergence analysis results, the performance of SGDM in finding a stationary point $\bm{\theta}^\star$ of $f$ (i.e., $\nabla f(\bm{\theta}^\star) = \bm{0}$) depends on the settings of the learning rate $\eta_t$ and the momentum weight $\beta_t$.

Moreover, we would like to emphasize that the setting of the batch size $b_t$ affects the performance of SGDM. Previous results presented in \citep{shallue2019,zhang2019} numerically showed that, for deep learning optimizers, the number of steps needed to train a DNN is halved for each doubling of the batch size. In \citep{l.2018dont}, it was numerically shown that using an enormous batch size leads to a reduction in the number of parameter updates and the model training time. In addition, related research has explored plain SGD, highlighting the broader interest in batch-size strategies beyond SGDM. For instance, the two-scale adaptive (TSA) method \citep{gao2022tsa} co-adapts the batch size and step size to obtain exact convergence and favorable sample complexity. The related work has also investigated batch-size growth for Riemannian SGD \citep{sakai2025a}. While these studies share the idea of using an increasing batch size, their analyses are limited to plain SGD. Therefore, in this paper, we theoretically investigate how the setting of the batch size affects the convergence of SGDM under constant hyperparameters.

\subsection{Contribution}
In this paper, we focus on mini-batch SGDM with a constant learning rate $\eta > 0$ and constant momentum weight $\beta \in [0,1)$, which is frequently used to train {DNNs}. 
\begin{enumerate}
\item[1.] The first theoretical contribution of the paper is to show that an upper bound of $\min_{t\in [0:T-1]} \mathbb{E}[\|\nabla f (\bm{\theta}_{t})\|]$ for mini-batch SGDM using a {\em constant} batch size $b$ is 
    \begin{align*}
        O \left( \sqrt{\frac{f(\bm{\theta}_0) - f^\star}{\eta T} + \frac{L \eta \sigma^2}{b}} \right),
    \end{align*}
which implies that mini-batch SGDM does not always minimize the expectation of the full gradient norm of the empirical loss in training {a DNN} (Table \ref{table:1}; Theorems \ref{thm:1_1} and \ref{thm:2_1}).
\end{enumerate}
The bias term $\frac{f(\bm{\theta}_0) - f^\star}{\eta T}$ converges to $0$ when $T \to + \infty$. However, the variance term $\frac{L \eta \sigma^2}{b}$ remains a constant positive real number regardless of how large $T$ is. In contrast, using a large batch size $b$ makes the variance term $\frac{L \eta \sigma^2}{b}$ small. Hence, we can expect that the upper bound of $\min_{t\in [0:T-1]} \mathbb{E}[\|\nabla f (\bm{\theta}_{t})\|]$ for mini-batch SGDM with {\em increasing} batch size converges to $0$.

\begin{enumerate}
\item[2.] The second theoretical contribution is to show that an upper bound of $\min_{t\in [0:T-1]} \mathbb{E}[\|\nabla f (\bm{\theta}_{t})\|]$ for mini-batch SGDM with {\em increasing} batch size $b_t$ such that $b_t$ is multiplied by $\delta > 1$ every $E$ epochs is 
    \begin{align}\label{contribution_1}
        O \left( \sqrt{\frac{f(\bm{\theta}_0) - f^\star}{\eta T} + \frac{L \eta \sigma^2 \delta}{(\beta^2 \delta - 1) b_0 T}} \right),
    \end{align}
which implies that mini-batch SGDM minimizes the expectation of the full gradient norm of the empirical loss in the sense of an $O(\frac{1}{\sqrt{T}})$ rate of convergence (Table \ref{table:1}; Theorems \ref{thm:1} and \ref{thm:2}).
\end{enumerate}
The previous results reported in \citep{Byrd:2012aa,balles2016coupling,pmlr-v54-de17a,l.2018dont,goyal2018accuratelargeminibatchsgd,shallue2019,zhang2019} indicated that increasing batch sizes are useful for training DNNs with deep-learning optimizers. However, the existing analyses of SGDM have indicated that knowing only the theoretical performance of mini-batch SGDM with an increasing batch size may be insufficient (Table \ref{table:1}). The paper shows theoretically that SGDM with an increasing batch size converges to stationary points of the empirical loss (Theorems \ref{thm:1} and \ref{thm:2}). The previous results in \citep[Theorem 1]{ijcai2018p410}, \citep[Theorem 1]{pmlr-v119-mai20b}, and \citep[Corollary 1]{pmlr-v97-yu19d} (Table \ref{table:1}(1), (4), and (5)) showed that SGDM with a constant learning rate $\eta = O(\frac{1}{\sqrt{T}})$ and a constant momentum weight $\beta$ has convergence rate $O(\frac{1}{T^{1/4}})$. Our results (Theorems \ref{thm:1} and \ref{thm:2}) guarantee that, if the batch size increases, then SGDM satisfies $\min_{t\in [0:T-1]} \mathbb{E}[\|\nabla f (\bm{\theta}_{t})\|] = O(\frac{1}{T^{1/2}})$, which is an improvement on the previous convergence rate $O(\frac{1}{T^{1/4}})$.

The result in \eqref{contribution_1} indicates that the performance of mini-batch SGDM with an increasing batch size $b_t$ depends on $\delta$. Let $\eta$ and $\beta$ be fixed (e.g., $\eta = 0.1$ and $\beta = 0.9$). Then, \eqref{contribution_1} indicates that the larger $\delta$ is, the smaller the variance term $\frac{L \eta \sigma^2 \delta}{(\beta^2 \delta - 1) b_0 T}$ becomes (since $\frac{\delta}{\beta^2 \delta - 1} = \frac{1}{(0.9)^2 - 1/\delta}$ becomes small as $\delta$ becomes large). We are interested in verifying whether this theoretical result holds in practice. Our numerical results in Section~\ref{sec:4} support the theoretical findings, suggesting that an increasing batch size leads to faster convergence.

We trained {ResNet-18 on the CIFAR-100 and Tiny-ImageNet datasets} by using not only NSHB and SHB but also baseline optimizers: SGD, Adam \citep{adam}, AdamW \citep{loshchilov2018decoupled}, and RMSprop \citep{rmsprop}. The results on Tiny-ImageNet are provided in Appendix \ref{appendix:A.4}. A particularly interesting result in Section \ref{sec:4} is that an increasing batch size benefits Adam in the sense of minimizing $\min_{t\in [0:T-1]} \|\nabla f (\bm{\theta}_{t})\|$ fastest. Hence, in the future, we would like to verify whether Adam with an increasing batch size theoretically {has a better convergence rate than} SGDM.

\begin{enumerate}
\item[3.] The third contribution of this paper is to show that an increasing batch size significantly reduces the \emph{computational cost} required to achieve optimization and generalization criteria. 
\end{enumerate}
In this paper, we define the computational cost as the stochastic first-order oracle (SFO) complexity, in accordance with prior work that formalized this concept \citep{pmlr-v202-sato23b, Imaizumi19062024}, where SFO is interpreted as the total number of stochastic gradient evaluations. This perspective resonates with the extensive empirical investigation made by Shallue et al. \citep{shallue2019}, which systematically examined how the batch size affects training efficiency across diverse neural network settings. Let $T$ be the number of training steps and $b$ the batch size. Then, the total SFO complexity is given by $Tb$ when using a fixed batch size. Our numerical results in Section \ref{sec:4.2} demonstrate that, under realistic GPU memory constraints, SGDM with an increasing batch size requires significantly fewer gradient computations to achieve competitive optimization and generalization performance compared with using a fixed batch size. These results indicate that an increasing batch size is not only theoretically appealing but also practically effective in reducing the computational burden in deep-learning training.

\section{Mini-batch SGDM for Empirical Risk Minimization}
\label{sec:2}
\subsection{Empirical risk minimization} 
Let $\bm{\theta} \in \mathbb{R}^d$ be a parameter of a {DNN}, where $\mathbb{R}^d$ is $d$-dimensional Euclidean space with inner product $\langle \cdot, \cdot \rangle$ and induced norm $\|\cdot \|$. Let $\mathbb{R}_+ := \{ x \in \mathbb{R} \colon x \geq 0 \}$ and $\mathbb{R}_{++} := \{ x \in \mathbb{R} \colon x > 0 \}$. Let $\mathbb{N}$ be the set of natural numbers. Let $S = \{(\bm{x}_1,\bm{y}_1), \ldots, (\bm{x}_n,\bm{y}_n)\}$ be the training set, where the data point $\bm{x}_i$ is associated with a label $\bm{y}_i$ and $n\in \mathbb{N}$ is the number of training samples. Let $f_i (\cdot) := f(\cdot;(\bm{x}_i,\bm{y}_i)) \colon \mathbb{R}^d \to \mathbb{R}_+$ be the loss function corresponding to the $i$-th labeled training data $(\bm{x}_i,\bm{y}_i)$. Empirical risk minimization (ERM) minimizes the empirical loss defined for all $\bm{\theta} \in \mathbb{R}^d$ as $f (\bm{\theta}) = \frac{1}{n} \sum_{i\in [n]} f(\bm{\theta};(\bm{x}_i,\bm{y}_i)) = \frac{1}{n} \sum_{i\in [n]} f_i(\bm{\theta})$, where $[n] := \{1, 2, \cdots, n\}$.

We assume that the loss functions $f_i$ ($i\in [n]$) satisfy the following conditions.

\begin{assumption}\label{assum:1}
Let $n$ be the number of training samples and let $L_i > 0$ ($i\in [n]$).

(A1) $f_i \colon \mathbb{R}^d \to \mathbb{R}$ ($i\in [n]$) is differentiable and $L_i$-smooth (i.e., there exists $L_i > 0$ such that, for all $\bm{\theta}_1, \bm{\theta}_2 \in \mathbb{R}^d$, $\|\nabla f_i (\bm{\theta}_1) - \nabla f_i (\bm{\theta}_2) \|\leq L_i \|\bm{\theta}_1 - \bm{\theta}_2\|$), $L := \frac{1}{n} \sum_{i\in [n]} L_i$, and $f^\star$ is the minimum value of $f$ over $\mathbb{R}^d$.

(A2) Let $\xi$ be a random variable independent of $\bm{\theta} \in \mathbb{R}^d$. $\nabla f_{\xi} \colon \mathbb{R}^d \to \mathbb{R}^d$ is the stochastic gradient of $\nabla f$ such that (i) for all $\bm{\theta} \in \mathbb{R}^d$, $\mathbb{E}_{\xi}[\nabla f_{\xi}(\bm{\theta})] = \nabla f(\bm{\theta})$ and (ii) there exists $\sigma \geq 0$ such that, for all $\bm{\theta} \in \mathbb{R}^d$, $\mathbb{V}_{\xi}[\nabla f_{\xi}(\bm{\theta})] = \mathbb{E}_{\xi}[\| \nabla f_{\xi}(\bm{\theta}) - \nabla f(\bm{\theta})\|^2] \leq \sigma^2$, where $\mathbb{E}_\xi[\cdot]$ denotes the expectation with respect to $\xi$. 

(A3) Let $b \in \mathbb{N}$ such that $b \leq n$ and let $\bm{\xi} = (\xi_{1}, \xi_{2}, \cdots, \xi_{b})^\top$ comprise $b$ independent and identically distributed variables that are independent of $\bm{\theta} \in \mathbb{R}^d$. The full gradient $\nabla f (\bm{\theta})$ is taken to be the mini-batch gradient at $\bm{\theta}$ defined by $\nabla f_{B}(\bm{\theta}) := \frac{1}{b} \sum_{i=1}^b \nabla f_{\xi_{i}}(\bm{\theta})$.
\end{assumption}

\subsection{Mini-batch NSHB and mini-batch SHB}
\label{sec:2.2}
Let $\bm{\theta}_t \in \mathbb{R}^d$ be the $t$-th approximated parameter of {DNN}. Then, mini-batch NSHB uses $b_t$ loss functions $f_{\xi_{t,1}}, \cdots, f_{\xi_{t,b_t}}$ randomly chosen from $\{f_1,\cdots,f_n\}$ at each step $t$, where $\bm{\xi}_t = (\xi_{t,1}, \cdots, \xi_{t,b_t})^\top$ is independent of $\bm{\theta}_t$ and $b_t$ is a batch size satisfying $b_t \leq n$. The Mini-batch NSHB {optimizer} is listed in Algorithm \ref{algo:1}.

The simplest optimizer for adding a momentum term (denoted by $\beta \bm{m}_{t-1}$) to SGD is the stochastic heavy ball (SHB) method \citep{polyak1964}, which is provided in PyTorch \citep{NEURIPS2019_bdbca288}. SHB is defined as follows:
\begin{align}\label{shb}
\begin{split}
\bm{m}_t = \beta \bm{m}_{t-1} + \nabla f_{B_t}(\bm{\theta}_t), \text{ } \bm{\theta}_{t+1} = \bm{\theta}_t - \alpha \bm{m}_t,
\end{split}
\end{align}
where $\beta \in [0,1)$ and $\alpha > 0$. SHB defined by \eqref{shb} with $\beta = 0$ coincides with SGD. SHB defined by \eqref{shb} has the form $\bm{\theta}_{t+1} = \bm{\theta}_t - \alpha \nabla f_{B_t}(\bm{\theta}_t) + \beta (\bm{\theta}_t - \bm{\theta}_{t-1})$. Meanwhile, Algorithm \ref{algo:1} is called the normalized-SHB (NSHB) {optimizer} \citep{nshb} and has the form $\bm{\theta}_{t+1} = \bm{\theta}_t - \eta (1 - \beta) \nabla f_{B_t}(\bm{\theta}_t) + \beta (\bm{\theta}_t - \bm{\theta}_{t-1})$. Hence, NSHB (Algorithm \ref{algo:1}) with $\eta = \frac{\alpha}{1-\beta}$ coincides with SHB defined by \eqref{shb}. 

\begin{algorithm}[htb]
\caption{Mini-batch NSHB optimizer}
\label{algo:1}
\begin{algorithmic}[1]
\REQUIRE
$\bm{\theta}_0, \bm{m}_{-1} := \bm{0}$ (initial point), 
$b_t > 0$ (batch size), 
$\eta > 0$ (learning rate),
$\beta \in [0,1)$ (momentum weight), 
$T \geq 1$ (steps)
\ENSURE 
$(\bm{\theta}_t) \subset \mathbb{R}^d$
\FOR{$t=0,1,\ldots,T-1$}
\STATE{ 
$\nabla f_{B_t}(\bm{\theta}_t)
:=
\frac{1}{b_t} \sum_{i=1}^{b_t} \nabla f_{\xi_{t,i}}(\bm{\theta}_t)$}
\STATE{
$\bm{m}_t := \beta \bm{m}_{t-1} + (1 - \beta) \nabla f_{B_t}(\bm{\theta}_t)$
}
\STATE{
$\bm{\theta}_{t+1} 
:= \bm{\theta}_t - \eta \bm{m}_t$}
\ENDFOR
\end{algorithmic}
\end{algorithm}

\section{Mini-batch SGDM with Constant and Increasing Batch Sizes}
\label{sec:3}
\subsection{Constant batch size scheduler}
\label{sec:3.1}
The following indicates that an upper bound of $\min_{t \in [0:T-1]} \mathbb{E}[\| \nabla f(\bm{\theta}_t) \|]$ of mini-batch NSHB using a constant batch size 
\begin{align}\label{constant_bs}
\text{[Constant BS] }
b_t = b \text{ } (t \in \mathbb{N})
\end{align}
does not always converge to $0$ {(a proof of Theorem \ref{thm:1_1} is given in Appendix \ref{appendix:thm_1_1}).}

\begin{mytheorem}
[Upper bound of $\min_{t} \mathbb{E}\| \nabla f(\bm{\theta}_t) \|$ of mini-batch NSHB with Constant BS]
\label{thm:1_1}
Suppose that Assumption \ref{assum:1} holds and consider the sequence $(\bm{\theta}_t)$ generated by Algorithm \ref{algo:1} with a momentum weight $\beta \in (0,1)$, a constant learning rate $\eta > 0$ such that $\eta \leq \frac{1-\beta}{2\sqrt{2} \sqrt{\beta+\beta^2} L}$, and Constant BS defined by \eqref{constant_bs}, where $L := \frac{1}{n} \sum_{i\in [n]} L_i$ and $f^\star$ is the minimum value of $f$ over $\mathbb{R}^d$ (see (A1)). Then, for all $T \geq 1$,
\begin{align*}
\min_{t \in [0:T-1]} \mathbb{E} [\| \nabla f(\bm{\theta}_t) \|^2 ]
\leq
\frac{2 (f(\bm{\theta}_0) - f^\star)}{\eta T}
+
\frac{L\eta \sigma^2}{b}\left\{  
\frac{3 \beta^2 + \beta}{2(1 + \beta)} + 1
\right\},
\end{align*} 
that is,
\begin{align*}
\min_{t \in [0:T-1]} \mathbb{E} [\| \nabla f(\bm{\theta}_t) \|]
= O \left(\sqrt{ \frac{1}{T} + \frac{\sigma^2}{b}  } \right).
\end{align*}
\end{mytheorem}

From the discussion in Section \ref{sec:2.2}, we find that NSHB (Algorithm \ref{algo:1}) with $\eta = \frac{\alpha}{1-\beta}$ coincides with SHB defined by \eqref{shb}. The theoretical results for SHB can also be derived accordingly. Appendices \ref{appendix:thm_1} and \ref{appendix:thm_1_1} contain a detailed analysis of SHB.

\subsection{Increasing batch size scheduler}
\label{sec:3.2}
We consider an increasing batch size $b_t$ such that 
\begin{align*}
b_t \leq b_{t+1} \text{ } (t \in \mathbb{N}). 
\end{align*}
An example of $b_t$ \citep{l.2018dont,umeda2024increasingbatchsizelearning} is, for all $m \in [0:M]$ and all $t \in S_m = \mathbb{N} \cap [\sum_{k=0}^{m-1} K_{k} E_{k}, \sum_{k=0}^{m} K_{k} E_{k})$ ($S_0 := \mathbb{N} \cap [0, K_0 E_0)$), 
\begin{align}\label{exponential_bs}
\text{[Exponential Growth BS] }
b_t =
\delta^{m \left\lceil \frac{t}{\sum_{k=0}^{m} K_{k} E_{k}} \right\rceil} b_0,
\end{align}
where $\delta > 1$, and $E_m$ and $K_m$ are the numbers of, respectively, epochs and steps per epoch when the batch size is $\delta^m b_0$. For example, the exponential growth batch size defined by (\ref{exponential_bs}) with $\delta = 2$ makes the batch size double every $E_m$ epochs. We may modify the parameters $a$ and $\delta$ to $a_t$ and $\delta_t$ monotone increasing with $t$. The total number of steps for the batch size to increase $M$ times is $T = \sum_{m=0}^M K_m E_m$.

The following is a convergence analysis of Algorithm \ref{algo:1} with increasing batch sizes.

\begin{mytheorem}
[Convergence of mini-batch NSHB with Exponential Growth BS]
\label{thm:1}
Suppose that Assumption \ref{assum:1} holds and consider the sequence $(\bm{\theta}_t)$ generated by Algorithm \ref{algo:1} with a momentum weight $\beta \in (0,1)$, a constant learning rate $\eta > 0$ such that 
\begin{align}\label{eta_condition}
\eta 
\leq \max \left\{ \frac{1-\beta}{2\sqrt{2}\sqrt{\beta+\beta^2}L},
\frac{(1-\beta)^2}{(5\beta^2 - 6\beta + 5)L} \right\},
\end{align}
and Exponential Growth BS as defined by \eqref{exponential_bs} with $\delta > 1$ and $\beta^2 \delta > 1$. Then, for all $T \geq 1$,
\begin{align*}
\min_{t \in [0:T-1]} \mathbb{E} [\| \nabla f(\bm{\theta}_t) \|^2 ]
\leq
\frac{2 (f(\bm{\theta}_0) - f^\star)}{\eta T} 
+
\frac{2 L \eta \sigma^2 K_{\max} E_{\max} \delta}{(\beta^2 \delta -1) b_0 T} \left( \frac{\beta^2}{1 - \beta^2} 
- \frac{1}{\delta - 1} \right),
\end{align*}
where $K_{\max} := \max \{ K_m \colon m \in [0:M]\}$ and $E_{\max} := \max \{ E_m \colon m \in [0:M]\}$, that is,
\begin{align*}
\min_{t \in [0:T-1]} \mathbb{E} [\| \nabla f(\bm{\theta}_t) \|]
= O \left(\frac{1}{\sqrt{T}} \right).
\end{align*}
\end{mytheorem}

From the discussion in Section \ref{sec:2.2} indicating that NSHB (Algorithm \ref{algo:1}) with $\eta = \frac{\alpha}{1-\beta}$ coincides with SHB defined by \eqref{shb}, Theorem \ref{thm:1} leads to the following convergence rate of SHB defined by \eqref{shb} with an increasing batch size.

{Here, we sketch a proof of Theorem \ref{thm:1} (a detailed proof is given in Appendix \ref{appendix:thm_1}).}
\begin{enumerate}
\item First, we show that $\frac{\sigma^2}{b_t}$ is an upper bound of the variance of $\nabla f_{B_t}(\bm{\theta}_t)$ (Proposition \ref{prop:1}) and that $(1-\beta)^2 \sigma^2 \sum_{i=0}^t \frac{\beta^{2(t-i)}}{b_i}$ is an upper bound of the variance of $\bm{m}_t = (1-\beta) \sum_{i=0}^t \beta^{t-i} \nabla f_{B_i} (\bm{\theta}_i)$ (Lemma \ref{lem:4}) by using the idea underlying the proof of \citep[Lemma 1]{NEURIPS2020_d3f5d4de}.
\item Next, we show that an auxiliary point $\bm{z}_t = \frac{1}{1 - \beta} \bm{\theta}_t - \frac{\beta}{1 - \beta} \bm{\theta}_{t-1}$ ($t \geq 1$), which is used to analyze SGDM \citep{ijcai2018p410,pmlr-v97-yu19d,NEURIPS2020_d3f5d4de}, satisfies $\mathbb{E}_{\bm{\xi}_t} [f (\boldsymbol z_{t+1} )] \leq f (\boldsymbol z_t ) - \eta \underbrace{\mathbb{E}_{\bm{\xi}_t} [\langle \nabla f (\boldsymbol z_t), \nabla f_{B_t} (\bm{\theta}_t) \rangle ]}_{X_t} + \frac{L\eta^2}{2}\underbrace{\mathbb{E}_{\bm{\xi}_t} [\| \nabla f_{B_t} (\bm{\theta}_t) \| ^2 ]}_{Y_t}$ by using the descent lemma (see \eqref{eq:39}). Then, using the Cauchy--Schwarz inequality, Young's inequality, and the upper bound of the variance of $\bm{m}_t$ (Lemma \ref{lem:4}), we arrive at an an upper bound on $- \eta \mathbb{E}[X_t]$ (see \eqref{x_t_2}). As well, we find an upper bound on $\mathbb{E}[Y_t]$ by using the upper bound $\frac{\sigma^2}{b_t}$ of the variance of $\nabla f_{B_t}(\bm{\theta}_t)$ (see Lemma \ref{prop:4} for details on the upper bounds of $- \eta \mathbb{E}[X_t]$ and $\mathbb{E}[Y_t]$). 
\item After that, we define the Lyapunov function $L_t$ by $L_t = f(\bm{z}_t) - f^\star + \sum_{i=1}^{t-1} c_i \|\bm{\theta}_{t+1-i} -\bm{\theta}_{t-i} \|^2$, where $c_i$ is defined as in Lemma \ref{prop:5}. Using the above upper bounds of $- \eta \mathbb{E}[X_t]$ and $\mathbb{E}[Y_t]$, we find that $\mathbb{E}[L_{t+1} - L_t] \leq - D \mathbb{E}[\|\nabla f (\bm{\theta}_t)\|^2] + U_t$ (Lemma \ref{prop:5}), where $D \in \mathbb{R}$ depends on $\eta$, $\beta$, and $c_1$, and $U_t > 0$ depends on $\sigma^2$, $b_t$, and $c_1$.
\item Then, setting $\eta$ such that it satisfies \eqref{eta_condition} (see also  Appendix \ref{appendix:A.7}) leads to the finding that $D \geq \frac{\eta}{2} > 0$ and $U_t \leq L \eta^2 \sigma^2 \sum_{i=0}^t \frac{\beta^{2(t-i)}}{b_i}$ (see \eqref{d_1} and \eqref{u_1}). As a result, we have that
    \begin{align*}
    \frac{1}{T} \sum_{t=0}^{T-1} \mathbb{E} [\| \nabla f(\bm{\theta}_t) \|^2 ]
    &\leq 
    \frac{2 L_{0}}{\eta T} 
    + \frac{2 L \eta \sigma^2}{T} \sum_{t=0}^{T-1} \sum_{i=0}^t \frac{\beta^{2(t-i)}}{b_i}
    \end{align*}
(see Lemma \ref{prop:6}). Finally, using\eqref{exponential_bs} leads to the assertion of Theorem \ref{thm:1}.
\end{enumerate}

\subsection{Setting of hyperparameter $\delta$ in Exponential Growth BS \eqref{exponential_bs}}
\label{sec:3.3}
Let $\eta$ and $\beta$ be fixed in Algorithm \ref{algo:1} (e.g., $\eta = 0.1$ and $\beta = 0.9$). Then, Theorems \ref{thm:1} and \ref{thm:2} indicate that $O ( \sqrt{\frac{f(\bm{\theta}_0) - f^\star}{\eta T} + \frac{L \eta \sigma^2 \delta}{(\beta^2 \delta - 1) b_0 T}} )$ is an upper bound of $\min_{t\in [0:T-1]} \mathbb{E}[\|\nabla f (\bm{\theta}_t)\|]$ for each of mini-batch NSHB and mini-batch SHB with exponential growth BS \eqref{exponential_bs}, which implies that the larger $\delta$ is, the smaller the variance term $\frac{L \eta \sigma^2 \delta}{(\beta^2 \delta - 1) b_0 T}$ becomes (since $\frac{\delta}{\beta^2 \delta - 1} = \frac{1}{(0.9)^2 - 1/\delta}$ becomes small as $\delta$ becomes large). In Section \ref{sec:4}, we verify whether this theoretical result holds in practice.

\subsection{Comparison of our convergence results with previous ones}
Let us compare Theorems \ref{thm:1_1}, \ref{thm:1}, \ref{thm:2_1} and \ref{thm:2} with the previous results listed in Table \ref{table:1}. Theorem 1 in \cite{NEURIPS2020_d3f5d4de} ((6) in Table \ref{table:1}) indicated that NSHB using a constant learning rate $\eta = O(\frac{1}{L})$ and a constant momentum weight $\beta$ satisfies $\min_{t \in [0:T-1]} \mathbb{E}[\|\nabla f(\bm{\theta}_t)\|]= O(\sqrt{\frac{1}{T} + \sigma^2})$. Since the upper bound $O(\sqrt{\frac{1}{T} + \sigma^2})$ converges to $O(\sigma) > 0$ when $T \to + \infty$, NSHB in this case does not always converge to stationary points of $f$. The result in \cite{NEURIPS2020_d3f5d4de} coincides with Theorem \ref{thm:1_1} indicating that NSHB has $\min_{t \in [0:T-1]} \mathbb{E}[\|\nabla f(\bm{\theta}_t)\|]= O(\sqrt{\frac{1}{T} + \frac{\sigma^2}{b}})$ in the sense that NSHB using a constant learning rate and momentum weight does not converge to stationary points of $f$\footnote{\citep[Theorem 3]{NEURIPS2020_d3f5d4de} {proved convergence of multistage SGDM. However, since the proof of \citep[(60), Pages 35 and 36]{NEURIPS2020_d3f5d4de} might not hold for $\beta_i < 1$, the theorem does not apply here.}}. Corollary 1 in \cite{pmlr-v97-yu19d} ((5) in Table \ref{table:1}) indicated that SHB using constant $\eta = O(\frac{1}{\sqrt{T}})$ and constant momentum weight $\beta$ satisfies $\min_{t \in [0:T-1]} \mathbb{E}[\|\nabla f(\bm{\theta}_t)\|]= O(\frac{1}{T^{1/4}})$. $\eta = O(\frac{1}{\sqrt{T}})$ is needed in order set the number of steps $T$ before implementing SHB. Since $T$ is fixed, we cannot diverge it; that is, the upper bound $O(\frac{1}{T^{1/4}})$ for SHB is a fixed positive constant and does not converge to $0$. Meanwhile, from Theorem \ref{thm:1_1}, it follows that SHB with a constant learning rate $\eta = O(\frac{1}{L})$ and a constant momentum weight also satisfies $\min_{t \in [0:T-1]} \mathbb{E}[\|\nabla f(\bm{\theta}_t)\|]= O(\sqrt{\frac{1}{T} + \frac{\sigma^2}{b}})$. Hence, Theorem \ref{thm:1_1} coincides with the result in Corollary 1 in \cite{pmlr-v97-yu19d} in the sense that SHB with a constant learning rate does not always converge to stationary points of $f$.

Theorems 1 and 2 in \cite{NEURIPS2019_4eff0720} ((2) and (3) in Table \ref{table:1}) indicated that QHM, which is a generalization of NSHB, using a decaying learning rate $\eta_t$ and a decaying momentum weight $\beta_t$ or an increasing momentum weight $\beta_t$ satisfies $\liminf_{t \to + \infty} \|\nabla f(\bm{\theta}_t) \| = 0$. Our results in the form of Theorems \ref{thm:1} and \ref{thm:2} guarantee the convergence of NSHB and SHB with constant learning rate $\eta = O(\frac{1}{L})$, constant momentum weight $\beta$, and an increasing batch size $b_t$ in the sense of $\min_{t \in [0:T-1]} \mathbb{E}[\|\nabla f(\bm{\theta}_t)\|]= O(\frac{1}{\sqrt{T}})$.

\section{Numerical Results}
\label{sec:4}
We examined training ResNet-18 on the CIFAR-100 and Tiny-ImageNet datasets using not only NSHB and SHB but also baseline optimizers: SGD, Adam, AdamW, and RMSprop with constant and increasing batch sizes. The results on Tiny-ImageNet are provided in Appendix \ref{appendix:A.4}. We used a computer equipped with NVIDIA A100 80GB and Dual Intel Xeon Silver 4316 2.30GHz, 40 Cores (20 cores per CPU, 2 CPUs). The software environment was Python 3.8.2, PyTorch 2.2.2+cu118, and CUDA 12.2. We set the total number of epochs to $E = 200$ and the constant momentum weight as the default values in PyTorch. The learning rate for Adam and AdamW was set to $10^{-3}$, for RMSprop to $10^{-2}$, and for SGD, SHB, and NSHB to $10^{-1}$; see also Figure \ref{fig_1}(a). We should note that the learning rates used in our experiments appear to be theoretically justified within the range implied by Theorem \ref{thm:1} (see Appendix \ref{appendix:A.7} for further details). All results are averaged over three independent trials, and the mean, maximum, and minimum at each epoch are shown.

\begin{figure}[H]
    \centering
    \begin{minipage}{0.4\textwidth}
        \centering
        \includegraphics[width=0.75\textwidth]{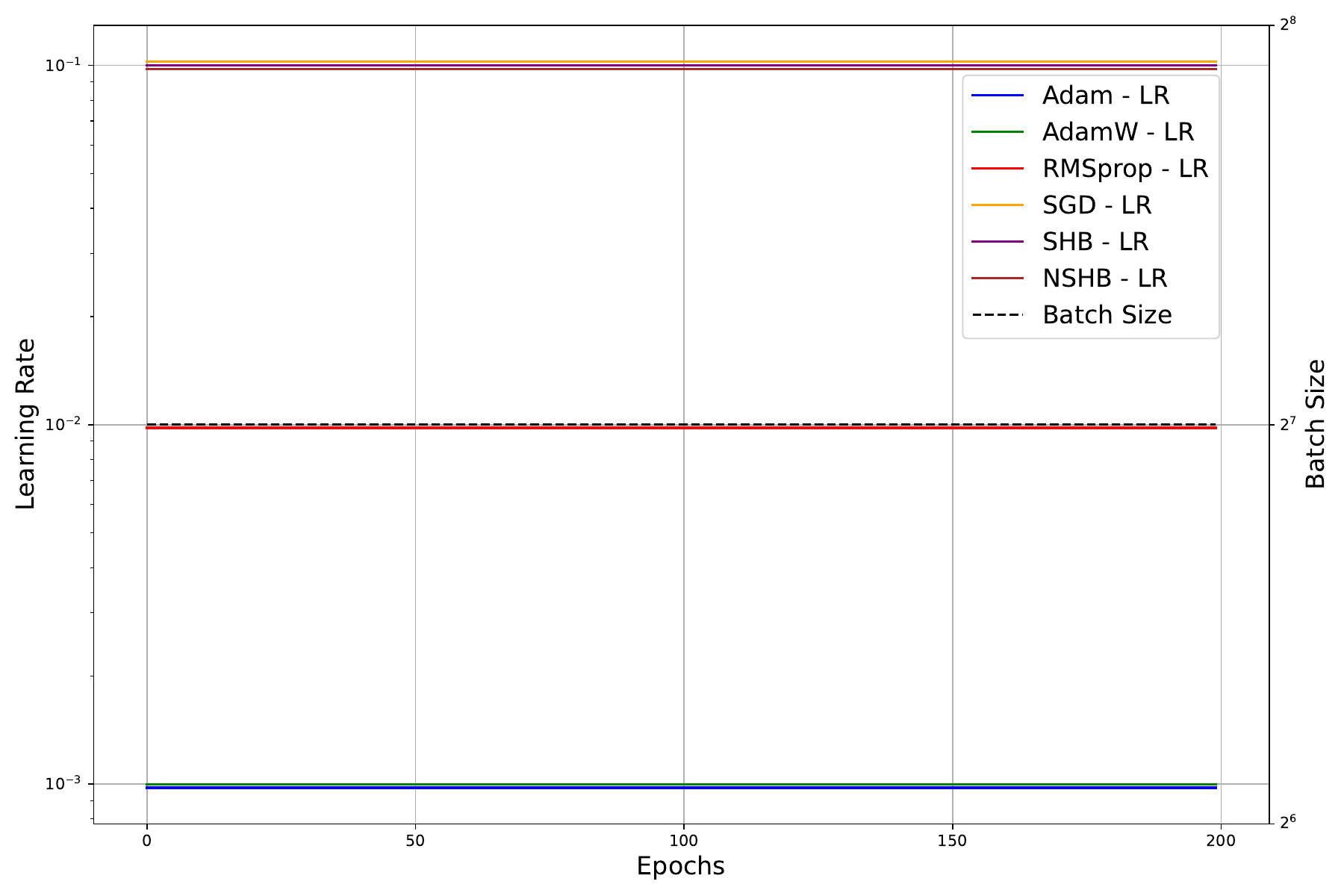} 
        \subfigure{Learning rate and batch size schedules}
    \end{minipage} \hfill
    \begin{minipage}{0.4\textwidth}
        \centering
        \includegraphics[width=0.75\textwidth]{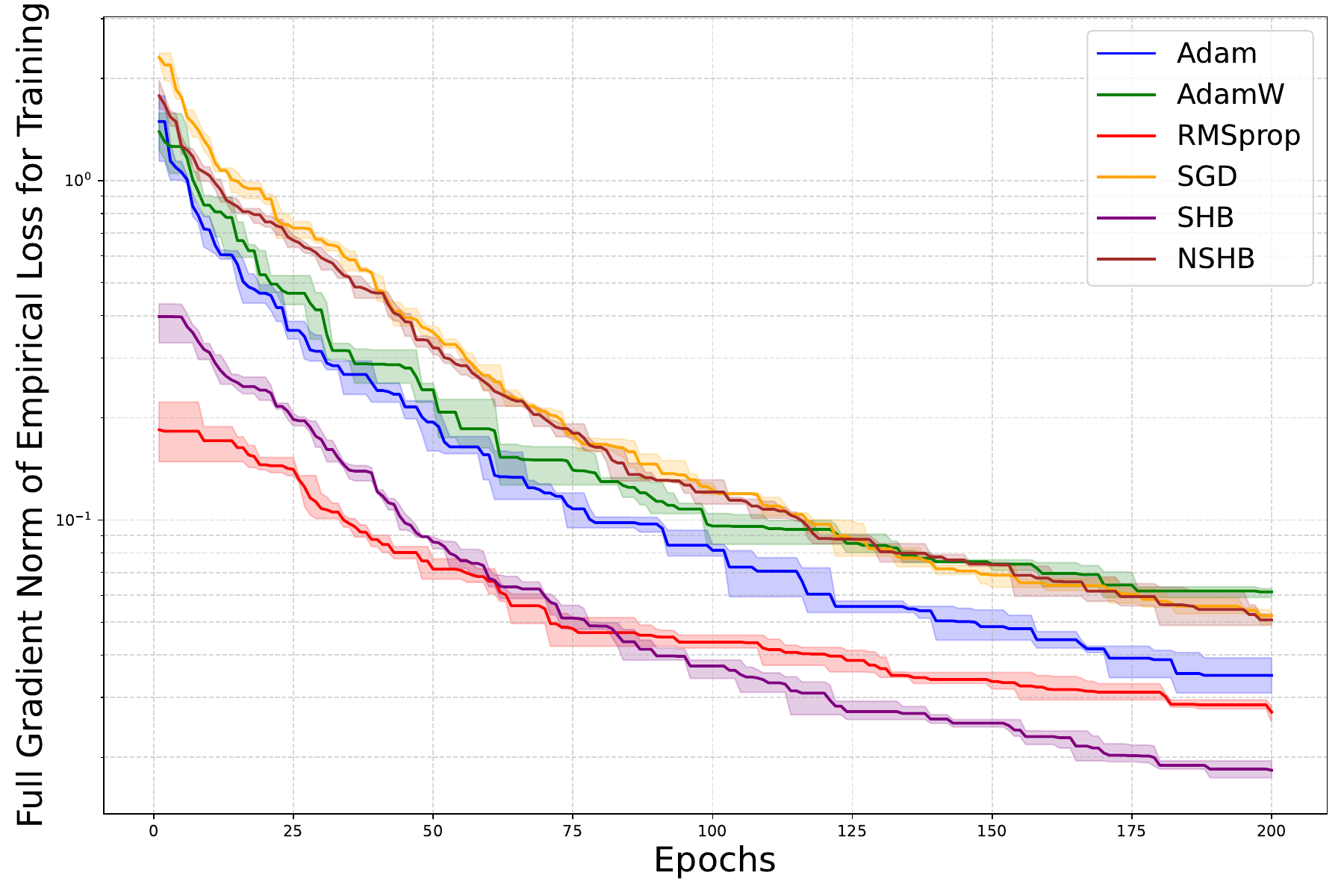} 
        \subfigure{Full gradient norm versus epochs}
    \end{minipage}
    \begin{minipage}{0.4\textwidth}
        \centering
        \includegraphics[width=0.75\textwidth]{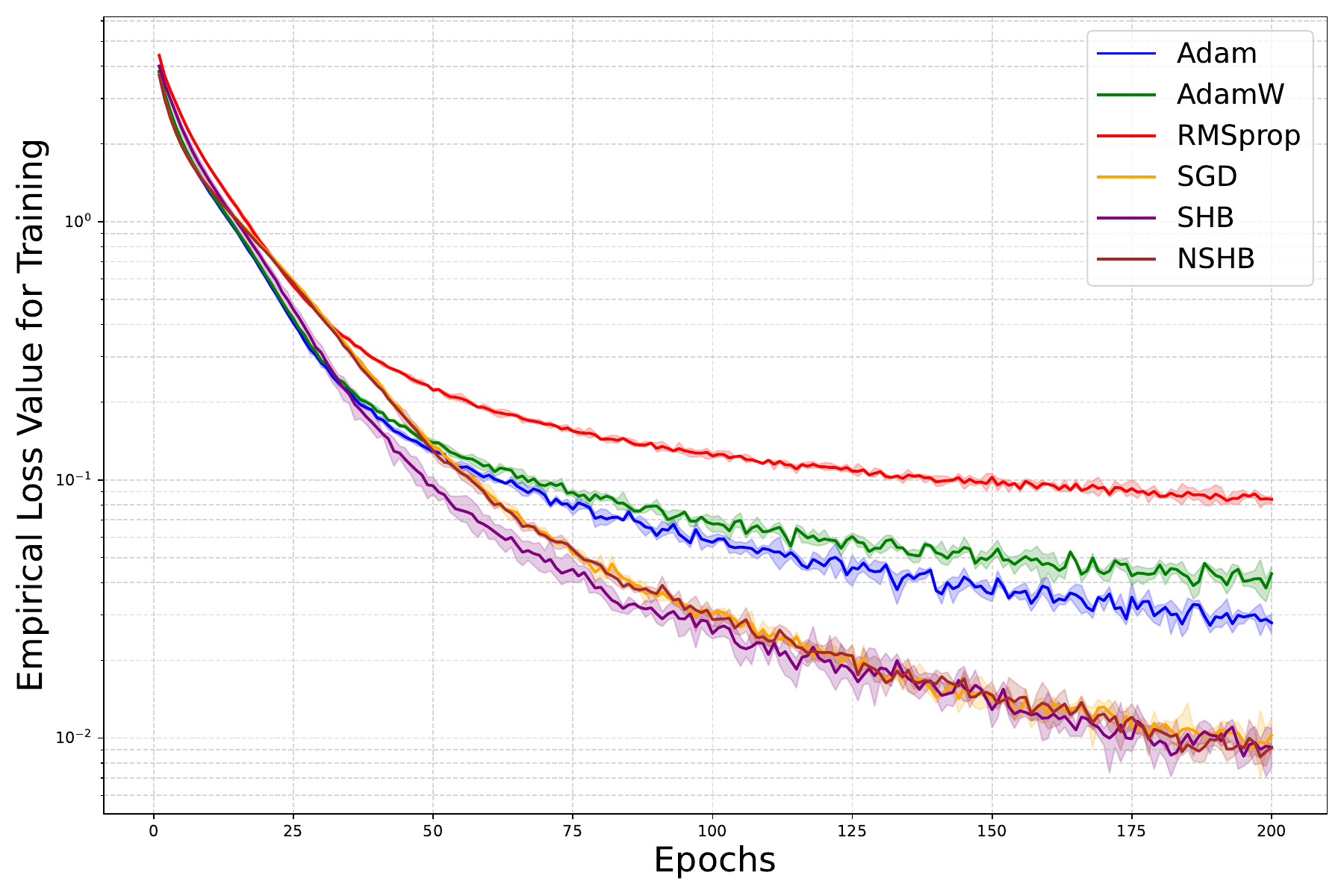} 
        \subfigure{Empirical loss versus epochs}
    \end{minipage} \hfill
    \begin{minipage}{0.4\textwidth}
        \centering
        \includegraphics[width=0.75\textwidth]{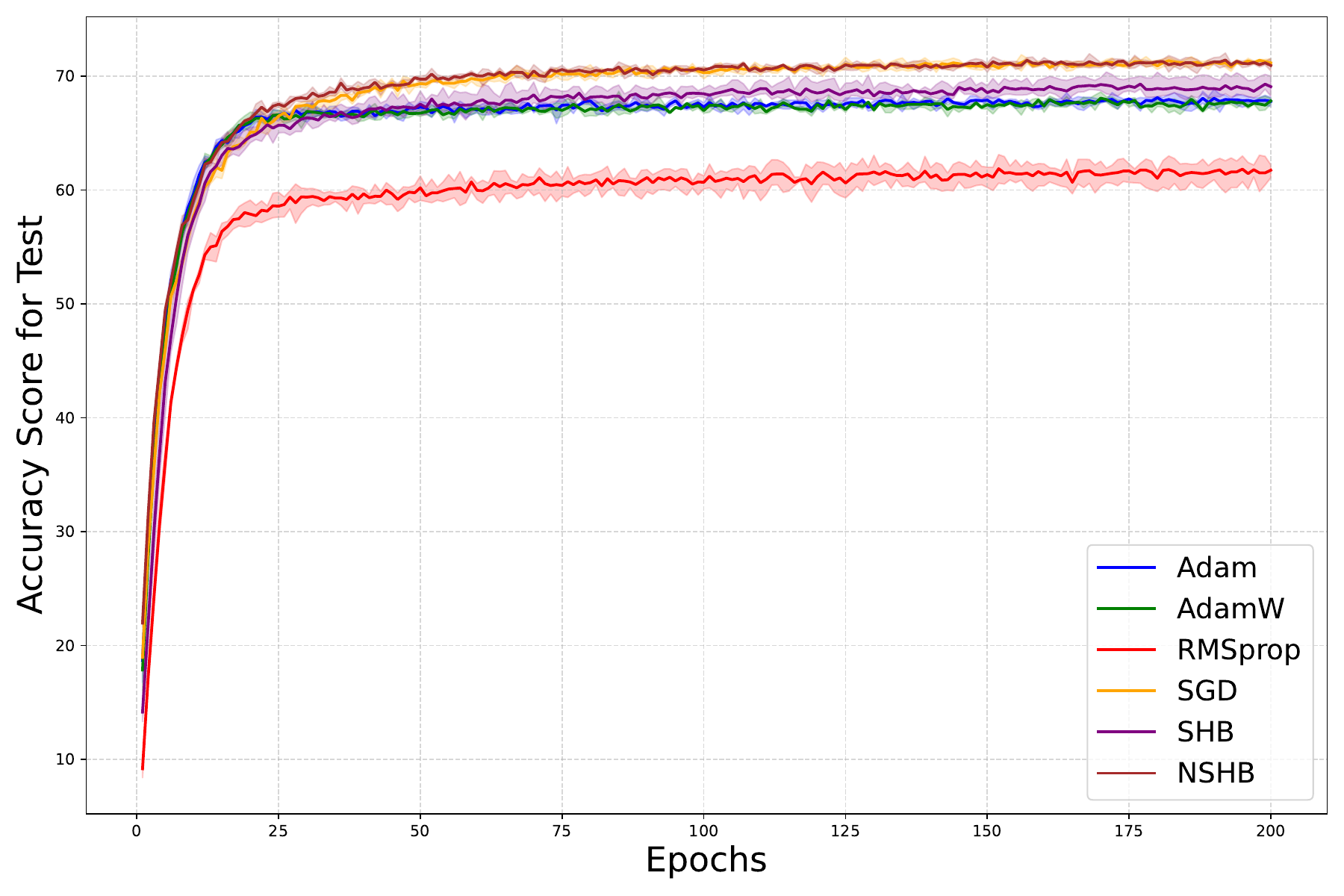} 
        \subfigure{Test accuracy score versus epochs}
    \end{minipage}
    \caption{(a) Schedules for each optimizer with constant learning rates and a constant batch size, (b) Full gradient norm of empirical loss for training, (c) Empirical loss value for training, and (d) Accuracy score for test to train ResNet-18 on CIFAR-100 dataset.}
    \label{fig_1}
\end{figure}

\begin{figure}[H]
    \centering
    \begin{minipage}{0.4\textwidth}
        \centering
        \includegraphics[width=0.75\textwidth]{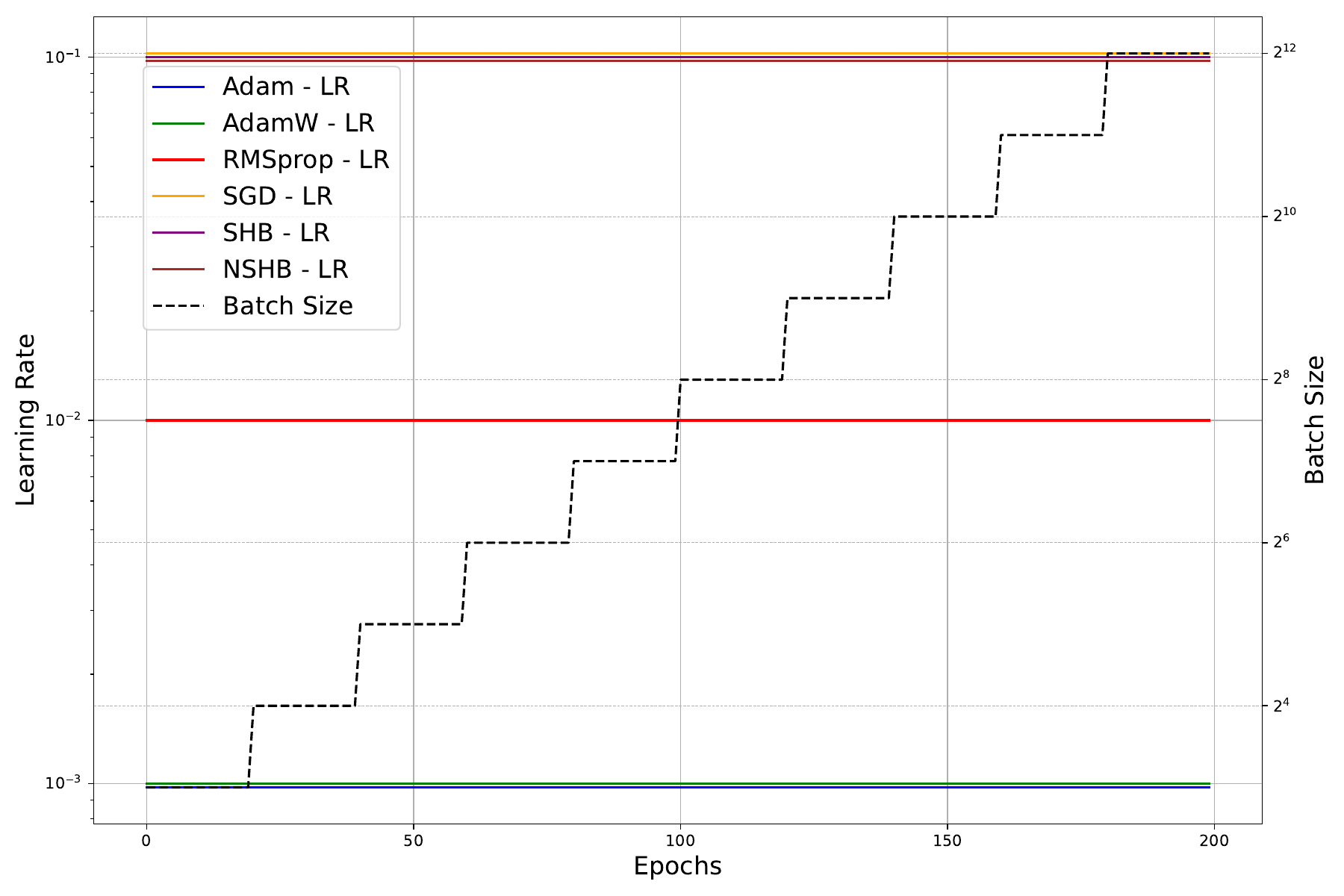} 
        \subfigure{Learning rate and batch size schedules}
    \end{minipage} \hfill
    \begin{minipage}{0.4\textwidth}
        \centering
        \includegraphics[width=0.75\textwidth]{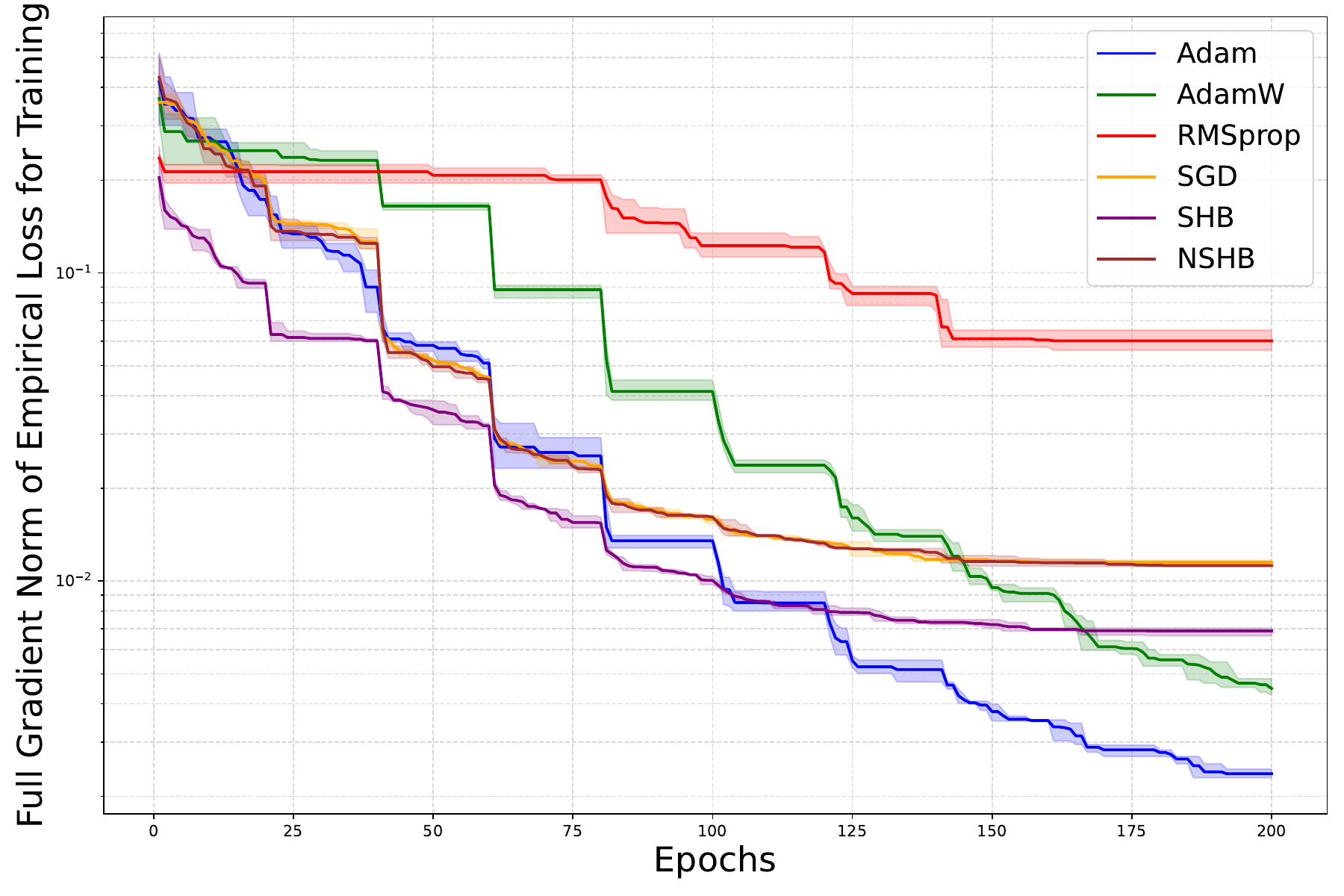} 
        \subfigure{Full gradient norm versus epochs}
    \end{minipage}
    \begin{minipage}{0.4\textwidth}
        \centering
        \includegraphics[width=0.75\textwidth]{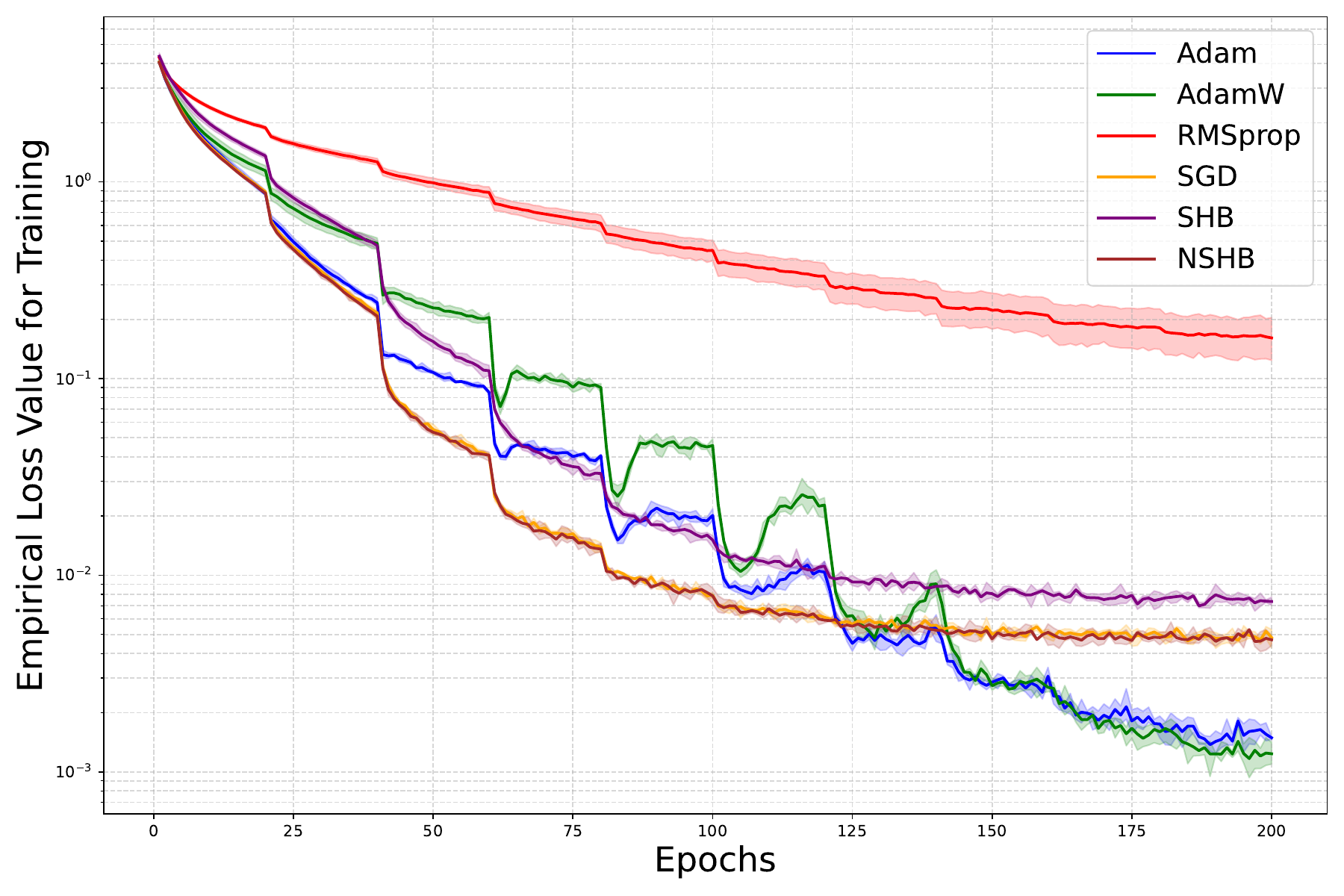} 
        \subfigure{Empirical loss versus epochs}
    \end{minipage} \hfill
    \begin{minipage}{0.4\textwidth}
        \centering
        \includegraphics[width=0.75\textwidth]{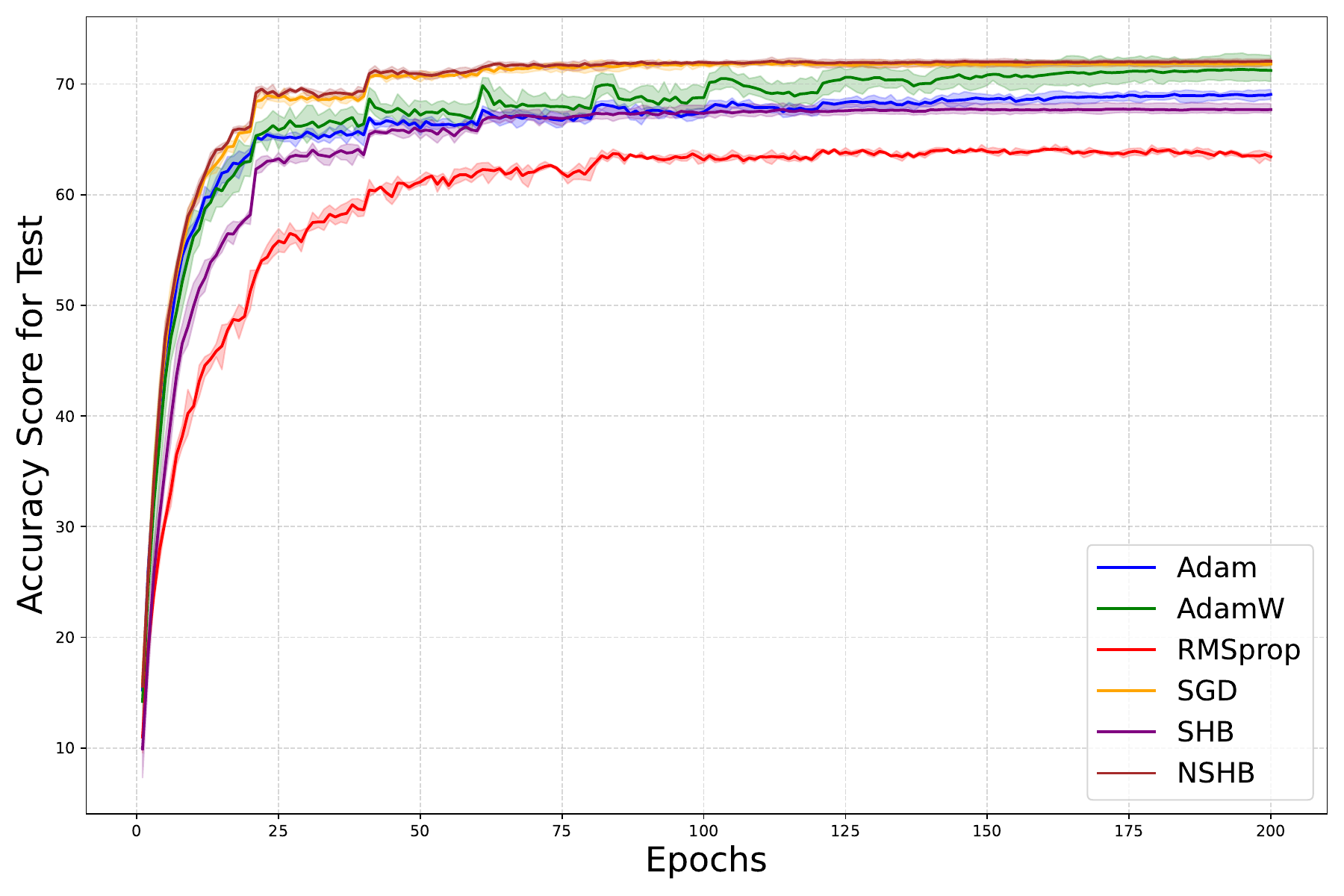} 
        \subfigure{Test accuracy score versus epochs}
    \end{minipage}
    \caption{(a) Schedules for each optimizer with constant learning rate and a batch size doubling every 20 epochs, (b) Full gradient norm of empirical loss for training, (c) Empirical loss value for training, and (d) Accuracy score for test to train ResNet-18 on CIFAR-100 dataset.}
    \label{fig_2}
\end{figure}

\begin{figure}[H]
    \centering
    \begin{minipage}{0.4\textwidth}
        \centering
        \includegraphics[width=0.75\textwidth]{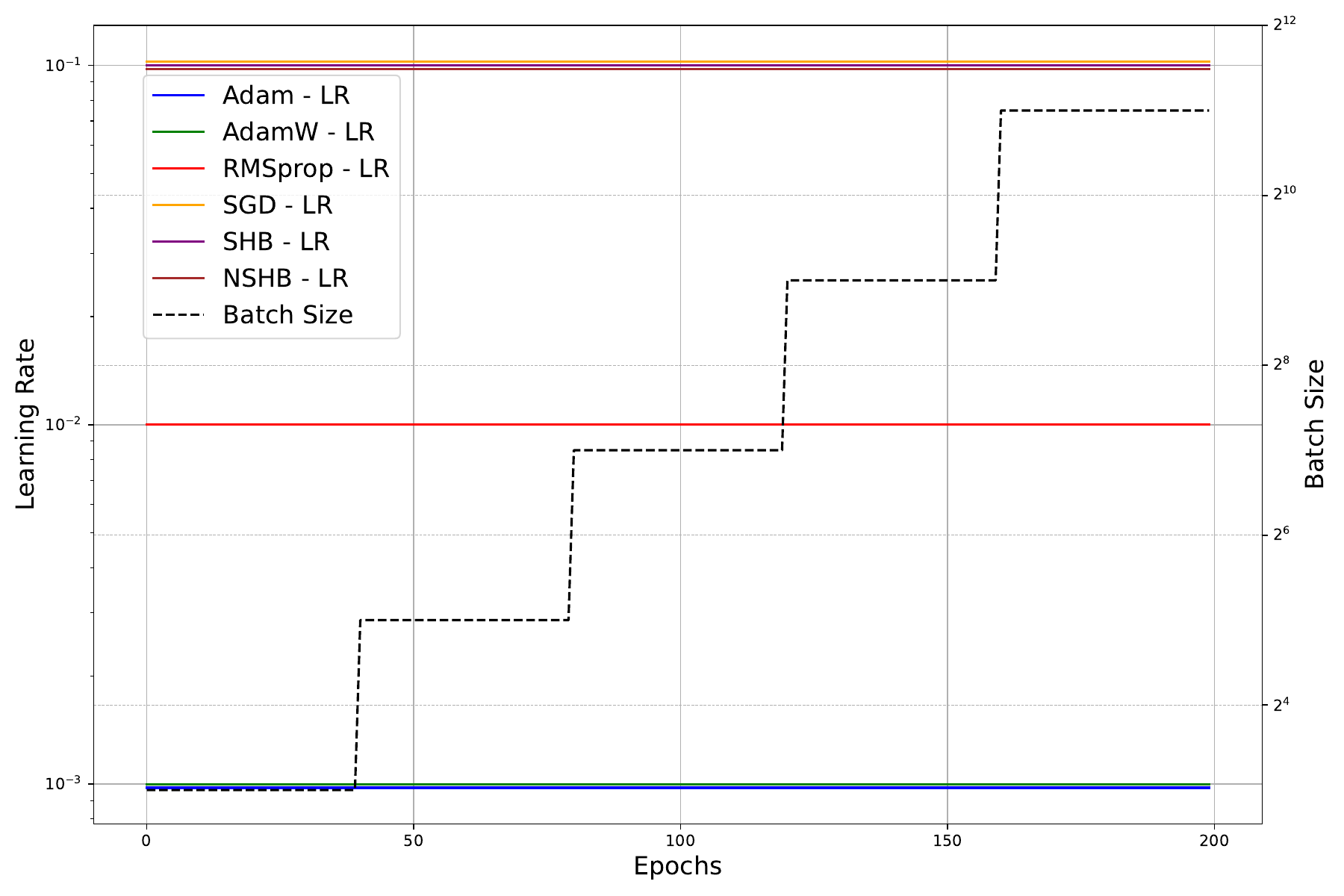} 
        \subfigure{Learning rate and batch size schedules}
    \end{minipage} \hfill
    \begin{minipage}{0.4\textwidth}
        \centering
        \includegraphics[width=0.75\textwidth]{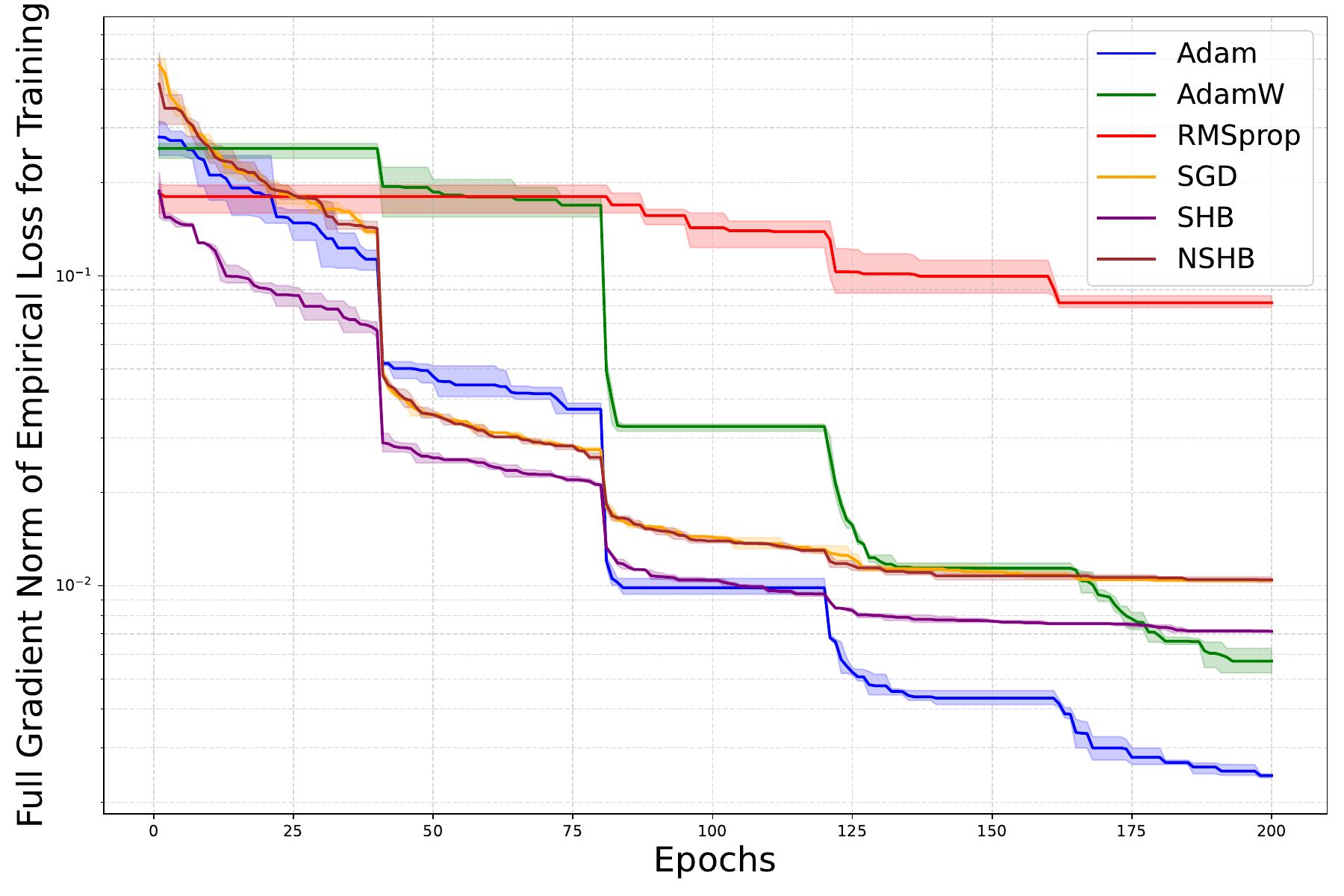} 
        \subfigure{Full gradient norm versus epochs}
    \end{minipage}
    \begin{minipage}{0.4\textwidth}
        \centering
        \includegraphics[width=0.75\textwidth]{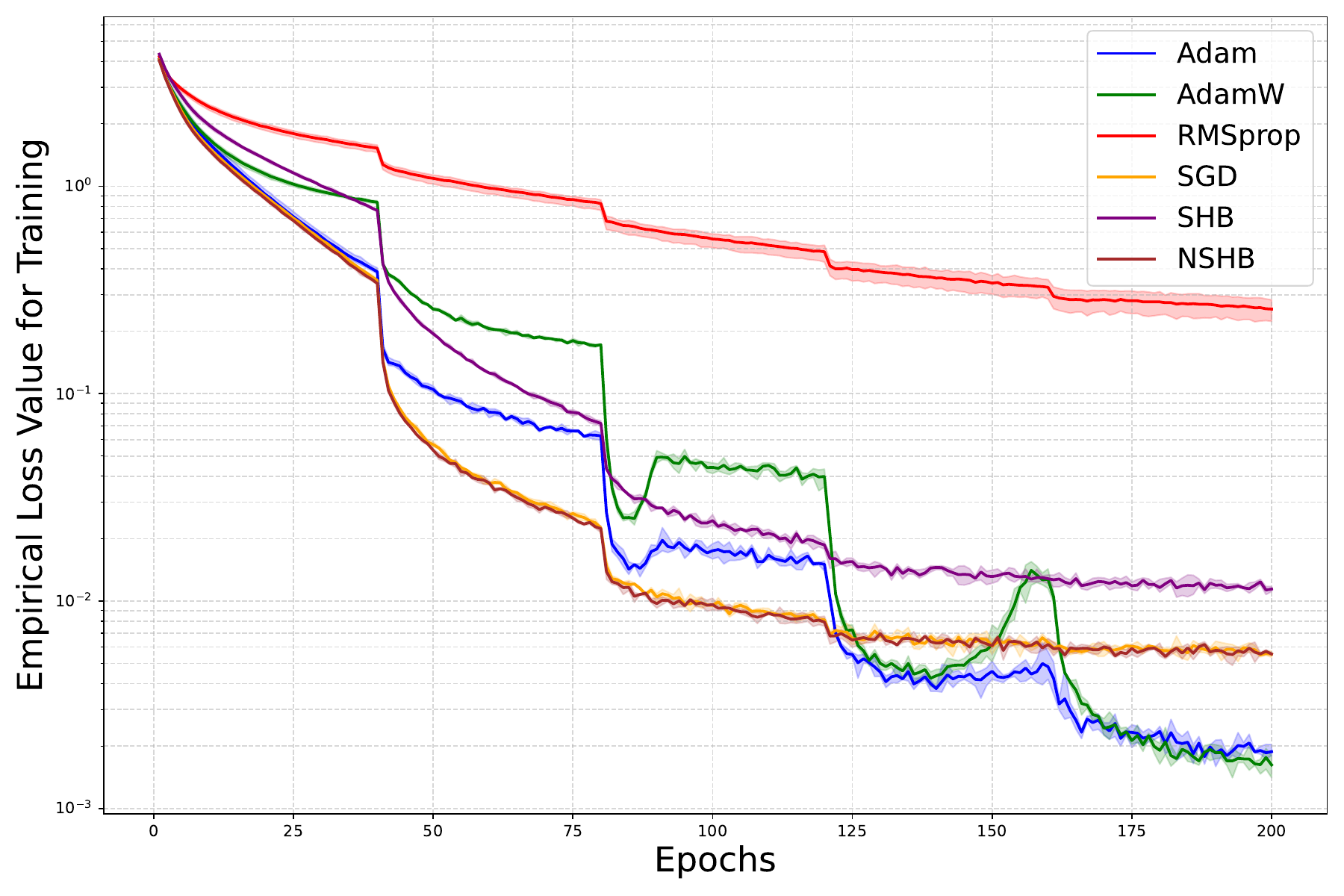} 
        \subfigure{Empirical loss versus epochs}
    \end{minipage} \hfill
    \begin{minipage}{0.4\textwidth}
        \centering
        \includegraphics[width=0.75\textwidth]{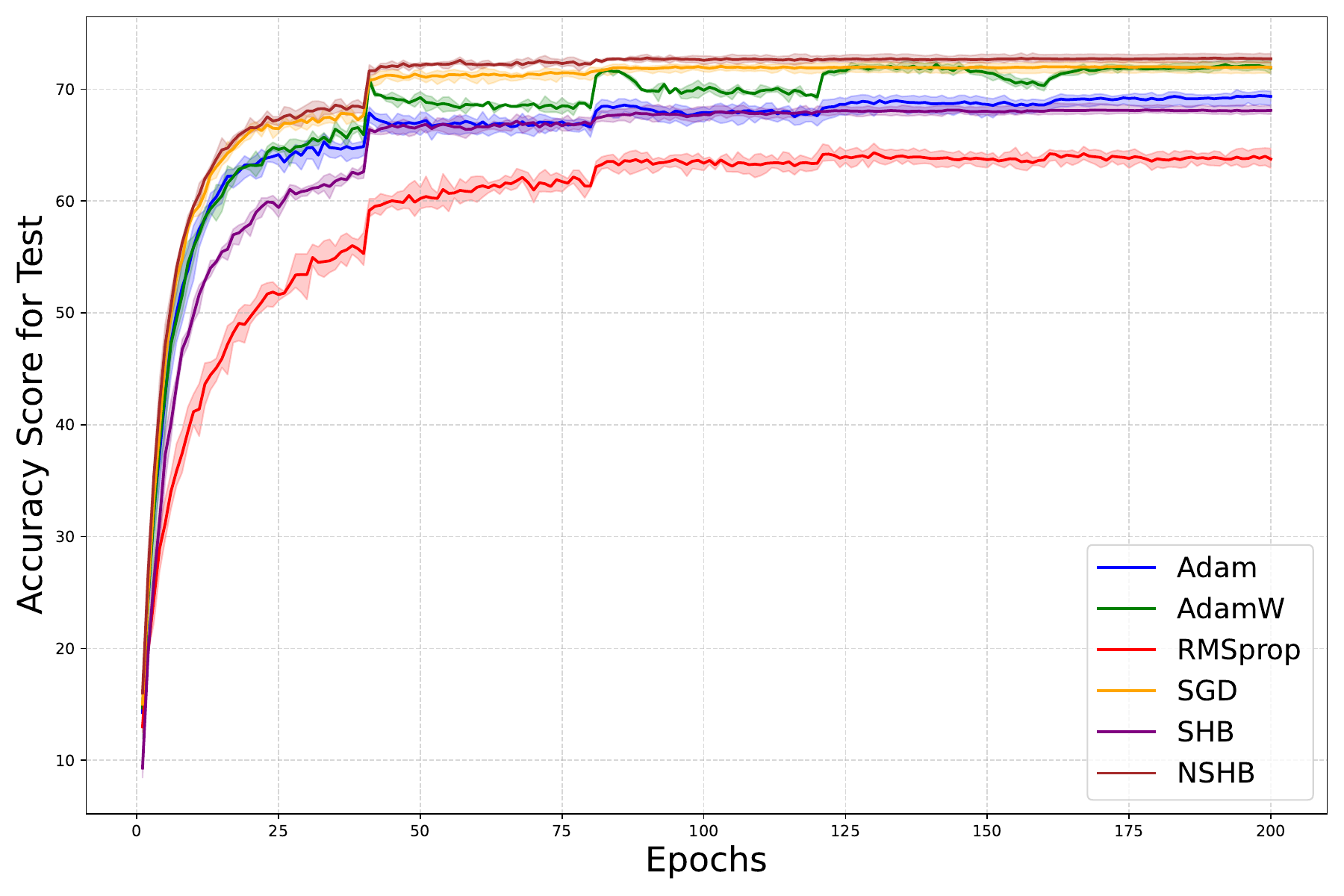} 
        \subfigure{Test accuracy score versus epochs}
    \end{minipage}
    \caption{(a) Schedules for each optimizer with constant learning rates and a batch size quadrupling every 40 epochs, (b) Full gradient norm of empirical loss for training, (c) Empirical loss value for training, and (d) Accuracy score for test to train ResNet-18 on CIFAR-100 dataset.}
    \label{fig_3}
\end{figure}

Let us first consider the learning rate and batch size scheduler in Figure \ref{fig_1}(a) with a constant batch size ($b=2^7$). Figure \ref{fig_1}(b) compares the full gradient norm $\min_{e \in [E]} \|\nabla f (\bm{\theta}_e)\|$ for training for each optimizer and indicates that SHB decreased the full gradient norm quickly. Figures \ref{fig_1}(c) and (d) compare the empirical loss $f (\bm{\theta}_e)$ and the {test} accuracy score. These figures indicate that SGD, SHB, and NSHB minimized $f$ quickly and had test accuracies of approximately 70 \%. Next, let us compare Figure \ref{fig_1} with Figure \ref{fig_2} for when the scheduler uses the same learning rates as in Figure \ref{fig_1}(a) and a batch size doubling every 20 epochs with the initial batch size set at $b_0 = 2^3$. Figures \ref{fig_2}(b) and (c) both show that using a doubly increasing batch size results in a faster decrease in $\min_{e \in [E]} \|\nabla f (\bm{\theta}_e)\|$ and $f (\bm{\theta}_e)$, compared with using a constant batch size as in Figures \ref{fig_1}(b) and (c). The numerical results in Figures \ref{fig_1}(b) and \ref{fig_2}(b) are supported theoretically by Theorems \ref{thm:1_1}, \ref{thm:1}, \ref{thm:2_1} and \ref{thm:2} indicating that NSHB and SHB with increasing batch sizes minimize the gradient norm of $f$ faster than with constant batch sizes. In Figures \ref{fig_1}(d) and \ref{fig_2}(d), it can be seen that using a doubly increasing batch size leads to improved test accuracy for all optimizers except SHB, compared with using a constant batch size. Earlier, we observed that, with a constant batch size, convergence is slower and accuracy improves more gradually. On the other hand, these results suggest that using an increasing batch size leads to faster convergence and more efficient training. Additionally, the optimizer's performance is better overall when using an increasing batch size. 

Now, let us compare Figure \ref{fig_2} ($\delta = 2$) with Figure \ref{fig_3} ($\delta = 4$) when the scheduler uses the same learning rates as in Figure \ref{fig_1}(a) and a batch size quadrupling every 40 epochs with the initial batch size set at $b_0 = 2^3$. From Figures \ref{fig_3}(b) and (c), it can be observed that the larger the batch size is, the faster the decrease of the full gradient norm $\|\nabla f (\bm{\theta}_e)\|$ and the empirical loss $f (\bm{\theta}_e)$ become. Specifically, the quadruply increasing batch size ($\delta = 4$; Figure \ref{fig_3}) decreases the full gradient norm $\|\nabla f (\bm{\theta}_e)\|$ and the empirical loss $f (\bm{\theta}_e)$ more rapidly than the doubly increasing batch size ($\delta = 2$; Figure \ref{fig_2}). Figures \ref{fig_2}(d) and \ref{fig_3}(d) indicate that SGD and NSHB had test accuracies greater than 70 \%, which implies that, for SGD and NSHB, using an increasing batch size would improve generalization more than using a constant batch size (Figure \ref{fig_1}(d)).

\subsection{Discussion and future work} 
\label{sec:4.1}
\textbf{Fast convergence of Adam:} A particularly interesting result in Figures \ref{fig_2}--\ref{fig_3} is that an increasing batch size is applicable for Adam in the sense of it helping to minimize the full gradient norm of $f$ fastest. Hence, we can expect that Adam with an increasing batch size has a convergence rate better than the $O(\frac{1}{\sqrt{T}})$ convergence rate of NSHB and SHB in Theorems \ref{thm:1} and \ref{thm:2}. In the future, we should verify that this result holds theoretically. 

\noindent\textbf{Full gradient norm and training loss versus test accuracy:} As promised in Theorems \ref{thm:1_1} and \ref{thm:1}, NSHB with increasing batch sizes ($\delta = 2,4$) minimized the full gradient norm of $f$ faster than with a constant batch size (Figures \ref{fig_1}(b), \ref{fig_2}(b), and \ref{fig_3}(b)). As a result, NSHB with an increasing batch size ($\delta = 2,4$) minimized the training loss $f$ (Figures \ref{fig_1}(c), \ref{fig_2}(c), and \ref{fig_3}(c)) and had higher test accuracies than with a constant batch size (Figures \ref{fig_1}(d), \ref{fig_2}(d), and \ref{fig_3}(d)). Moreover, Figures \ref{fig_1}--\ref{fig_3} indicate that AdamW had almost the same trend. Although Adam and AdamW with increasing batch sizes both minimized $f$ quickly, their test accuracies were different (Figure \ref{fig_3}(d)). Here, we have the following insights: 

\noindent(1) An increasing batch size quickly minimizes the full gradient norm of the training loss in both theory and practice. In particular, SGDM with an increasing batch size converges to stationary points of the training loss, as promised in our theoretical results.

\noindent(2) Optimal increasing-batch size-schedulers with which optimizers have high test accuracies should be discussed. Specifically, we need to find the optimal $E_m$ and $\delta$ with which SGDM and adaptive methods (e.g., Adam and AdamW) can improve generalization.

\subsection{Computational cost evaluation}
\label{sec:4.2}
We evaluated the efficiency of fixed and increasing batch-size schedules in terms of the number of stochastic gradient computations required to achieve specific training goals. To quantify this, we define the SFO complexity, which corresponds to the total number of gradient evaluations. If the batch size is $b$ and the number of training steps is $T$, then the SFO complexity is given by $Tb$.

To simulate realistic GPU memory constraints, we capped the maximum batch size at 1024. All experiments were conducted using the CIFAR-100 dataset with the NSHB optimizer, under the same settings as in the other numerical experiments.

We compared the SFO complexity required to (i) reach a gradient norm threshold (e.g., $\|\nabla f (\bm{\theta}_t)\| < 0.05$) and (ii) achieve 70\% test accuracy (see also Appendices \ref{appendix:A.5} and \ref{appendix:A.6}).

\noindent\textbf{Experimental settings.} We considered both fixed and increasing batch size schedules in the gradient norm and test accuracy evaluations. For the gradient norm evaluation, the fixed setting used $b = 8$ and $b = 128$, while the increasing setting started with $b = 8$ (doubling every 20 epochs) and with $b = 128$ (doubling every 50 epochs). For the test accuracy evaluation, the fixed setting also used $b = 8$ and $b = 128$, while the increasing setting started with $b = 8$ (doubling every 20 epochs) and with $b = 128$ (doubling every 25 epochs).

\noindent\textbf{Results.}
\begin{table}[H]
\normalsize
\centering
\caption{SFO complexity to reach gradient norm threshold ($b=8$)}
\label{table:2}
\begin{tabular}{lcc}
\toprule
Method & $\|\nabla f (\bm{\theta}_t) \| < 0.1$ & $\|\nabla f (\bm{\theta}_t) \| < 0.05$ \\
\midrule
Fixed batch size ($b=8$) & 2,750,000 & 5,250,000 \\
Increasing batch size (initial $b=8$) & 2,050,016 & 2,500,160 \\
\bottomrule
\end{tabular}
\end{table}

\begin{table}[H]
\normalsize
\centering
\caption{SFO complexity to reach gradient norm threshold ($b=128$)}
\label{table:3}
\begin{tabular}{lcc}
\toprule
Method & $\|\nabla f (\bm{\theta}_t) \| < 0.1$ & $\|\nabla f (\bm{\theta}_t) \| < 0.06$ \\
\midrule
Fixed batch size ($b=128$) & 5,755,520 & 9,809,408 \\
Increasing batch size (initial $b=128$) & 2,903,808 & 5,061,376 \\
\bottomrule
\end{tabular}
\end{table}

\begin{table}[H]
\large
\centering
\caption{SFO complexity to reach 70\% test accuracy}
\label{table:4}
\begin{tabular}{lc}
\toprule
Method & SFO \\
\midrule
Fixed batch size ($b=8$) & 5,250,000 \\
Increasing batch size (initial $b=8$) & 2,050,016 \\
Fixed batch size ($b=128$) & 2,502,400 \\
Increasing batch size (initial $b=128$) & 1,301,376 \\
\bottomrule
\end{tabular}
\end{table}

\noindent\textbf{Discussion.} These results show that increasing the batch size significantly reduces the total number of stochastic gradient computations needed to achieve optimization and generalization goals, especially under realistic memory constraints. Our experiments confirm that using an increasing batch size reduces gradient evaluations compared with using a fixed batch size, even when both achieve similar test accuracy and optimization performance. This highlights that larger batch sizes are not just a theoretical convenience but offer clear practical benefits. They provide an efficient way to reduce training costs without compromising generalization, especially in large-scale deep learning under memory and compute constraints.

\section{Conclusion}
\label{sec:5}
This paper presented convergence analyses of mini-batch SGDM with a constant learning rate and momentum weight. We showed that, unlike prior studies that assume a decaying learning rate to ensure convergence, increasing the batch size under a constant learning rate and momentum not only guarantees convergence but also achieves faster convergence. Numerical experiments supported our theory, demonstrating faster convergence, higher test accuracy, and reduced computational costs compared with a constant batch size. Moreover, our results suggested that increasing batch size can also benefit adaptive methods such as Adam and AdamW. Future work includes extending our analysis to larger-scale datasets and deeper architectures, as well as generalizing the framework beyond the exponential growth schedule to cover polynomial growth and adaptive schemes. These directions of study will further clarify the role of batch size in modern optimization and strengthen the connection between theory and practice.

\section*{Acknowledgements}
\label{sec:ack}
We are sincerely grateful to Program Chairs, Area Chairs, and the four anonymous reviewers for helping us improve the original manuscript. This work was supported by the Japan Society for the Promotion of Science (JSPS) KAKENHI Grant Number 24K14846 awarded to Hideaki Iiduka.

\bibliography{acml25}
\bibliographystyle{plainnat}

\appendix
\onecolumn
\section{Proofs of Theorems in the Paper}
\subsection{Proposition and lemma}
The following proposition holds for the mini-batch gradient.

\begin{myproposition}\label{prop:1}
Let $t \in \mathbb{N}$, $\bm{\xi}_t$ be a random variable independent of $\bm{\xi}_j$ ($j \in [0:t-1]$), $\bm{\theta}_t \in \mathbb{R}^d$ be independent of $\bm{\xi}_t$, and $\nabla f_{B_t}(\bm{\theta}_t)$ be the mini-batch gradient, where $f_{\xi_{t,i}}$ ($i\in [b_t]$) is the stochastic gradient (see Assumption \ref{assum:1}(A2)). Then, the following hold:
\begin{align*}
&\mathbb{E}_{\bm{\xi}_t}\left[\nabla f_{B_t}(\bm{\theta}_t) \Big|\hat{\bm{\xi}}_{t-1} \right] = \nabla f (\bm{\theta}_t)
\text{ and }
\mathbb{V}_{\bm{\xi}_t}\left[\nabla f_{B_t}(\bm{\theta}_t) \Big|\hat{\bm{\xi}}_{t-1} \right] 
\leq \frac{\sigma^2}{b_t},
\end{align*}
where $\mathbb{E}_{\bm{\xi}_t}[\cdot|\hat{\bm{\xi}}_{t-1}]$ and $\mathbb{V}_{\bm{\xi}_t}[\cdot|\hat{\bm{\xi}}_{t-1}]$ are respectively the expectation and variance with respect to $\bm{\xi}_t$ conditioned on $\bm{\xi}_{t-1} = \hat{\bm{\xi}}_{t-1}$.
\end{myproposition}

\begin{proof} 
Assumption \ref{assum:1}(A3) and the independence of $b_t$ and $\bm{\xi}_t$ ensure that 
\begin{align*}
\mathbb{E}_{\bm{\xi}_t} \left[\nabla f_{B_t}(\bm{\theta}_t) \Big| \hat{\bm{\xi}}_{t-1} \right]
= 
\mathbb{E}_{\bm{\xi}_t} \left[\frac{1}{b_t} \sum_{i=1}^{b_t} \nabla f_{\xi_{t,i}} (\bm{\theta}_t) \Bigg| \hat{\bm{\xi}}_{t-1} \right]
=
\frac{1}{b_t} \sum_{i=1}^{b_t}
\mathbb{E}_{\xi_{t,i}} \left[\nabla f_{\xi_{t,i}} (\bm{\theta}_t) \Big| \hat{\bm{\xi}}_{t-1} \right],
\end{align*}
which, together with Assumption \ref{assum:1}(A2)(i) and the independence of $\bm{\xi}_t$ and $\bm{\xi}_{t-1}$, implies that
\begin{align}\label{eq:1}
\mathbb{E}_{\bm{\xi}_t} \left[\nabla f_{B_t}(\bm{\theta}_t) \Big| \hat{\bm{\xi}}_{t-1} \right] 
= 
\frac{1}{b_t} \sum_{i=1}^{b_t}
\nabla f (\bm{\theta}_t) 
= \nabla f (\bm{\theta}_t).
\end{align}
Assumption \ref{assum:1}(A3), the independence of $b_t$ and $\bm{\xi}_t$, and (\ref{eq:1}) imply that
\begin{align*}
\mathbb{V}_{\bm{\xi}_t} \left[ \nabla f_{B_t}(\bm{\theta}_t) \Big| \hat{\bm{\xi}}_{t-1} \right]
&= 
\mathbb{E}_{\bm{\xi}_t} \left[ \|\nabla f_{B_t} (\bm{\theta}_t) - \nabla f (\bm{\theta}_t)\|^2 \Big| \hat{\bm{\xi}}_{t-1} \right]\\
&= 
\mathbb{E}_{\bm{\xi}_t} 
\left[ \left\| 
\frac{1}{b_t} \sum_{i=1}^{b_t} \nabla f_{\xi_{t,i}} (\bm{\theta}_t) - \nabla f (\bm{\theta}_t) \right\|^2 \Bigg| \hat{\bm{\xi}}_{t-1} \right]\\
&= 
\frac{1}{b_t^2} \mathbb{E}_{\bm{\xi}_t} 
\left[ \left\| 
\sum_{i=1}^{b_t} \left( 
\nabla f_{\xi_{t,i}} (\bm{\theta}_t) - \nabla f (\bm{\theta}_t) 
\right) 
\right\|^2 \Bigg| \hat{\bm{\xi}}_{t-1} \right].
\end{align*}
From the independence of $\xi_{t,i}$ and $\xi_{t,j}$ ($i \neq j$) and Assumption \ref{assum:1}(A2)(i), for all $i,j \in [b_t]$ such that $i \neq j$, 
\begin{align*}
&\mathbb{E}_{\xi_{t,i}}[\langle \nabla f_{\xi_{t,i}}(\bm{\theta}_t) - \nabla f (\bm{\theta}_t), \nabla f_{\xi_{t,j}}(\bm{\theta}_t)- \nabla f (\bm{\theta}_t) \rangle| \hat{\bm{\xi}}_{t-1}]\\
&=
\langle \mathbb{E}_{\xi_{t,i}}[ \nabla f_{\xi_{t,i}}(\bm{\theta}_t) | \hat{\bm{\xi}}_{t-1}] - \mathbb{E}_{\xi_{t,i}}[\nabla f (\bm{\theta}_t)|\hat{\bm{\xi}}_{t-1}],\nabla f_{\xi_{t,j}}(\bm{\theta}_t)- \nabla f (\bm{\theta}_t) \rangle\\
&= 0.
\end{align*}
Hence, Assumption \ref{assum:1}(A2)(ii) guarantees that
\begin{align*}
\mathbb{V}_{\bm{\xi}_t} \left[ \nabla f_{B_t}(\bm{\theta}) \Big| \hat{\bm{\xi}}_{t-1} \right]
=
\frac{1}{b_t^2}    
\sum_{i=1}^{b_t} \mathbb{E}_{\xi_{t,i}} \left[\left\| 
\nabla f_{\xi_{t,i}} (\bm{\theta}_t) - \nabla f (\bm{\theta}_t)  
\right\|^2 \Big| \hat{\bm{\xi}}_{t-1} \right]
\leq
\frac{\sigma^2 b_t}{b_t^2}  
= \frac{\sigma^2}{b_t},
\end{align*}
which completes the proof.
\end{proof}

Motivated by Lemma 1 in \citep{NEURIPS2020_d3f5d4de}, we prove the following lemma. 

\begin{mylemma}\label{lem:4}
Under Assumption \ref{assum:1}, Algorithm \ref{algo:1} satisfies that, for all $t\in \{0\} \cup \mathbb{N}$,
\begin{align*}
\mathbb{E}\left[\left\Vert \boldsymbol m_t - \left(1-\beta\right)\sum_{i=0}^{t} \beta^{t - i}\nabla f(\bm{\theta}_i) \right\rVert^2\right] \le 
(1-\beta)^2 \sigma^2\sum_{i=0}^{t} \frac{\beta^{2(t - i)}}{b_i},
\end{align*}
where $\mathbb{E}$ denotes the total expectation defined by $\mathbb{E} := \mathbb{E}_{\bm{\xi}_0} \mathbb{E}_{\bm{\xi}_1} \cdots \mathbb{E}_{\bm{\xi}_{t}}$.
\end{mylemma}

\begin{proof}
The definition of $\bm{m}_t$ and $\bm{m}_{-1} := \bm{0}$ ensure that
\begin{align*}
\bm{m}_t 
&= \beta \bm{m}_{t-1} + (1 - \beta) \nabla f_{B_t}(\bm{\theta}_t)\\
&= \beta \{\beta \bm{m}_{t-2} + (1 - \beta) \nabla f_{B_{t-1}}(\bm{\theta}_{t-1})\}  + (1 - \beta) \nabla f_{B_t}(\bm{\theta}_t)\\
&= \beta^2 \bm{m}_{t-2} + (1 - \beta) \{ \beta \nabla f_{B_{t-1}}(\bm{\theta}_{t-1})
+ \beta^0 \nabla f_{B_t}(\bm{\theta}_t)  \}\\
&= \beta^{t+1} \bm{m}_{-1} + (1 - \beta) \sum_{i=0}^{t} \beta^{t-i} \nabla f_{B_{i}}(\bm{\theta}_{i})\\
&= (1 - \beta) \sum_{i=0}^{t} \beta^{t-i} \nabla f_{B_{i}}(\bm{\theta}_{i}),
\end{align*}
which, together with $\|\bm{\theta}\|^2 = \langle \bm{\theta}, \bm{\theta} \rangle$, implies that
\begin{align*}
&\left\| \boldsymbol m_t - (1-\beta)\sum_{i=0}^t \beta^{t-i} \nabla f(\bm{\theta}_i) \right\|^2 \\
&= (1-\beta)^2 \left\| \sum_{i=0}^t \beta^{t-i} 
(\nabla f_{B_i} (\bm{\theta}_i) - \nabla f (\bm{\theta}_i)) \right\|^2\\
&= 
(1-\beta)^2 
\sum_{i=0}^t \sum_{j=0}^t 
\left\langle \beta^{t-i}(\nabla f_{B_i} (\bm{\theta}_i) - \nabla f (\bm{\theta}_i)), \beta^{t-j} (\nabla f_{B_j} (\bm{\theta}_j) - \nabla f (\bm{\theta}_j)) \right\rangle. 
\end{align*}
Let $i$ and $j$ satisfy $0 \leq j < i \leq t$. Proposition \ref{prop:1} and Assumptions (A2) and (A3) imply that
\begin{align*}
&\mathbb{E} \left[ \left\langle \nabla f_{B_i} (\bm{\theta}_i) - \nabla f (\bm{\theta}_i), \nabla f_{B_j} (\bm{\theta}_j) - \nabla f (\bm{\theta}_j) \right\rangle \right] \\
&= \mathbb{E}_{\bm{\xi}_0}\mathbb{E}_{\bm{\xi}_1} \cdots \mathbb{E}_{\bm{\xi}_t}
\left[\langle \nabla f_{B_i} (\bm{\theta}_i) - \mathbb{E}_{\bm{\xi}_i}\left[\nabla f_{B_i}(\bm{\theta}_i) \right], \nabla f_{B_j} (\bm{\theta}_j) - \mathbb{E}_{\bm{\xi}_j}[\nabla f_{B_j}(\bm{\theta}_j)] \rangle\right] \\
&= \mathbb{E}_{\bm{\xi}_0}\mathbb{E}_{\bm{\xi}_1} \cdots \mathbb{E}_{\bm{\xi}_i}
\left[ \langle \nabla f_{B_i} (\bm{\theta}_i) - \mathbb{E}_{\bm{\xi}_i} [\nabla f_{B_i}(\bm{\theta}_i) ], \nabla f_{B_j} (\bm{\theta}_j) - \mathbb{E}_{\bm{\xi}_j}[\nabla f_{B_j}(\bm{\theta}_j)] \rangle \right] \\
&= \mathbb{E}_{\bm{\xi}_0}\mathbb{E}_{\bm{\xi}_1} \cdots \mathbb{E}_{\bm{\xi}_{i-1}}\left[\langle \mathbb{E}_{\bm{\xi}_i} [\nabla f_{B_i}(\bm{\theta}_i)] - \mathbb{E}_{\bm{\xi}_i}[\nabla f_{B_i}(\bm{\theta}_i)], \nabla f_{B_j}(\bm{\theta}_j) - \mathbb{E}_{\bm{\xi}_j} [\nabla f_{B_j}(\bm{\theta}_j)] \rangle\right] \\
&= 0.
\end{align*}
A similar argument as in the case of $j < i$ ensures the above equation also holds for $i < j$. Hence, Proposition \ref{prop:1} guarantees that, for all $t \in \mathbb{N}$, 
\begin{align*}
&\mathbb{E} \left[ \left\| \boldsymbol m_t - (1-\beta)\sum_{i=0}^t \beta^{t-i} \nabla f(\bm{\theta}_i) \right\|^2 \right] \\
&= 
(1-\beta)^2 
\sum_{i=0}^t \mathbb{E} \left[ \left\langle \beta^{t-i}(\nabla f_{B_i} (\bm{\theta}_i) - \nabla f (\bm{\theta}_i)), \beta^{t-i} (\nabla f_{B_i} (\bm{\theta}_i) - \nabla f (\bm{\theta}_i)) \right\rangle \right]\\
&= 
(1-\beta)^2 
 \sum_{i=0}^t \beta^{2(t-i)} \mathbb{E} \left[ \|\nabla f_{B_i} (\bm{\theta}_i) - \nabla f (\bm{\theta}_i) \|^2 \right]\\
&\leq (1-\beta)^2 \sum_{i=0}^t \beta^{2(t-i)} \frac{\sigma^2}{b_i},
\end{align*}
which completes the proof.
\end{proof}

\subsection{Proofs of Theorems \ref{thm:1} and \ref{thm:2}}
\label{appendix:thm_1}
Using Lemma \ref{lem:4}, we have the following.

\begin{mylemma}\label{prop:4}
Suppose that Assumption \ref{assum:1} holds and $(\bm{\theta}_t)$ is the sequence generated by Algorithm \ref{algo:1}. We define $(\bm{z}_t)$ for all $t \in \{0\} \cup \mathbb{N}$ as
\begin{align*}
\bm{z}_t := 
\begin{cases}
\bm{\theta}_t &\text{ } (t = 0)\\
\frac{1}{1-\beta} \bm{\theta}_t - \frac{\beta}{1-\beta} \bm{\theta}_{t-1} &\text{ }  (t \geq 1).
\end{cases}
\end{align*} 
Then, for all $t \in \{0\} \cup \mathbb{N}$,
\begin{align*}
\mathbb{E} [f(\boldsymbol z_{t+1})] 
&\le \mathbb{E} [f(\boldsymbol z_t ) ] 
+ \eta \left[ L \left\{ \left(\frac{\beta}{1-\beta} \right)^2 + \frac{3}{2} \right\} \eta - 1 \right] \mathbb{E} [\| \nabla f(\bm{\theta}_t) \|^2 ]\\
&\quad + \frac{L \sigma^2 \eta^2}{2} \left\{ \beta^2 \sum_{i=0}^{t-1} \frac{\beta^{2(t-1-i)}}{b_i} + \frac{1}{b_t}\right\}\\
&\quad + \left(\frac{1}{1-\beta}\right)^2 L\eta^2
{(1 - \beta^t)^2}
\mathbb{E}\left[\left\Vert \frac{1-\beta}{1-\beta^t}\sum_{i=0}^t \beta^{t-i} \nabla f(\bm{\theta}_i) - \nabla f(\bm{\theta}_t) \right\rVert^2\right],
\end{align*}
where $L := \frac{1}{n} \sum_{i\in [n]} L_i$ is the Lipschitz constant of $\nabla f$ and we assume that $\sum_{i=0}^{-1} a_i := 0$ for some $a_i \in \mathbb{R}$.
\end{mylemma}

\begin{proof}
The descent lemma \citep[Lemma 5.7]{beck2017} ensures that, for all $t \in \{0\} \cup \mathbb{N}$, 
\begin{align*}
\mathbb{E}_{\bm{\xi}_t} [f (\boldsymbol z_{t+1} )] 
&\le f (\boldsymbol z_t ) 
+ \mathbb{E}_{\bm{\xi}_t} [\langle \nabla f (\boldsymbol z_t ), \boldsymbol z_{t+1}-\boldsymbol z_t \rangle] 
+ \frac{L}{2}\mathbb{E}_{\bm{\xi}_t} [\| \boldsymbol z_{t+1}-\boldsymbol z_t \| ^2 ], 
\end{align*}
which, together with $\boldsymbol z_{t+1} = \boldsymbol z_t - \eta \nabla f_{B_t} (\bm{\theta}_t)$ \citep[Lemma 3]{NEURIPS2020_d3f5d4de}, implies that 
\begin{align}\label{eq:39}
\mathbb{E}_{\bm{\xi}_t} [f (\boldsymbol z_{t+1} )] 
&\leq f (\boldsymbol z_t ) 
- \eta 
\underbrace{\mathbb{E}_{\bm{\xi}_t} [\langle \nabla f (\boldsymbol z_t), \nabla f_{B_t} (\bm{\theta}_t) \rangle ]}_{X_t} 
+ \frac{L\eta^2}{2}
\underbrace{\mathbb{E}_{\bm{\xi}_t} [\| \nabla f_{B_t} (\bm{\theta}_t) \| ^2 ]}_{Y_t}. 
\end{align}
From Proposition \ref{prop:1}, we have that
\begin{align*}
X_t 
= \langle \nabla f(\boldsymbol z_t ), \mathbb{E}_{\bm{\xi}_t} [\nabla f_{B_t}(\bm{\theta}_t)] \rangle
= \langle \nabla f(\boldsymbol z_t ), \nabla f(\bm{\theta}_t) \rangle,
\end{align*}
which, together with the Cauchy--Schwarz inequality, Young's inequality, and $L$-smoothness of $f$, implies that, for all $\rho > 0$,
\begin{align*}
-\eta X_t  
&= \langle \nabla f (\boldsymbol z_t ), - \eta \nabla f (\bm{\theta}_t)\rangle\\
&= \langle \nabla f (\boldsymbol z_t ) - \nabla f (\bm{\theta}_t), - \eta \nabla f (\bm{\theta}_t)\rangle - \eta \|\nabla f (\bm{\theta}_t)\|^2 \\
&\leq 
(\sqrt{\eta} \| \nabla f (\boldsymbol z_t ) - \nabla f (\bm{\theta}_t)\|) (\sqrt{\eta} \|\nabla f (\bm{\theta}_t)\|) - \eta \|\nabla f (\bm{\theta}_t)\|^2 \\
&\le \frac{\rho \eta}{2} \| \nabla f(\boldsymbol z_t) - \nabla f(\boldsymbol \theta_t) \|^2
+ \frac{\eta}{2\rho} \| \nabla f(\bm{\theta}_t) \|^2 - \eta \| \nabla f(\bm{\theta}_t) \|^2\\
&\le \frac{\rho \eta L^2}{2} \| \boldsymbol z_t - \boldsymbol \theta_t \|^2
+ \eta \left( \frac{1}{2\rho} - 1 \right) \| \nabla f(\bm{\theta}_t) \|^2.
\end{align*}
The definitions of $\bm{z}_t$ and $\bm{\theta}_t$ ($= \bm{\theta}_{t-1} -\eta \bm{m}_{t-1}$) ensure that, for all $t \geq 1$,
\begin{align*}
\boldsymbol z_t 
= \frac{1}{1-\beta}\boldsymbol \theta_t - \frac{\beta}{1-\beta}
(\boldsymbol \theta_t + \eta \boldsymbol m_{t-1} )
= \boldsymbol \theta_t - \frac{\beta}{1-\beta}\eta \boldsymbol m_{t-1}.
\end{align*}
From $\bm{m}_{-1} := \bm{0}$ and $\bm{z}_0 = \bm{\theta}_0$, we have that, for all $t \in \{0\} \cup \mathbb{N}$,
\begin{align*}
\boldsymbol z_t 
= \boldsymbol \theta_t - \frac{\beta}{1-\beta}\eta \boldsymbol m_{t-1}.
\end{align*}
Accordingly, we have that 
\begin{align}\label{x_t}
-\eta X_t
\leq
\frac{\rho \eta^3 L^2}{2} \left(\frac{\beta}{1-\beta} \right)^2 \| \boldsymbol m_{t-1} \|^2
+ \eta \left( \frac{1}{2\rho} - 1 \right) \| \nabla f(\bm{\theta}_t) \|^2.
\end{align}
From $\| \bm{\theta}_1 + \bm{\theta}_2 \|^2 \leq 2 \|\bm{\theta}_1\|^2 + 2 \|\bm{\theta}_2\|^2$, we have that
\begin{align}\label{z_t}
\Vert \boldsymbol m_{t-1} \rVert^2
&\le 2 \left\Vert \boldsymbol m_{t-1} - (1-\beta)\sum_{i=0}^{t-1} \beta^{t-1-i}\nabla f (\bm{\theta}_i) \right\rVert^2 
+ 2 \underbrace{\left\Vert (1-\beta)\sum_{i=0}^{t-1} \beta^{t-1-i}\nabla f (\bm{\theta}_i) \right\rVert^2}_{Z_t}.
\end{align}
Moreover, for all $t \geq 2$,
\begin{align}\label{w_t}
\frac{1}{(1- \beta^{t-1})^2} Z_t
&\leq 
2 \|\nabla f (\bm{\theta}_t)\|^2 
+ 2 \underbrace{\left\Vert \frac{1-\beta}{1-\beta^{t-1}}\sum_{i=0}^{t-1}\beta^{t-1-i} \nabla f (\bm{\theta}_i) - \nabla f (\bm{\theta}_t)  \right\rVert^2}_{W_t}.
\end{align}
Meanwhile, we also have that
\begin{align*}
&\left\Vert \frac{1-\beta}{1-\beta^{t}}\sum_{i=0}^{t}\beta^{t-i} \nabla f (\bm{\theta}_i) - \nabla f (\bm{\theta}_t)  \right\rVert^2\\
&= \left\Vert \frac{1-\beta}{1-\beta^t}\left(\beta^{t} \nabla f(\bm{\theta}_0) + \beta^{t-1} \nabla f(\bm{\theta}_1) + \cdots + \beta^{t- (t-1)}\nabla f(\bm{\theta}_{t-1}) + \nabla f(\bm{\theta}_{t}) \right) - \nabla f(\bm{\theta}_{t}) \right\rVert^2 \\
&= 
\left\Vert \frac{1-\beta}{1-\beta^t}
\left(\beta^{t} \nabla f(\bm{\theta}_0) + \beta^{t-1} \nabla f(\bm{\theta}_1) + \cdots + \beta^{t- (t-1)}\nabla f(\bm{\theta}_{t-1}) \right) + \left(\frac{1-\beta}{1-\beta^t}-1\right) \nabla f(\bm{\theta}_t) \right\rVert^2 \\
&= \left\Vert \frac{1-\beta}{1-\beta^t}
\left(\beta^{t} \nabla f(\bm{\theta}_0) + \beta^{t-1} \nabla f(\bm{\theta}_1) + \cdots + \beta^{t- (t-1)}\nabla f(\bm{\theta}_{t-1}) \right)
 - \frac{\beta - \beta^t}{1 - \beta^t} \nabla f(\bm{\theta}_t) \right\rVert^2 \\
&= \left\Vert \frac{1-\beta}{1-\beta^t}\beta\left(\beta^{t-1} \nabla f(\bm{\theta}_0) + \cdots + \beta^{t- (t-1)}\nabla f (\bm{\theta}_{t-2})
+ \nabla f (\bm{\theta}_{t-1})  \right) 
- \frac{1-\beta^{t-1}}{1-\beta^t}\beta \nabla f(\bm{\theta}_t) \right\rVert^2 \\
&= \beta^2 \left(\frac{1-\beta^{t-1}}{1-\beta^t}\right)^2 
\left\Vert \frac{1-\beta}{1-\beta^{t-1}}\sum_{i=0}^{t-1} \beta^{t-1-i}\nabla f(\bm{\theta}_i) - \nabla f(\bm{\theta}_t) \right\rVert^2\\
&= \beta^2 \left(\frac{1-\beta^{t-1}}{1-\beta^t}\right)^2 W_t,
\end{align*}
which implies that, for all $t \geq 2$,
\begin{align}\label{w_t_1}
W_t 
=
\frac{(1 - \beta^t)^2}{\beta^2 (1 - \beta^{t-1})^2}
\left\Vert \frac{1-\beta}{1-\beta^{t}}\sum_{i=0}^{t}\beta^{t-i} \nabla f (\bm{\theta}_i) - \nabla f (\bm{\theta}_t)  \right\rVert^2.
\end{align}
From (\ref{z_t}), (\ref{w_t}), and (\ref{w_t_1}), 
\begin{align*}
\Vert \boldsymbol m_{t-1} \rVert^2
&\le 2 \left\Vert \boldsymbol m_{t-1} - (1-\beta)\sum_{i=0}^{t-1} \beta^{t-1-i}\nabla f (\bm{\theta}_i) \right\rVert^2 
+ 2 
\Bigg\{
2 (1- \beta^{t-1})^2 \|\nabla f (\bm{\theta}_t)\|^2\\ 
&\quad 
+ 2 (1- \beta^{t-1})^2 
\frac{(1 - \beta^t)^2}{\beta^2 (1 - \beta^{t-1})^2}
\left\Vert \frac{1-\beta}{1-\beta^{t}}\sum_{i=0}^{t}\beta^{t-i} \nabla f (\bm{\theta}_i) - \nabla f (\bm{\theta}_t)  \right\rVert^2
\Bigg\}\\
&= 
2 \left\Vert \boldsymbol m_{t-1} - (1-\beta)\sum_{i=0}^{t-1} \beta^{t-1-i}\nabla f (\bm{\theta}_i) \right\rVert^2
+ 4 (1- \beta^{t-1})^2 \|\nabla f (\bm{\theta}_t)\|^2\\
&\quad 
+   
\frac{4(1 - \beta^t)^2}{\beta^2}
\left\Vert \frac{1-\beta}{1-\beta^{t}}\sum_{i=0}^{t}\beta^{t-i} \nabla f (\bm{\theta}_i) - \nabla f (\bm{\theta}_t)  \right\rVert^2
\end{align*}
Hence, (\ref{x_t}) ensures that 
\begin{align*}
-\eta X_t
&\leq
\frac{\rho \eta^3 L^2}{2} \left(\frac{\beta}{1-\beta} \right)^2 
\Bigg\{
2 \left\Vert \boldsymbol m_{t-1} - (1-\beta)\sum_{i=0}^{t-1} \beta^{t-1-i}\nabla f (\bm{\theta}_i) \right\rVert^2
+ 4 (1- \beta^{t-1})^2 \|\nabla f (\bm{\theta}_t)\|^2\\
&\quad 
+   
\frac{4(1 - \beta^t)^2}{\beta^2}
\left\Vert \frac{1-\beta}{1-\beta^{t}}\sum_{i=0}^{t}\beta^{t-i} \nabla f (\bm{\theta}_i) - \nabla f (\bm{\theta}_t)  \right\rVert^2
\Bigg\}
+ \eta \left( \frac{1}{2\rho} - 1 \right) \| \nabla f(\bm{\theta}_t) \|^2\\
&=
\rho \eta^3 L^2 \left(\frac{\beta}{1-\beta} \right)^2 \left\Vert \boldsymbol m_{t-1} - (1-\beta)\sum_{i=0}^{t-1} \beta^{t-1-i}\nabla f (\bm{\theta}_i) \right\rVert^2\\
&\quad + 2 \rho \eta^3 L^2 \left(\frac{\beta}{1-\beta} \right)^2 (1- \beta^{t-1})^2 \|\nabla f (\bm{\theta}_t)\|^2\\
&\quad 
+ 2 \rho \eta^3 L^2 \left(\frac{1}{1-\beta} \right)^2 
(1 - \beta^t)^2
\left\Vert \frac{1-\beta}{1-\beta^{t}}\sum_{i=0}^{t}\beta^{t-i} \nabla f (\bm{\theta}_i) - \nabla f (\bm{\theta}_t)  \right\rVert^2\\
&\quad + \eta \left( \frac{1}{2\rho} - 1 \right) \| \nabla f(\bm{\theta}_t) \|^2.
\end{align*}
Let us take the total expectation on both sides of the above inequality. Lemma \ref{lem:4} then guarantees that, for all $t \geq 2$ and for all $\rho > 0$, 
\begin{align}\label{x_t_2}
\begin{split}
- \eta \mathbb{E}[X_t]
&\leq
\rho \eta^3 L^2 \left(\frac{\beta}{1-\beta} \right)^2 
(1 - \beta)^2 \sigma^2 \sum_{i=0}^{t-1} \frac{\beta^{2(t-1-i)}}{b_i}\\
&\quad + 2 \rho \eta^3 L^2 \left(\frac{\beta}{1-\beta} \right)^2 (1- \beta^{t-1})^2 
\mathbb{E}[\|\nabla f (\bm{\theta}_t)\|^2]\\
&\quad 
+ 2 \rho \eta^3 L^2 \left(\frac{1}{1-\beta} \right)^2 
(1 - \beta^t)^2
\mathbb{E} \left[\left\Vert \frac{1-\beta}{1-\beta^{t}}\sum_{i=0}^{t}\beta^{t-i} \nabla f (\bm{\theta}_i) - \nabla f (\bm{\theta}_t)  \right\rVert^2 \right]\\
&\quad + \eta \left( \frac{1}{2\rho} - 1 \right) \mathbb{E}[\| \nabla f(\bm{\theta}_t) \|^2]\\
&\leq
\rho \eta^3 L^2 \beta^2  \sigma^2 \sum_{i=0}^{t-1} \frac{\beta^{2(t-1-i)}}{b_i}
+ 2 \rho \eta^3 L^2 \left(\frac{\beta}{1-\beta} \right)^2 \mathbb{E} [\|\nabla f (\bm{\theta}_t)\|^2]\\
&\quad 
+ 2 \rho \eta^3 L^2 \left(\frac{1}{1-\beta} \right)^2 {(1 - \beta^t)^2}
\mathbb{E} \left[\left\Vert \frac{1-\beta}{1-\beta^{t}}\sum_{i=0}^{t}\beta^{t-i} \nabla f (\bm{\theta}_i) - \nabla f (\bm{\theta}_t)  \right\rVert^2 \right]\\
&\quad + \eta \left( \frac{1}{2\rho} - 1 \right) \mathbb{E}[\| \nabla f(\bm{\theta}_t) \|^2].
\end{split}
\end{align}
Moreover, Proposition \ref{prop:1} guarantees that 
\begin{align}\label{e_2}
\begin{split}
\mathbb{E} [Y_t ]
&= 
\mathbb{E}_{\bm{\xi}_t} \left[ \left\|\nabla f_{B_t} (\bm{\theta}_t) - \nabla f(\bm{\theta}_t) + \nabla f(\bm{\theta}_t) \right\|^2 \Big| \hat{\bm{\xi}}_{t-1} \right]\\
&=
\mathbb{E} [ \|\nabla f_{B_t} (\bm{\theta}_t) - \nabla f(\bm{\theta}_t) \|^2 ]
+ 2 \mathbb{E} [ \langle \nabla f_{B_t} (\bm{\theta}_t) - \nabla f(\bm{\theta}_t),
\nabla f(\bm{\theta}_t) \rangle ]
+ \mathbb{E}[ \| \nabla f(\bm{\theta}_t) \|^2 ]\\
&\leq \frac{\sigma^2}{b_t} 
+ \mathbb{E} [\| \nabla f(\bm{\theta}_t) \|^2]. 
\end{split}
\end{align}
Therefore, from (\ref{eq:39}), (\ref{x_t_2}), and (\ref{e_2}), for all $t \geq 2$ and for all $\rho > 0$, we have
\begin{align*}
\mathbb{E} [f (\boldsymbol z_{t+1} )] 
&\leq 
\mathbb{E} [f (\boldsymbol z_t )]
- \eta \mathbb{E}[X_t] 
+ \frac{L\eta^2}{2} \mathbb{E}[Y_t]\\
&\leq
\mathbb{E} [f (\boldsymbol z_t )]
+
\rho \eta^3 L^2 \beta^2  \sigma^2 \sum_{i=0}^{t-1} \frac{\beta^{2(t-1-i)}}{b_i}\\
&\quad + 2 \rho \eta^3 L^2 \left(\frac{\beta}{1-\beta} \right)^2 \mathbb{E}[\|\nabla f (\bm{\theta}_t)\|^2]
+ \eta \left( \frac{1}{2\rho} - 1 \right) \mathbb{E}[\| \nabla f(\bm{\theta}_t) \|^2]\\
&\quad 
+ 2 \rho \eta^3 L^2 \left(\frac{1}{1-\beta} \right)^2 {(1 - \beta^t)^2} 
\left\Vert \frac{1-\beta}{1-\beta^{t}}\sum_{i=0}^{t}\beta^{t-i} \nabla f (\bm{\theta}_i) - \nabla f (\bm{\theta}_t)  \right\rVert^2\\
&\quad + 
\frac{L\eta^2}{2}
\left(  
\frac{\sigma^2}{b_t} 
+ \mathbb{E} [\| \nabla f(\bm{\theta}_t) \|^2]
\right)\\
&= 
\mathbb{E} [f (\boldsymbol z_t )]
+
\underbrace{\left\{
2 \rho \eta^3 L^2 \left(\frac{\beta}{1-\beta} \right)^2
+
\eta \left( \frac{1}{2\rho} - 1 \right)
+
\frac{L\eta^2}{2}
\right\}}_{A} \mathbb{E} [\| \nabla f(\bm{\theta}_t) \|^2]\\
&\quad 
+
L \eta^2 \sigma^2 \underbrace{\left( \rho \eta L \beta^2  \sum_{i=0}^{t-1} \frac{\beta^{2(t-1-i)}}{b_i}
+ \frac{1}{2 b_t}
\right)}_{B_t}\\
&\quad 
+ \underbrace{2 \rho \eta^3 L^2}_{C} \left(\frac{1}{1-\beta} \right)^2 {(1 - \beta^t)^2}
\left\Vert \frac{1-\beta}{1-\beta^{t}}\sum_{i=0}^{t}\beta^{t-i} \nabla f (\bm{\theta}_i) - \nabla f (\bm{\theta}_t)  \right\rVert^2.
\end{align*}
The setting $\rho := \frac{1}{2 L \eta}$ implies that, for all $t \geq 2$,
\begin{align*}
A 
&= \frac{L^2 \eta^3 }{L \eta} \left(\frac{\beta}{1-\beta} \right)^2
+
\eta \left( L \eta - 1 \right)
+
\frac{L\eta^2}{2}
= 
L \eta^2\left(\frac{\beta}{1-\beta} \right)^2
+
\eta \left( L \eta - 1 \right)
+
\frac{L\eta^2}{2}\\
&= 
L \left\{ \left(\frac{\beta}{1-\beta} \right)^2 + \frac{3}{2} \right\} \eta^2 - \eta,\\
B_t 
&= \frac{L \eta \beta^2}{2 L \eta}  \sum_{i=0}^{t-1} \frac{\beta^{2(t-1-i)}}{b_i}
+ \frac{1}{2 b_t}
= \frac{1}{2} \left( \beta^2  \sum_{i=0}^{t-1} \frac{\beta^{2(t-1-i)}}{b_i}
+ \frac{1}{b_t}
\right), \text{ } 
C = \frac{2 L^2 \eta^3}{2 L \eta} = L \eta^2. 
\end{align*}
When $t = 0$, from $\|\bm{m}_{-1}\| = 0$ and $\sum_{i=0}^{-1} a_i := 0$, the assertion in Lemma \ref{prop:4} holds. When $t = 1$, assuming $\frac{1}{1 - \beta^0} := 1$, the assertion in Lemma \ref{prop:4} again holds. This completes the proof.
\end{proof}

Using Lemma \ref{prop:4}, we have the following lemma.

\begin{mylemma}\label{prop:5}
Suppose that Assumption \ref{assum:1} holds and $(\bm{\theta}_t)$ is the sequence generated by Algorithm \ref{algo:1} with $\eta > 0$ satisfying
\begin{align*}
\eta \le \frac{1-\beta}{2\sqrt{2}\sqrt{\beta+\beta^2} L}.
\end{align*}
Let $(\bm{z}_t)$ be the sequence defined as in Lemma \ref{prop:4} and define $L_t \in \mathbb{R}$ for all $t \in \{0\} \cup \mathbb{N}$ as
\begin{align*}
L_t := 
\begin{cases}
f(\boldsymbol z_0) - f^\star &\text{ } (t = 0)\\
f(\boldsymbol z_1) - f^\star &\text{ } (t = 1)\\
f(\boldsymbol z_t) - f^\star + \sum_{i=1}^{t-1} c_i \left\Vert \boldsymbol \theta_{t+1-i}-\boldsymbol \theta_{t-i} \right\rVert^2 &\text{ } (t \geq 2),
\end{cases}
\end{align*} 
where $f^\star$ is the minimum value of $f$ over $\mathbb{R}^d$ and $(c_i) \subset \mathbb{R}_{++}$ is defined by

\resizebox{\textwidth}{!}{$
c_1 = \dfrac{(\beta+\beta^2) L^3 \eta^2}{(1-\beta)^2 \{ (1-\beta)^2 -4 (\beta+\beta^2)L^2 \eta^2 \}} \text{ and } c_{i+1} = c_i - \left(4 c_1 \eta^2 + \frac{L\eta^2}{{(1-\beta)^2}}\right)\beta^i\left(i+\frac{\beta}{1-\beta}\right)L^2 \quad (i \in [t-1]).$}
Then, for all $t \in \{0\} \cup \mathbb{N}$,
\begin{align*}
\mathbb{E} [L_{t+1}-L_t ]
&\leq  
\eta \left[ L \left\{ \left(\frac{\beta}{1-\beta} \right)^2 + \frac{3}{2} \right\} \eta - 1 + 4 c_1 \eta \right] \mathbb{E} [\| \nabla f(\bm{\theta}_t) \|^2 ] \nonumber \\
&\quad + \frac{L \sigma^2 \eta^2}{2} \left\{ \beta^2 \sum_{i=0}^{t-1} \frac{\beta^{2(t-1-i)}}{b_i} + \frac{1}{b_t}\right\} 
+ 2 c_1 \eta^2 (1-\beta)^2 \sigma^2 \sum_{i=0}^t \frac{\beta^{2(t-i)}}{b_i}, 
\end{align*}
where we assume that $\sum_{i=0}^{-1} a_i := 0$ for some $a_i \in \mathbb{R}$.
\end{mylemma}

\begin{proof}
Let $t \geq 2$. The definition of $L_t$ implies that
\begin{align*}%\label{eq:43}
\mathbb{E}[L_{t+1} - L_t]
&= 
\mathbb{E} [f(\boldsymbol z_{t+1}) - f(\boldsymbol z_{t})]
+ \mathbb{E} \left[ \sum_{i=1}^{t} c_i \left\Vert \boldsymbol \theta_{t+2-i}-\boldsymbol \theta_{t+1-i} \right\rVert^2 - \sum_{i=1}^{t-1} c_i \left\Vert \boldsymbol \theta_{t+1-i}-\boldsymbol \theta_{t-i} \right\rVert^2 \right]\\
&=
\mathbb{E} [f(\boldsymbol z_{t+1}) - f(\boldsymbol z_{t})]
+ 
\mathbb{E} [ c_1 \left\Vert \boldsymbol \theta_{t+1}-\boldsymbol \theta_{t} \right\rVert^2 ]\\
&\quad + \mathbb{E} \left[ \sum_{i=1}^{t-1} c_{i+1} \left\Vert \boldsymbol \theta_{t+1-i}-\boldsymbol \theta_{t-i} \right\rVert^2 - \sum_{i=1}^{t-1} c_i \left\Vert \boldsymbol \theta_{t+1-i}-\boldsymbol \theta_{t-i} \right\rVert^2 \right]\\
&=
\mathbb{E} [f(\boldsymbol z_{t+1}) - f(\boldsymbol z_{t})]
+ 
\mathbb{E} [ c_1 \left\Vert \boldsymbol \theta_{t+1}-\boldsymbol \theta_{t} \right\rVert^2 ]\\
&\quad + \sum_{i=1}^{t-1} (c_{i+1} - c_i) \mathbb{E} [ \left\Vert \boldsymbol \theta_{t+1-i}-\boldsymbol \theta_{t-i} \right\rVert^2 ].
\end{align*}
Lemma \ref{prop:4} thus ensures that
\begin{align}\label{eq:43}
\mathbb{E}[L_{t+1} - L_t]
&\le 
\eta \left[ L \left\{ \left(\frac{\beta}{1-\beta} \right)^2 + \frac{3}{2} \right\} \eta - 1 \right] \mathbb{E} [\| \nabla f(\bm{\theta}_t) \|^2 ] \nonumber + \frac{L \sigma^2 \eta^2}{2} \left\{ \beta^2 \sum_{i=0}^{t-1} \frac{\beta^{2(t-1-i)}}{b_i} + \frac{1}{b_t}\right\} \\ 
&\quad + \left(\frac{1}{1-\beta}\right)^2 L\eta^2 {(1 - \beta^t)^2} \mathbb{E}\left[\left\Vert \frac{1-\beta}{1-\beta^t}\sum_{i=0}^t \beta^{t-i} \nabla f(\bm{\theta}_i) - \nabla f(\bm{\theta}_t) \right\rVert^2\right]\nonumber \\
&\quad + 
c_1 \eta^2 \mathbb{E} [\left\Vert \boldsymbol m_{t} \right\rVert^2 ]
+ \sum_{i=1}^{t-1} (c_{i+1} - c_i) \mathbb{E} [ \left\Vert \boldsymbol \theta_{t+1-i}-\boldsymbol \theta_{t-i} \right\rVert^2 ].
\end{align}
A similar discussion to the one showing (\ref{z_t}) and (\ref{w_t}) ensures that
\begin{align*}
\Vert \boldsymbol m_{t} \rVert^2
&\le 2 \left\Vert \boldsymbol m_{t} - (1-\beta)\sum_{i=0}^{t} \beta^{t-i}\nabla f (\bm{\theta}_i) \right\rVert^2\\ 
&\quad 
+ 2 \left\{ 2 (1- \beta^{t})^2\|\nabla f (\bm{\theta}_t)\|^2 
+ 2 (1- \beta^{t})^2 \left\Vert \frac{1-\beta}{1-\beta^{t}}\sum_{i=0}^{t}\beta^{t-i} \nabla f (\bm{\theta}_i) - \nabla f (\bm{\theta}_t)  \right\rVert^2 \right\},
\end{align*}
which, together with Lemma \ref{lem:4} and $1 - \beta^t \leq 1$, implies that
\begin{align*}
\mathbb{E}[\|\bm{m}_t\|^2]
&\leq
2 (1-\beta)^2 \sigma^2 \sum_{i=0}^t \frac{\beta^{2(t-i)}}{b_i}
+ 4 (1- \beta^{t})^2 \|\nabla f (\bm{\theta}_t)\|^2\\
&\quad + 4 (1- \beta^{t})^2 \left\Vert \frac{1-\beta}{1-\beta^{t}}\sum_{i=0}^{t}\beta^{t-i} \nabla f (\bm{\theta}_i) - \nabla f (\bm{\theta}_t)  \right\rVert^2\\
&\leq
2 (1-\beta)^2 \sigma^2 \sum_{i=0}^t \frac{\beta^{2(t-i)}}{b_i}
+ 4 \|\nabla f (\bm{\theta}_t)\|^2\\
&\quad + 4 (1- \beta^{t})^2 
\underbrace{\left\Vert \frac{1-\beta}{1-\beta^{t}}\sum_{i=0}^{t}\beta^{t-i} \nabla f (\bm{\theta}_i) - \nabla f (\bm{\theta}_t)  \right\rVert^2}_{V_t}.
\end{align*}
Lemma 2 in \citep{NEURIPS2020_d3f5d4de} guarantees that 
\begin{align*}
\mathbb{E}[V_t] \leq \sum_{i=1}^{t-1} a_{t,i} \mathbb{E}[\|\bm{\theta}_{i+1} - \bm{\theta}_i\|^2], \text{ where } 
a_{t,i} := \frac{L^2 \beta^{t-i}}{1 - \beta^t}
\left( t - i + \frac{\beta}{1-\beta} \right),
\end{align*}
which implies that 
\begin{align}\label{a_t_i}
\mathbb{E}[V_t] \leq \sum_{i=1}^{t-1} a_{t,t-i} \mathbb{E}[\|\bm{\theta}_{t+1-i} - \bm{\theta}_{t-i} \|^2], \text{ where } 
a_{t,t-i} := \frac{L^2 \beta^{i}}{1 - \beta^t}
\left( i + \frac{\beta}{1-\beta} \right).
\end{align}
Moreover, $c_1 > 0$ when $\eta \le \frac{1-\beta}{2\sqrt{2}L\sqrt{\beta+\beta^2}}$. Hence, \eqref{eq:43} ensures that 
\begin{align*}
&\mathbb{E}[L_{t+1} - L_t]\\
&\le 
\eta \left[ L \left\{ \left(\frac{\beta}{1-\beta} \right)^2 + \frac{3}{2} \right\} \eta - 1 \right] \mathbb{E} [\| \nabla f(\bm{\theta}_t) \|^2 ] \\
&\quad + \frac{L \sigma^2 \eta^2}{2} \left\{ \beta^2 \sum_{i=0}^{t-1} \frac{\beta^{2(t-1-i)}}{b_i} + \frac{1}{b_t}\right\}  
+ \left(\frac{1}{1-\beta}\right)^2 L\eta^2 {(1- \beta^{t})^2} \mathbb{E}[V_t] \\
&\quad + 
c_1 \eta^2 \left\{ 2 (1-\beta)^2 \sigma^2 \sum_{i=0}^t \frac{\beta^{2(t-i)}}{b_i}
+ 4 \mathbb{E}[\|\nabla f (\bm{\theta}_t)\|^2]
+ 4 (1- \beta^{t})^2 \mathbb{E}[V_t] \right\} \\
&\quad + \sum_{i=1}^{t-1} (c_{i+1} - c_i) \mathbb{E} [ \left\Vert \boldsymbol \theta_{t+1-i}-\boldsymbol \theta_{t-i} \right\rVert^2 ]\\
&= 
\eta \left[ L \left\{ \left(\frac{\beta}{1-\beta} \right)^2 + \frac{3}{2} \right\} \eta - 1 + 4 c_1 \eta \right] \mathbb{E} [\| \nabla f(\bm{\theta}_t) \|^2 ] \nonumber \\
&\quad + \frac{L \sigma^2 \eta^2}{2} \left\{ \beta^2 \sum_{i=0}^{t-1} \frac{\beta^{2(t-1-i)}}{b_i} + \frac{1}{b_t}\right\} 
+ 2 c_1 \eta^2 (1-\beta)^2 \sigma^2 \sum_{i=0}^t \frac{\beta^{2(t-i)}}{b_i}\\ 
&\quad + 
\sum_{i=1}^{t-1} 
\underbrace{\left\{
\left(\frac{1}{1-\beta}\right)^2 L\eta^2 {(1- \beta^{t})^2} a_{t,t-i} 
+ 
4 c_1 \eta^2 (1- \beta^{t})^2  a_{t,t-i} 
+ (c_{i+1} - c_i) 
\right\}}_{N_{t,i}}
\mathbb{E} [ \left\Vert \boldsymbol \theta_{t+1-i}-\boldsymbol \theta_{t-i} \right\rVert^2 ].
\end{align*}
Finally, we prove that $N_{t,i} \leq 0$. From the definition of $a_{t,t-i}$ in \eqref{a_t_i} and
\begin{align*}
    {c_{i+1} = c_i - \left(4 c_1 \eta^2 + \frac{L\eta^2}{(1-\beta)^2}\right)\beta^i\left(i+\frac{\beta}{{1-\beta}}\right)L^2,}
\end{align*}
{we have that 
\begin{align*}
    N_{t,i} 
    &=
    \left(\frac{1}{1-\beta}\right)^2 L\eta^2
    {(1 - \beta^t)^2} a_{t,t-i} 
    + 
    4 c_1 \eta^2 (1- \beta^{t})^2  a_{t,t-i} 
    + (c_{i+1} - c_i)\\
    &=
    \left\{\left(\frac{1}{1-\beta}\right)^2 L\eta^2
    {(1- \beta^{t})^2} + 
    4 c_1 \eta^2 (1- \beta^{t})^2
    \right\}
    \frac{L^2 \beta^{i}}{1 - \beta^t}\left( i + \frac{\beta}{1-\beta} \right)\\
    &\quad - \left(4 c_1 \eta^2 
    + \frac{L\eta^2}{{(1-\beta)^2}}\right)\beta^i\left(i+\frac{\beta}{1-\beta}\right)L^2\\
    &= 
    L^2 \eta^2 \beta^{i} \left( i + \frac{\beta}{1-\beta} \right)
    \left[
    \left\{
    \frac{1 - \beta^t}{(1 - \beta)^2} L
    + 4 c_1 (1 - \beta^t) \right\}
    - \left\{ 4 c_1 + \frac{L}{(1 - \beta)^2} \right\}
    \right]\\
    &= 
    L^2 \eta^2 \beta^{i} \left( i + \frac{\beta}{1-\beta} \right)
    \left[
    \frac{L}{(1 - \beta)^2} (1 - \beta^t -1)
    + 4 c_1 (1 - \beta^t -1 )
    \right]\\
    &= 
    - L^2 \eta^2 \beta^{i} \left( i + \frac{\beta}{1-\beta} \right)
    \left[
    \frac{L}{(1 - \beta)^2} 
    + 4 c_1 
    \right] \beta^t.
\end{align*}
From 
\begin{align}\label{condition_1}
\eta \le \frac{1-\beta}{2\sqrt{2}L\sqrt{\beta+\beta^2}}
\text{ and } 
c_1 = \dfrac{(\beta+\beta^2) L^3 \eta^2}{(1-\beta)^2 \{ (1-\beta)^2 -4 (\beta+\beta^2)L^2 \eta^2 \}},
\end{align}
we have that
\begin{align*}
c_1 = \frac{\eta^2 L^3 \frac{\beta+\beta^2}{(1-\beta)^4}}{1-4\eta^2 L^2 \frac{\beta+\beta^2}{(1-\beta)^2}} > 0. 
\end{align*}
Accordingly, 
\begin{align*}
    N_{t,i} = - L^2 \eta^2 \beta^{i} \left( i + \frac{\beta}{1-\beta} \right)
    \left[
    \frac{L}{(1 - \beta)^2} 
    + 4 c_1 
    \right] \beta^t
    < 0.
\end{align*}
}
This completes the proof.
\end{proof}

Lemma \ref{prop:5} leads to the following.

\begin{mylemma}\label{prop:6}
Suppose that Assumption \ref{assum:1} holds and $(\bm{\theta}_t)$ is the sequence generated by Algorithm \ref{algo:1} with $\eta > 0$ satisfying
\begin{align*}
\eta 
\leq \max \left\{ \frac{1-\beta}{2\sqrt{2}\sqrt{\beta+\beta^2}L},
\frac{(1-\beta)^2}{(5\beta^2 - 6\beta + 5)L} \right\}.
\end{align*}
Then, for all $T \geq 1$,
\begin{align*}
\frac{1}{T} \sum_{t=0}^{T-1} \mathbb{E} [\| \nabla f(\bm{\theta}_t) \|^2 ]
\leq
\frac{2 (f(\bm{\theta}_0) - f^\star)}{\eta T}
+
\frac{2 L \eta \sigma^2}{T} \sum_{t=0}^{T-1} \sum_{i=0}^t \frac{\beta^{2(t-i)}}{b_i}.
\end{align*}
\end{mylemma}

\begin{proof}
Lemma \ref{prop:5} guarantees that, for all $t \in \{0\} \cup \mathbb{N}$,
\begin{align*}
&\underbrace{-\eta \left[ L \left\{ \left(\frac{\beta}{1-\beta} \right)^2 + \frac{3}{2} \right\} \eta - 1 + 4 c_1 \eta \right]}_{D} \mathbb{E} [\| \nabla f(\bm{\theta}_t) \|^2 ]\\ 
&\quad \leq \mathbb{E} [L_{t+1}-L_t ] 
+ \underbrace{\frac{L \sigma^2 \eta^2}{2} \left\{ \beta^2 \sum_{i=0}^{t-1} \frac{\beta^{2(t-1-i)}}{b_i} + \frac{1}{b_t}\right\} 
+ 2 c_1 \eta^2 (1-\beta)^2 \sigma^2 \sum_{i=0}^t \frac{\beta^{2(t-i)}}{b_i}}_{U_t}, 
\end{align*}
where $\beta \in [0,1)$, and $\eta$ and $c_1$ satisfy \eqref{condition_1}. Summing the above inequality from $t = 0$ to $t = T-1$ gives
\begin{align*}
D \sum_{t=0}^{T-1} \mathbb{E} [\| \nabla f(\bm{\theta}_t) \|^2 ]
\leq 
\sum_{t=0}^{T-1} \mathbb{E} [L_{t}-L_{t+1} ] 
+ \sum_{t=0}^{T-1} U_t
= 
\mathbb{E}[L_{0} - L_{T} ] 
+ \sum_{t=0}^{T-1} U_t,
\end{align*}
which, together with $L_T \geq 0$, implies that
\begin{align}\label{key_1}
D \sum_{t=0}^{T-1} \mathbb{E} [\| \nabla f(\bm{\theta}_t) \|^2 ]
\leq 
L_{0} 
+ \sum_{t=0}^{T-1} U_t.
\end{align}
From \eqref{condition_1}, we have
\begin{align*}
c_1 = \frac{\eta^2 L^3 \frac{\beta+\beta^2}{(1-\beta)^4}}{1-4\eta^2 L^2 \frac{\beta+\beta^2}{(1-\beta)^2}} \text{ and }
\eta^2 L^2 \frac{\beta+\beta^2}{(1-\beta)^2} \le \frac{1}{8},
\end{align*}
which implies that
\begin{align}\label{condition_c_1}
c_1 
\le \frac{L}{8(1-\beta)^2} \left(1 - \frac{4}{8}  \right)^{-1}
= \frac{L}{4(1-\beta)^2}. 
\end{align}
Accordingly, from \eqref{condition_c_1} and $\eta \le \frac{(1-\beta)^2}{L(5\beta^2 - 6\beta + 5)}$, we have that
\begin{align}\label{d_1}
\begin{split}
D 
&= 
- L \left\{ \left(\frac{\beta}{1-\beta} \right)^2 + \frac{3}{2} \right\} \eta^2 + \eta - 4 c_1 \eta^2 
\geq 
- L \left\{ \left(\frac{\beta}{1-\beta} \right)^2 + \frac{3}{2} \right\} \eta^2 + \eta 
- \frac{L \eta^2}{(1 - \beta)^2}\\
&= - L \eta^2 \frac{5 \beta^2 - 6 \beta + 5}{2(1-\beta)^2} + \eta
\geq - \frac{\eta}{2} + \eta = \frac{\eta}{2} > 0.
\end{split}
\end{align}
Moreover, from \eqref{condition_c_1}, we have
\begin{align}\label{u_1}
\begin{split}
U_t 
&=
\frac{L \sigma^2 \eta^2}{2} \left\{ \beta^2 \sum_{i=0}^{t-1} \frac{\beta^{2(t-1-i)}}{b_i} + \frac{1}{b_t}\right\} 
+ 2 c_1 \eta^2 (1-\beta)^2 \sigma^2 \sum_{i=0}^t \frac{\beta^{2(t-i)}}{b_i}\\
&= 
\sigma^2 \left\{\frac{L \eta^2}{2} + 2 c_1 \eta^2 (1-\beta)^2  \right\}
\sum_{i=0}^{t} \frac{\beta^{2(t-i)}}{b_i}\\
&\le 
\sigma^2 \left( \frac{L \eta^2}{2} + \frac{L \eta^2}{2}\right) \sum_{i=0}^t \frac{\beta^{2(t-i)}}{b_i}
= L \eta^2 \sigma^2 \sum_{i=0}^t \frac{\beta^{2(t-i)}}{b_i}.
\end{split}
\end{align}
Therefore, \eqref{key_1}, \eqref{d_1}, and \eqref{u_1} ensure that
\begin{align*}
\frac{1}{T} \sum_{t=0}^{T-1} \mathbb{E} [\| \nabla f(\bm{\theta}_t) \|^2 ]
\leq 
\frac{L_{0}}{D T} 
+ 
\frac{1}{T} \sum_{t=0}^{T-1} U_t
\leq
\frac{2 L_{0}}{\eta T}
+
\frac{2 L \eta \sigma^2}{T} \sum_{t=0}^{T-1} \sum_{i=0}^t \frac{\beta^{2(t-i)}}{b_i}.
\end{align*}
This completes the proof.
\end{proof}

\begin{proof}[Proof of Theorem \ref{thm:1}] Let $b_t$ be defined by \eqref{exponential_bs}, i.e., for all $m \in [0:M]$ and all $t \in S_m = \mathbb{N} \cap [\sum_{k=0}^{m-1} K_{k} E_{k}, \sum_{k=0}^{m} K_{k} E_{k})$ ($S_0 := \mathbb{N} \cap [0, K_0 E_0)$), 
\begin{align*}
b_t =
\delta^{m \left\lceil \frac{t}{\sum_{k=0}^{m} K_{k} E_{k}} \right\rceil} b_0,
\end{align*}
which implies that $b_j = \delta^j b_0$ ($j \in S_j$) and 
\begin{align*}
(b_0, b_1, \cdots, b_M) = (
\underbrace{b_0, b_0, \cdots, b_0}_{K_0 E_0},
\underbrace{\delta b_0, \delta b_0, \cdots, \delta b_0}_{K_1 E_1}, \cdots,
\underbrace{\delta^M b_0, \delta^M b_0, \cdots, \delta^M b_0}_{K_M E_M}
),
\end{align*}
where $T = \sum_{m=0}^M K_m E_m$. Define $K_{\max} := \max \{ K_m \colon m \in [0:M]\}$ and $E_{\max} := \max \{ E_m \colon m \in [0:M]\}$. Then, we have 
\begin{align*}
\sum_{t=0}^{T-1} \sum_{i=0}^t \frac{\beta^{2(t-i)}}{b_i}
&= 
\sum_{t=0}^{T-1} \sum_{i=0}^t \beta^{2(t-i)} \frac{K_i E_i}{\delta^i b_0}
\leq 
\frac{K_{\max} E_{\max}}{b_0}
\sum_{t=0}^{T-1} \sum_{i=0}^t \frac{\beta^{2(t-i)}}{\delta^i}\\
&=
\frac{K_{\max} E_{\max}}{b_0}
\sum_{t=0}^{T-1} \beta^{2t}\sum_{i=0}^t \frac{1}{(\beta^2 \delta)^i}
= 
\frac{K_{\max} E_{\max}}{b_0}
\sum_{t=0}^{T-1} \beta^{2t}
\frac{1 - (\frac{1}{\beta^2 \delta})^{t+1}}{1 - \frac{1}{\beta^2 \delta}}\\
&= 
\frac{K_{\max} E_{\max}}{b_0}
\sum_{t=0}^{T-1} \beta^{2t}
\frac{1 - (\frac{1}{\beta^2 \delta})^{t+1}}{1 - \frac{1}{\beta^2 \delta}}
= 
\frac{K_{\max} E_{\max} \delta}{b_0 (\beta^2 \delta -1)}
\sum_{t=0}^{T-1} \left\{ \beta^{2(t+1)} - \frac{1}{\delta^{t+1}} \right\},
\end{align*}
which implies that
\begin{align*}
\sum_{t=0}^{T-1} \sum_{i=0}^t \frac{\beta^{2(t-i)}}{b_i}
\leq
\frac{K_{\max} E_{\max} \delta}{b_0 (\beta^2 \delta -1)}
 \left\{ \frac{\beta^2 (1 - \beta^{2T})}{1 - \beta^2} 
- \frac{1 - (\frac{1}{\delta})^T}{\delta - 1} \right\}
\leq 
\frac{K_{\max} E_{\max} \delta}{b_0 (\beta^2 \delta -1)}
\left( \frac{\beta^2}{1 - \beta^2} 
- \frac{1}{\delta - 1} \right).
\end{align*}
This completes the proof.
\end{proof}

\begin{proof}[Proof of Theorem \ref{thm:2}] NSHB with $\eta = \frac{\alpha}{1-\beta}$ coincides with SHB defined by \eqref{shb} (Section \ref{sec:2.2}). Hence, Theorem \ref{thm:1} leads to Theorem \ref{thm:2}.
\end{proof}

\subsection{Proofs of Theorems \ref{thm:1_1} and \ref{thm:2_1}}
\label{appendix:thm_1_1}
Lemma \ref{lem:4} and the proof of Lemma \ref{prop:4} with $\rho = \frac{1-\beta}{2 L \eta}$ immediately yield the following lemma. Hence, we will omit its proof.

\begin{mylemma}\label{prop:4_1}
Suppose that Assumption \ref{assum:1} holds and $(\bm{\theta}_t)$ is the sequence generated by Algorithm \ref{algo:1}. Let $(\bm{z}_t)$ be the sequence defined as in Lemma \ref{prop:4}. Then, for all $t \in \{0\} \cup \mathbb{N}$,
\begin{align*}
\mathbb{E} [f(\boldsymbol z_{t+1})] 
&\le \mathbb{E} [f(\boldsymbol z_t ) ] 
+ \eta \left\{ L \left( \frac{1+\beta^2}{1-\beta}  + \frac{1}{2} \right) \eta - 1 \right\} \mathbb{E} [\| \nabla f(\bm{\theta}_t) \|^2 ]\\
&\quad + \frac{L \sigma^2 \eta^2}{2} \left( \frac{\beta^2}{1+\beta} + 1 \right)  
+ \frac{(1 - \beta^t)^2}{1 - \beta} L\eta^2 \mathbb{E}\left[\left\Vert \frac{1-\beta}{1-\beta^t}\sum_{i=0}^t \beta^{t-i} \nabla f(\bm{\theta}_i) - \nabla f(\bm{\theta}_t) \right\rVert^2\right],
\end{align*}
where $L := \frac{1}{n} \sum_{i\in [n]} L_i$ is the Lipschitz constant of $\nabla f$ and we assume that $\sum_{i=0}^{-1} a_i := 0$ for some $a_i \in \mathbb{R}$.
\end{mylemma}

Moreover, Lemma \ref{prop:4_1} and the proof of Lemma \ref{prop:5} with $\rho = \frac{1-\beta}{2 L \eta}$ yield the following lemma, whose proof we will also omit. 

\begin{mylemma}\label{prop:5_1}
Suppose that Assumption \ref{assum:1} holds and $(\bm{\theta}_t)$ is the sequence generated by Algorithm \ref{algo:1} with $\eta > 0$ satisfying
\begin{align*}
\eta \le \frac{1-\beta}{2\sqrt{2}\sqrt{\beta+\beta^2}L}.
\end{align*}
Let $(\bm{z}_t)$ be the sequence defined as in Lemma \ref{prop:4} and let $L_t \in \mathbb{R}$ be defined as in Lemma \ref{prop:5}, where $(c_i) \subset \mathbb{R}_{++}$ is defined by

\resizebox{\textwidth}{!}{$
c_1 = \dfrac{(\beta+\beta^2) L^3 \eta^2}{(1-\beta) \{ (1-\beta)^2 -4 (\beta+\beta^2)L^2 \eta^2 \}} \text{ and } c_{i+1} = c_i - \left(4 c_1 \eta^2 + \frac{L\eta^2}{1-\beta}\right)\beta^i\left(i+\frac{\beta}{1-\beta}\right)L^2 \quad (i \in [t-1]).$}
Then, for all $t \in \{0\} \cup \mathbb{N}$,
\begin{align*}
\mathbb{E} [L_{t+1}-L_t ]
&\leq  
\eta \left\{ \frac{(3 - \beta + \beta^2) L \eta }{2(1-\beta)} - 1 + 4 c_1 \eta \right\} \mathbb{E} [\| \nabla f(\bm{\theta}_t) \|^2 ] \\
&\quad + \left\{ \left( \frac{\beta^2}{1 + \beta} + 1  \right) \frac{L}{2} 
+ \frac{2 c_1 (1 - \beta)}{1 + \beta} \right \}\frac{\eta^2 \sigma^2}{b}. 
\end{align*}
\end{mylemma}

Lemma \ref{prop:5_1} and the proof of Lemma \ref{prop:6} with $\rho = \frac{1-\beta}{2 L \eta}$ lead to the following.

\begin{mylemma}\label{prop:6_1}
Suppose that Assumption \ref{assum:1} holds and $(\bm{\theta}_t)$ is the sequence generated by Algorithm \ref{algo:1} with $\eta > 0$ satisfying
\begin{align*}
\eta 
\leq \max \left\{ \frac{1-\beta}{2\sqrt{2} \sqrt{\beta+\beta^2} L},
\frac{1-\beta}{(5 - \beta + 2 \beta^2)L} \right\}.
\end{align*}
Then, for all $T \geq 1$,
\begin{align*}
\frac{1}{T} \sum_{t=0}^{T-1} \mathbb{E} [\| \nabla f(\bm{\theta}_t) \|^2 ]
\leq
\frac{2 (f(\bm{\theta}_0) - f^\star)}{\eta T}
+
\frac{L \eta \sigma^2}{b}\left\{ \frac{\beta + 3 \beta^2}{2(1 + \beta)} + 1  \right\}.
\end{align*}
\end{mylemma}

\begin{proof}[Proof of Theorems \ref{thm:1_1} and \ref{thm:2_1}] Lemma \ref{prop:6_1} leads to the assertion in Theorem \ref{thm:1_1}. NSHB with $\eta = \frac{\alpha}{1-\beta}$ coincides with SHB defined by \eqref{shb} (Section \ref{sec:2.2}). Hence, Theorem \ref{thm:1_1} leads to Theorem \ref{thm:2_1}.
\end{proof}

\begin{mytheorem}
[Upper bound of $\min_{t} \mathbb{E}\| \nabla f(\bm{\theta}_t) \|$ of mini-batch SHB with Constant BS]
\label{thm:2_1}
Suppose that Assumption \ref{assum:1} holds and consider the sequence $(\bm{\theta}_t)$ generated by \eqref{shb} with a momentum weight $\beta \in (0,1)$, a constant learning rate $\alpha > 0$ such that 
\begin{align*}
\alpha 
\leq \frac{(1-\beta)^2}{2\sqrt{2} \sqrt{\beta+\beta^2} L},
\end{align*}
and Constant BS defined by \eqref{constant_bs}. Then, for all $T \geq 1$,
\begin{align*}
&\min_{t \in [0:T-1]} \mathbb{E} [\| \nabla f(\bm{\theta}_t) \|^2 ] \leq
\frac{2 (1-\beta) (f(\bm{\theta}_0) - f^\star)}{\alpha T}
+
\frac{L\alpha \sigma^2}{(1-\beta)b}\left\{  
\frac{3 \beta^2 + \beta}{2(1 + \beta)} + 1
\right\},
\end{align*} 
that is,
\begin{align*}
\min_{t \in [0:T-1]} \mathbb{E} [\| \nabla f(\bm{\theta}_t) \|]
= O \left(\sqrt{ \frac{1}{T} + \frac{\sigma^2}{b}  } \right).
\end{align*}
\end{mytheorem}

\begin{mytheorem}
[Convergence of mini-batch SHB with Exponential Growth BS]
\label{thm:2}
Suppose that Assumption \ref{assum:1} holds and consider the sequence $(\bm{\theta}_t)$ generated by \eqref{shb} with a momentum weight $\beta \in (0,1)$, a constant learning rate $\alpha > 0$ such that 
\begin{align*}
\alpha 
\leq \max \left\{ \frac{(1-\beta)^2}{2\sqrt{2}\sqrt{\beta+\beta^2}L},
\frac{(1-\beta)^3}{(5\beta^2 - 6\beta + 5)L} \right\},
\end{align*}
and Exponential Growth BS defined by \eqref{exponential_bs} with $\delta > 1$ and $\beta^2 \delta > 1$. Then, for all $T \geq 1$,
\begin{align*}
&\min_{t \in [0:T-1]} \mathbb{E} [\| \nabla f(\bm{\theta}_t) \|^2 ]
\leq
\frac{2 (1 - \beta) (f(\bm{\theta}_0) - f^\star)}{\alpha T} 
+ \frac{2 L \alpha \sigma^2 K_{\max} E_{\max} \delta}{(1-\beta)(\beta^2 \delta -1) b_0 T} \left( \frac{\beta^2}{1 - \beta^2} 
- \frac{1}{\delta - 1} \right),
\end{align*} 
that is,
\begin{align*}
\min_{t \in [0:T-1]} \mathbb{E} [\| \nabla f(\bm{\theta}_t) \|]
= O \left(\frac{1}{\sqrt{T}} \right).
\end{align*}
\end{mytheorem}

\subsection{Training ResNet-18 on Tiny ImageNet}\label{appendix:A.4}

\begin{figure}[!htbp]
    \centering
    \begin{minipage}{0.4\textwidth}
        \centering
        \includegraphics[width=0.75\textwidth]{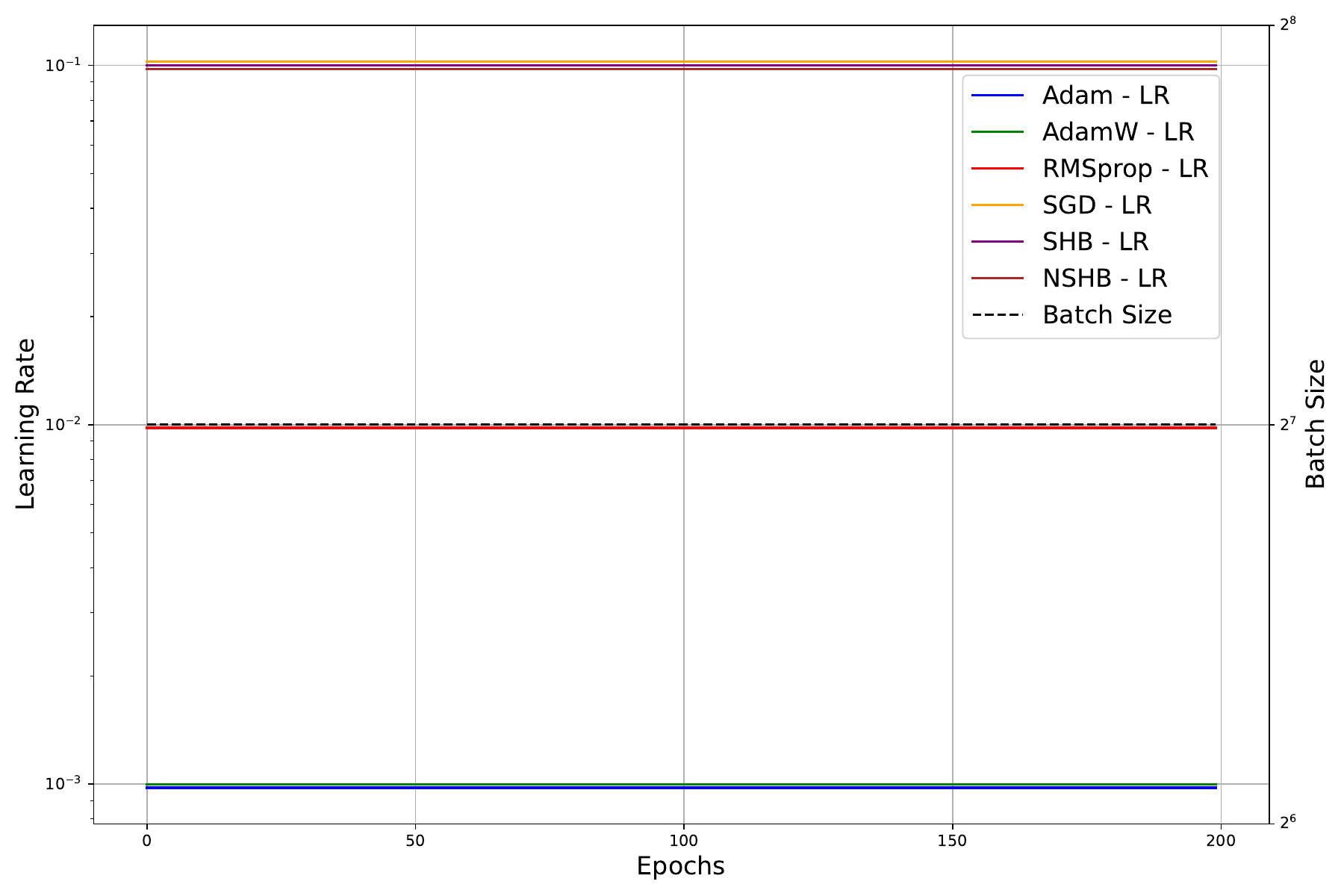} 
        \subfigure{Learning rate and batch size schedules}
    \end{minipage} \hfill
    \begin{minipage}{0.4\textwidth}
        \centering
        \includegraphics[width=0.75\textwidth]{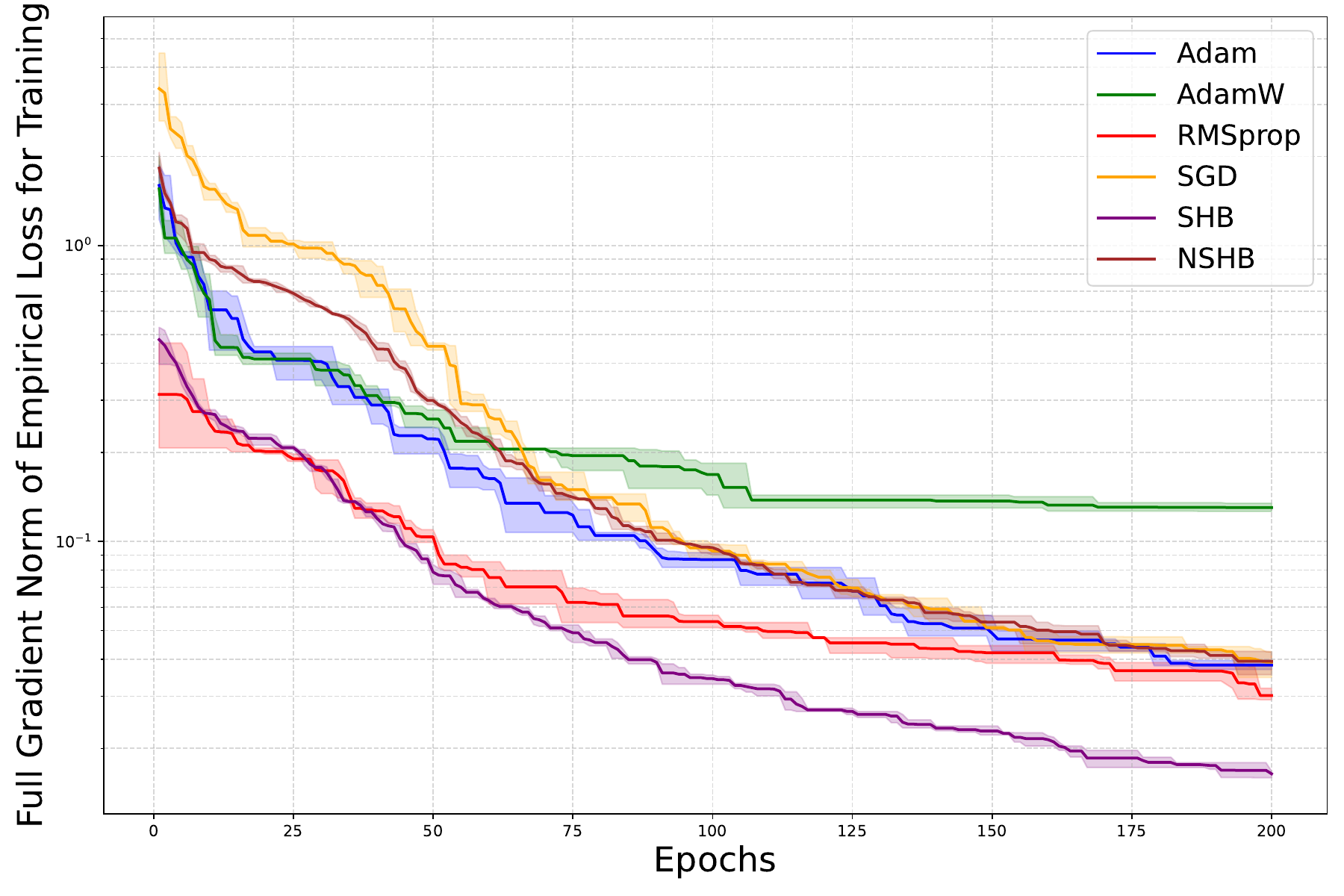} 
        \subfigure{Full gradient norm versus epochs}
    \end{minipage} \\[0.2cm]
    \begin{minipage}{0.4\textwidth}
        \centering
        \includegraphics[width=0.75\textwidth]{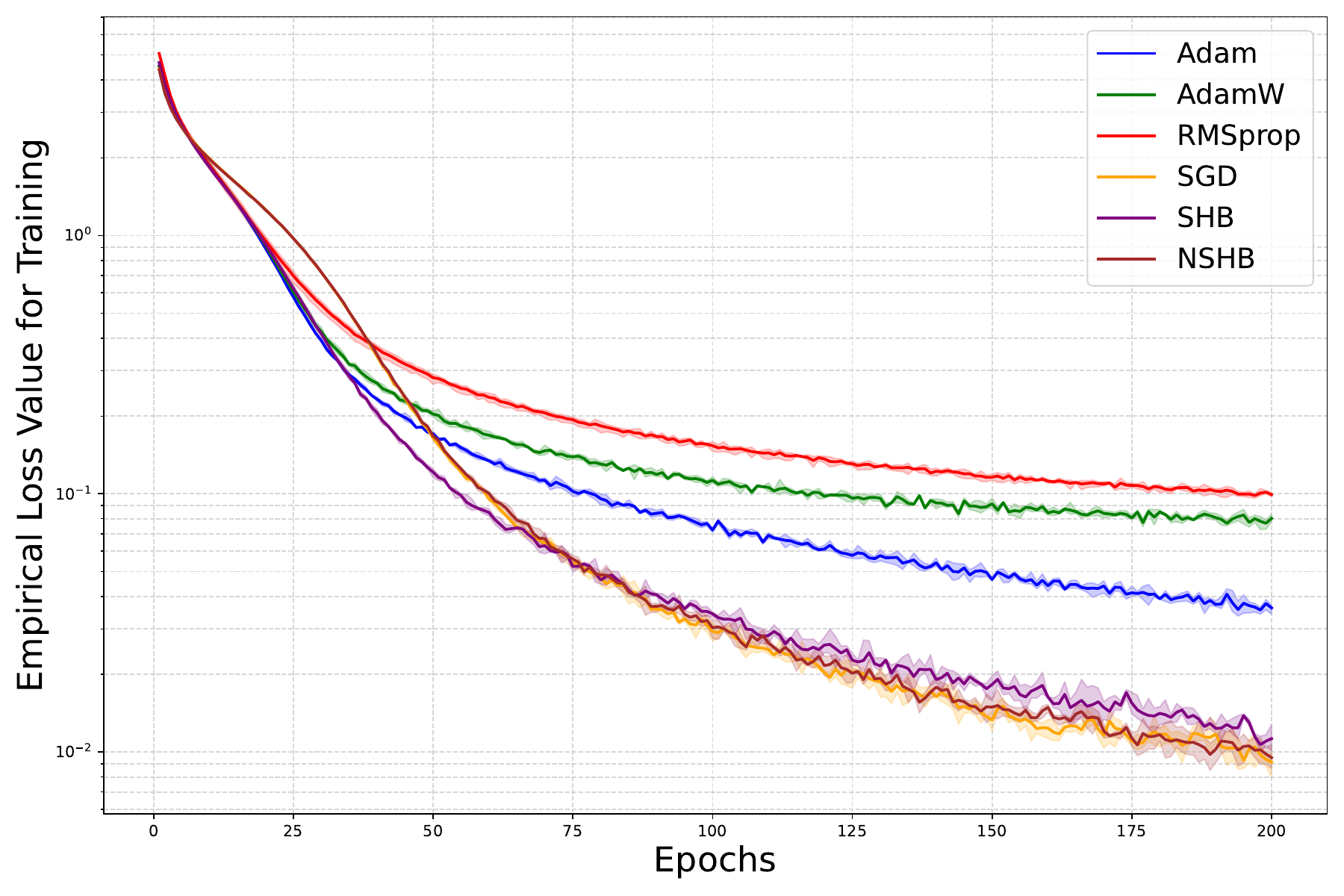} 
        \subfigure{Empirical loss versus epochs}
    \end{minipage} \hfill
    \begin{minipage}{0.4\textwidth}
        \centering
        \includegraphics[width=0.75\textwidth]{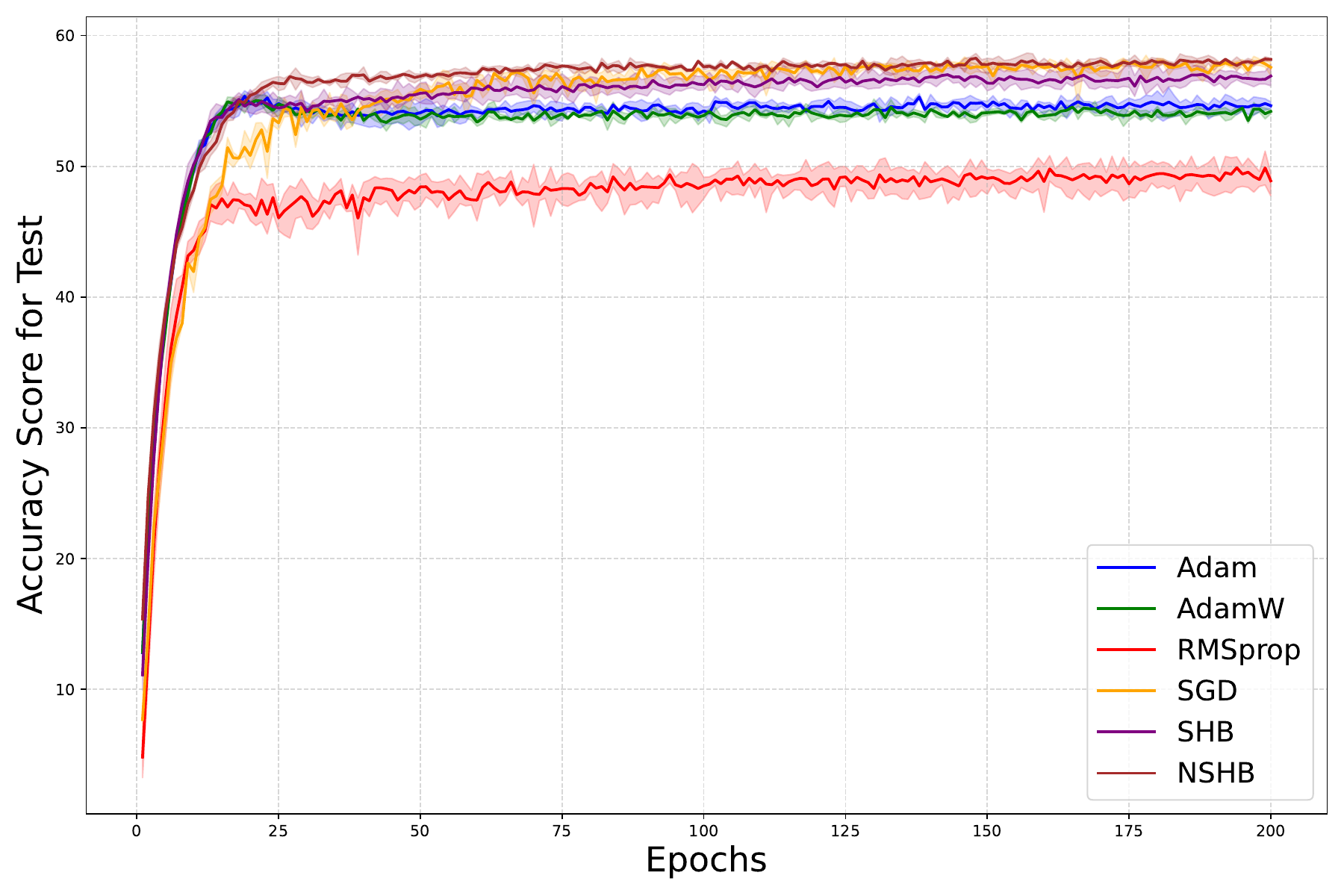} 
        \subfigure{Test accuracy score versus epochs}
    \end{minipage}
    \caption{(a) Schedules for each optimizer with constant learning rates and constant batch size, (b) Full gradient norm of empirical loss for training, (c) Empirical loss value for training, and (d) Accuracy score for test to train ResNet-18 on the {Tiny ImageNet} dataset.}
\end{figure}

\begin{figure}[!htbp]
    \centering
    \begin{minipage}{0.4\textwidth}
        \centering
        \includegraphics[width=0.75\textwidth]{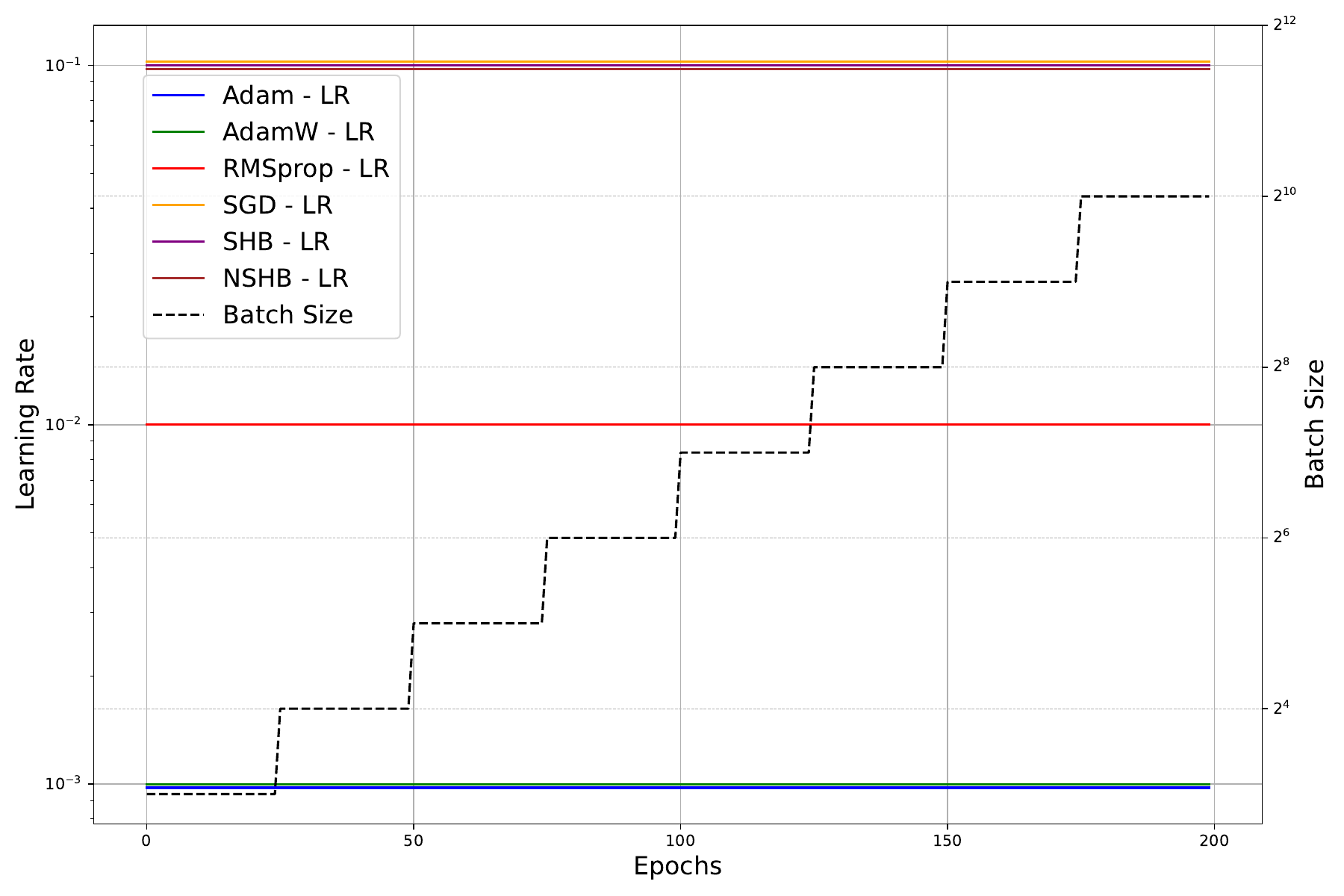}  
        \subfigure{Learning rate and batch size schedules}
    \end{minipage} \hfill
    \begin{minipage}{0.4\textwidth}
        \centering
        \includegraphics[width=0.75\textwidth]{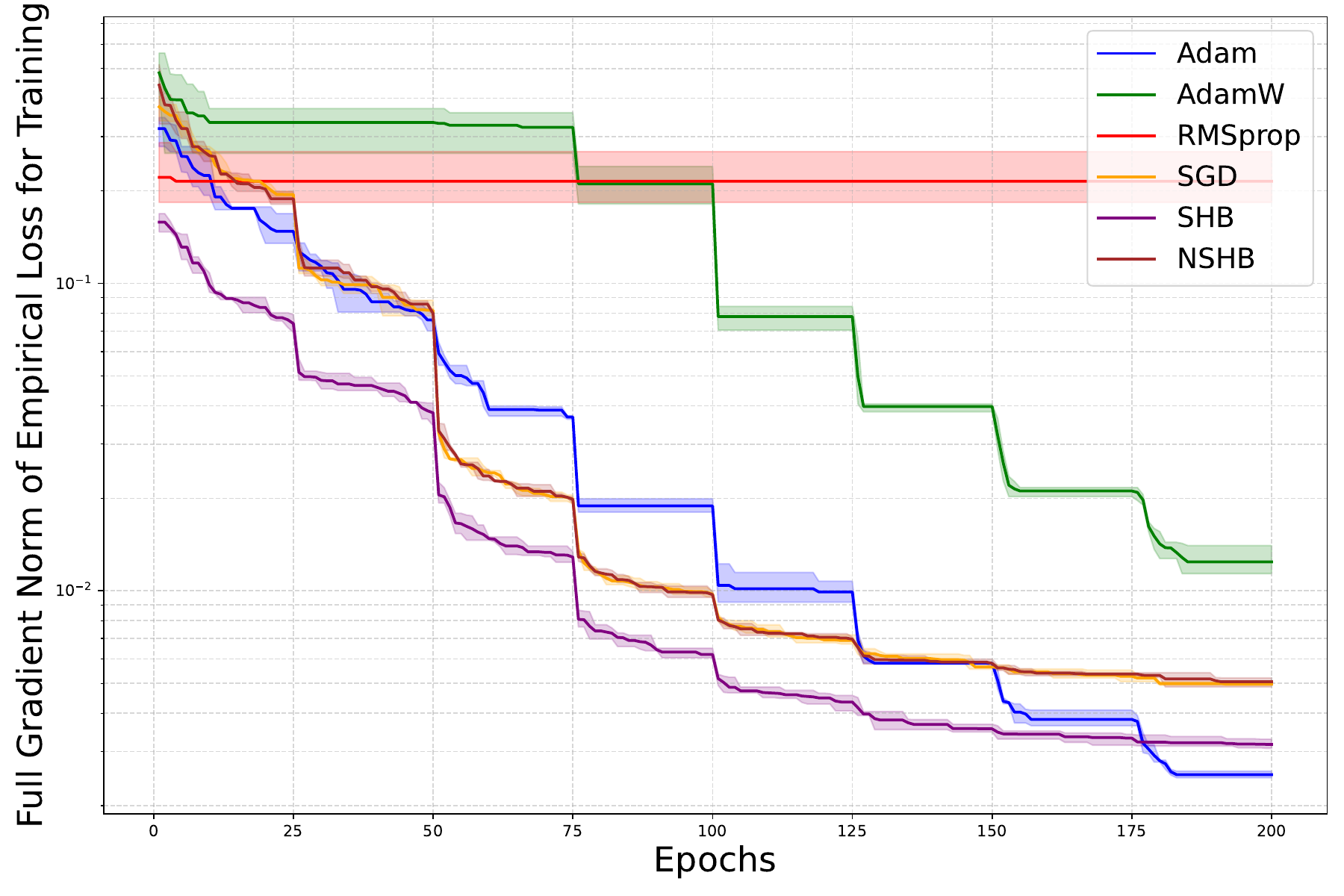} 
        \subfigure{Full gradient norm versus epochs}
    \end{minipage}
    \begin{minipage}{0.4\textwidth}
        \centering
        \includegraphics[width=0.75\textwidth]{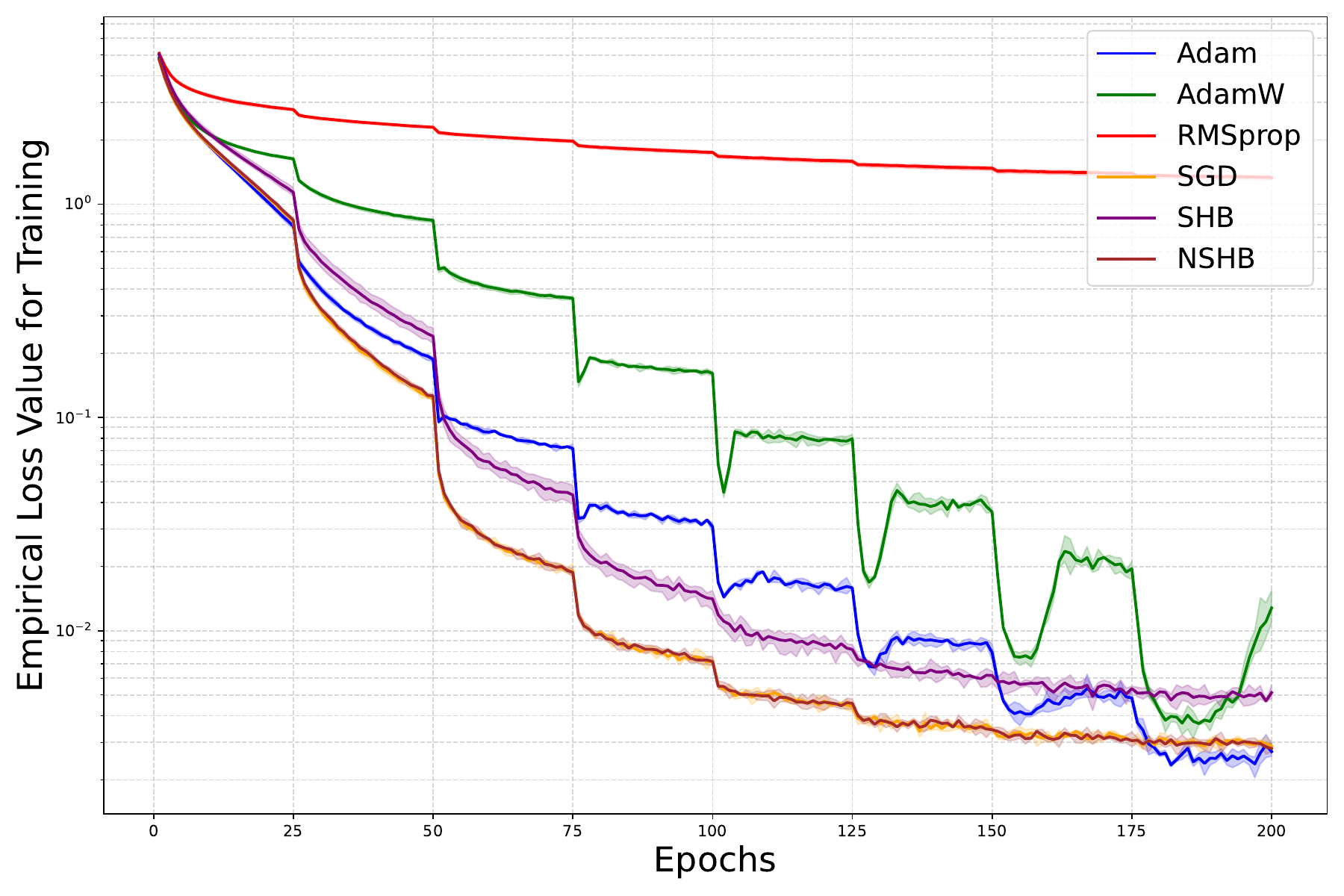} 
        \subfigure{Empirical loss versus epochs}
    \end{minipage} \hfill
    \begin{minipage}{0.4\textwidth}
        \centering
        \includegraphics[width=0.75\textwidth]{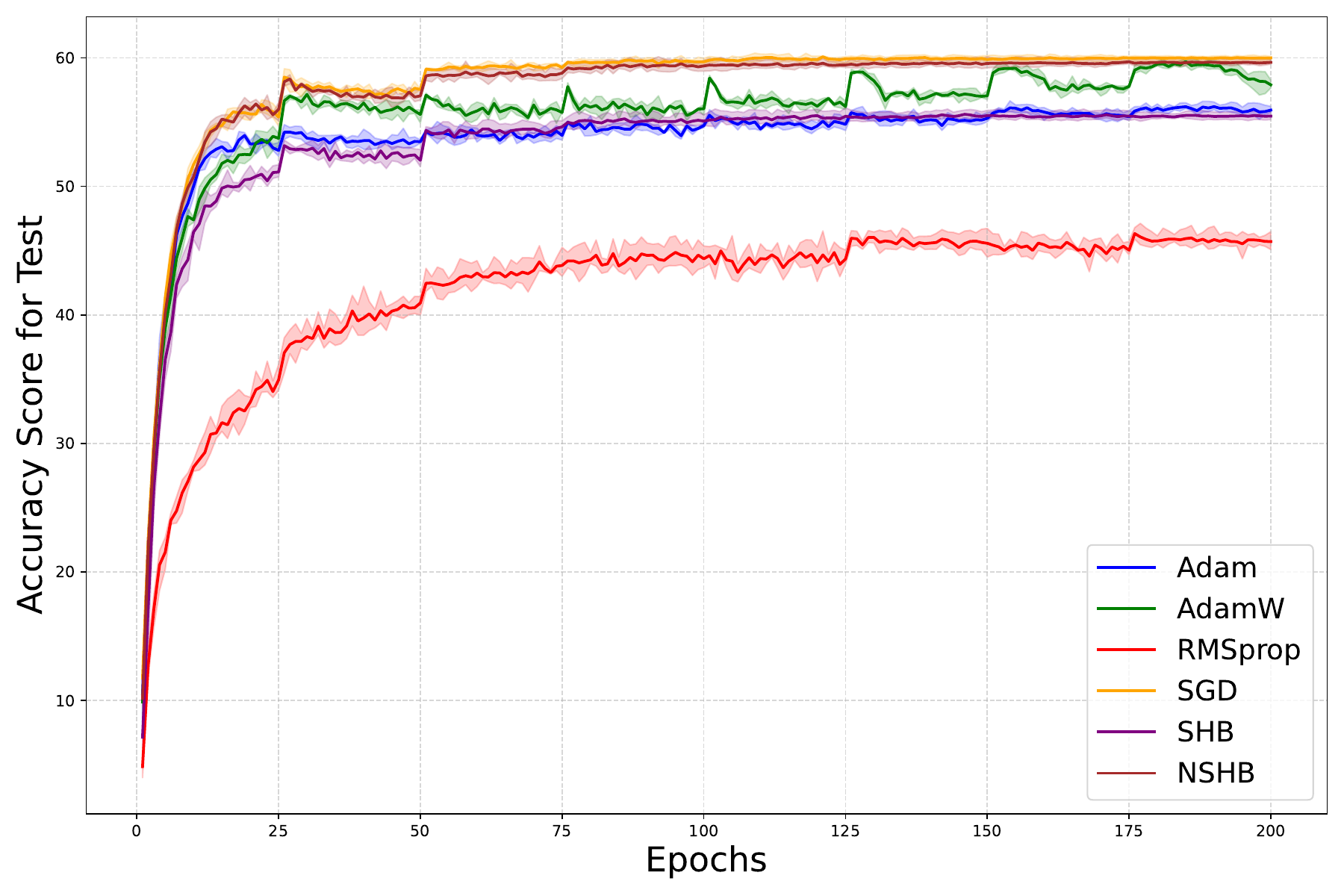} 
        \subfigure{Test accuracy score versus epochs}
    \end{minipage}
    \caption{(a) Schedules for each optimizer with constant learning rates and batch size doubling every 25 epochs, (b) Full gradient norm of empirical loss for training, (c) Empirical loss value for training, and (d) Accuracy score for test to train ResNet-18 on the {Tiny ImageNet} dataset.}
\end{figure}

\begin{figure}[!htbp]
    \centering
    \begin{minipage}{0.4\textwidth}
        \centering
        \includegraphics[width=0.75\textwidth]{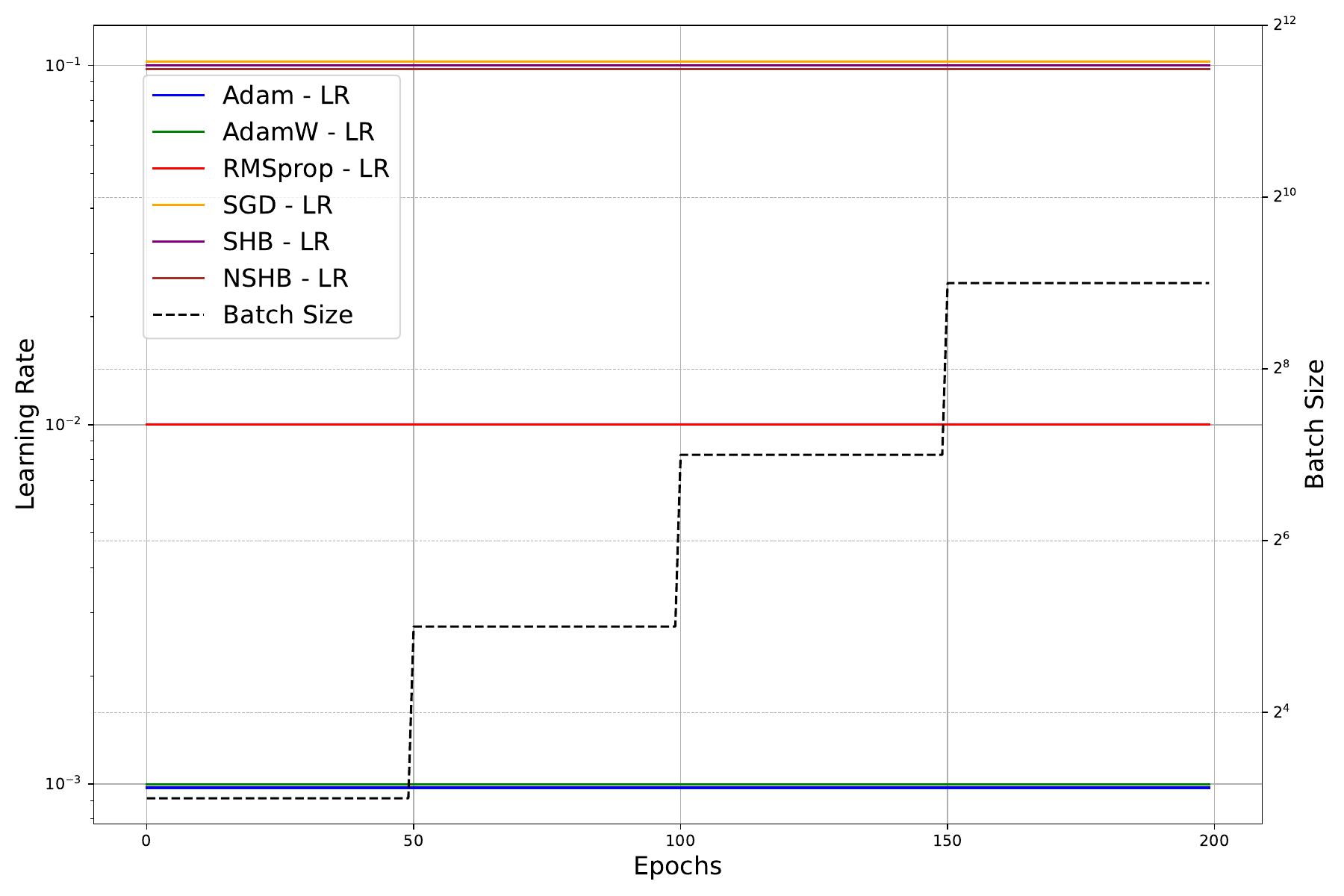} 
        \subfigure{Learning rate and batch size schedules}
    \end{minipage} \hfill
    \begin{minipage}{0.4\textwidth}
        \centering
        \includegraphics[width=0.75\textwidth]{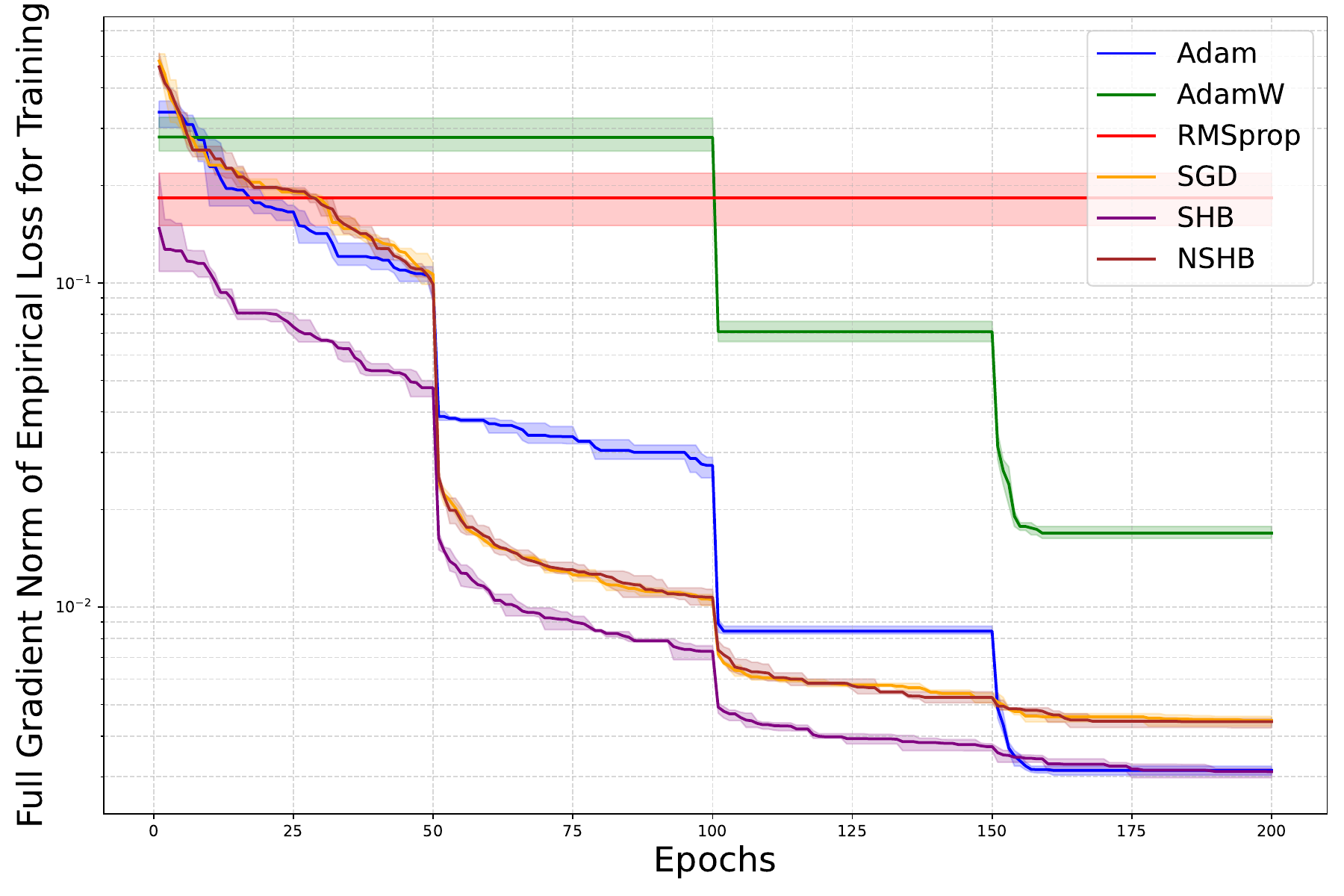} 
        \subfigure{Full gradient norm versus epochs}
    \end{minipage} 
    \begin{minipage}{0.4\textwidth}
        \centering
        \includegraphics[width=0.75\textwidth]{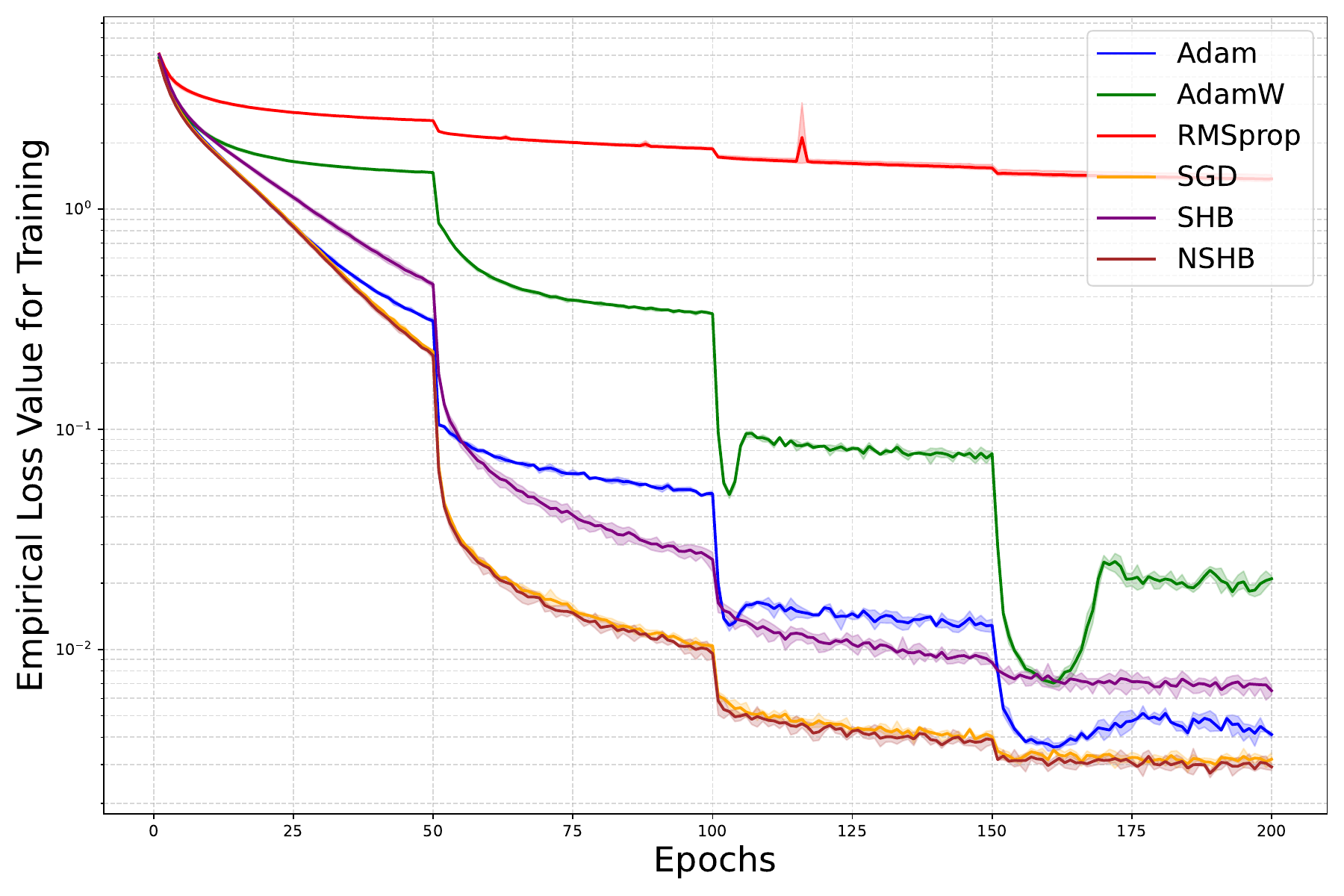} 
        \subfigure{Empirical loss versus epochs}
    \end{minipage} \hfill
    \begin{minipage}{0.4\textwidth}
        \centering
        \includegraphics[width=0.75\textwidth]{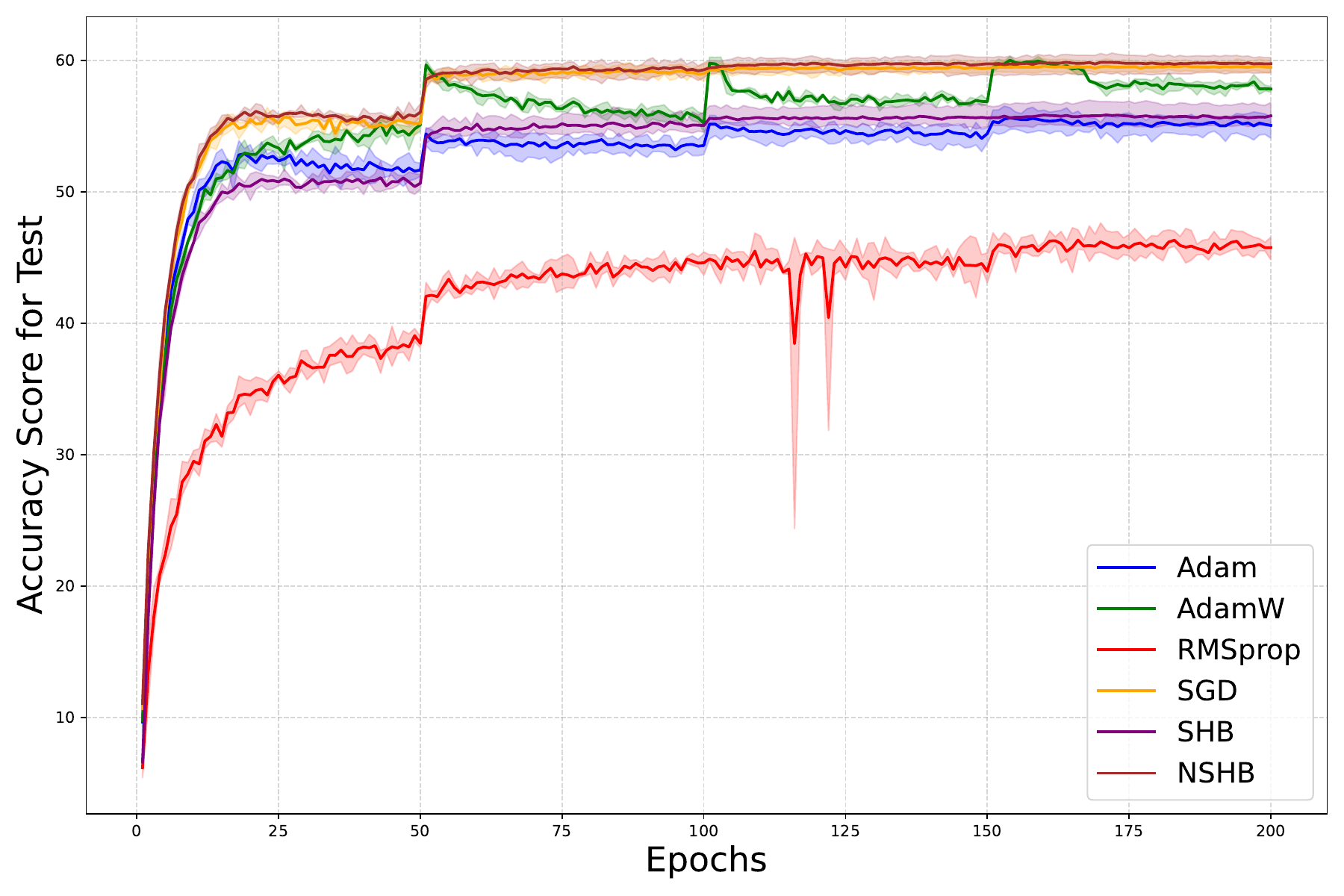} 
        \subfigure{Test accuracy score versus epochs}
    \end{minipage}
    \caption{(a) Schedules for each optimizer with constant learning rates and batch size quadrupling every 50 epochs, (b) Full gradient norm of empirical loss for training, (c) Empirical loss value for training, and (d) Accuracy score for test to train ResNet-18 on the {Tiny ImageNet} dataset.}
\end{figure}

\clearpage
\subsection{Visualization of SFO complexity (Tables \ref{table:2}-\ref{table:4})}
\label{appendix:A.5}

\begin{figure}[!htbp]
    \centering
    \subfigure[SFO complexity (Table \ref{table:2})]
{
        \includegraphics[width=0.45\textwidth]{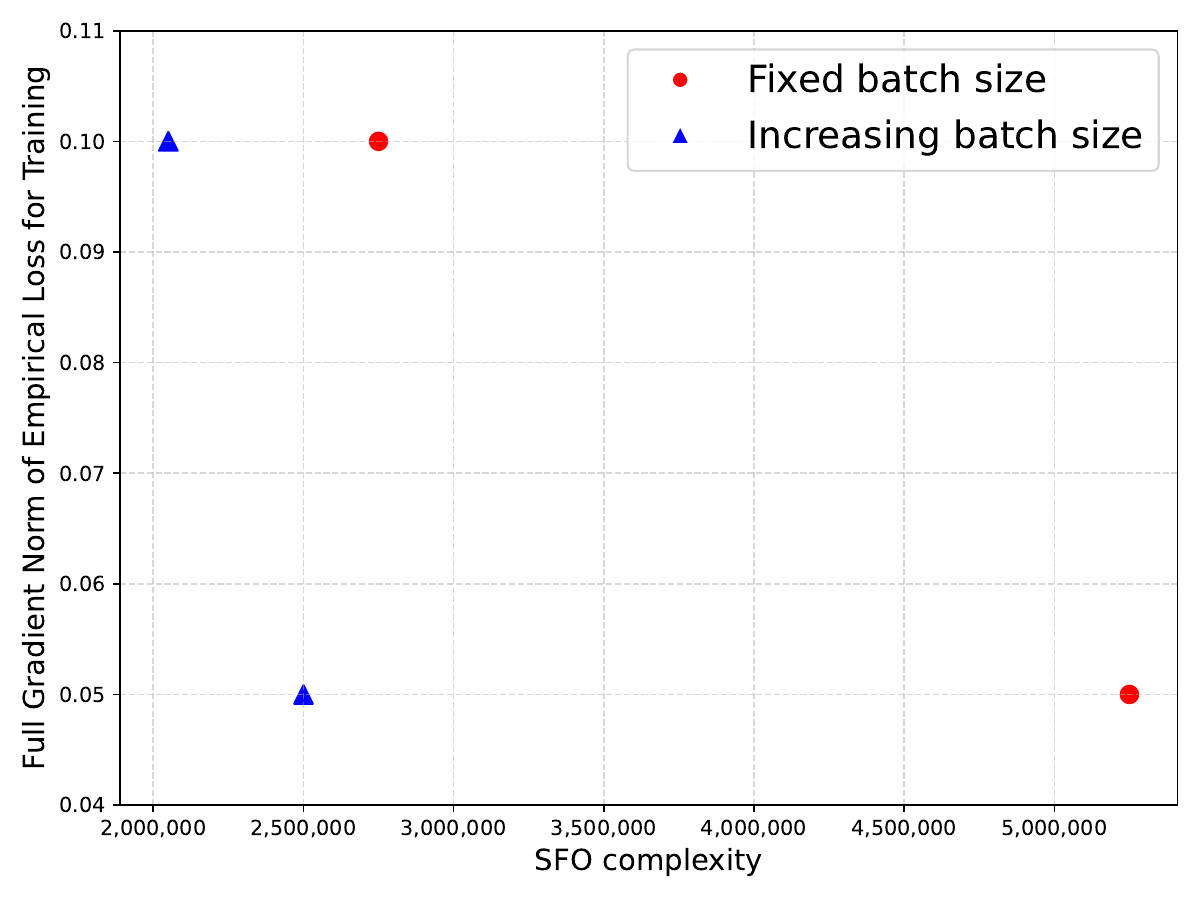}
    } \\
    \subfigure[SFO complexity (Table \ref{table:3})]
{
        \includegraphics[width=0.45\textwidth]{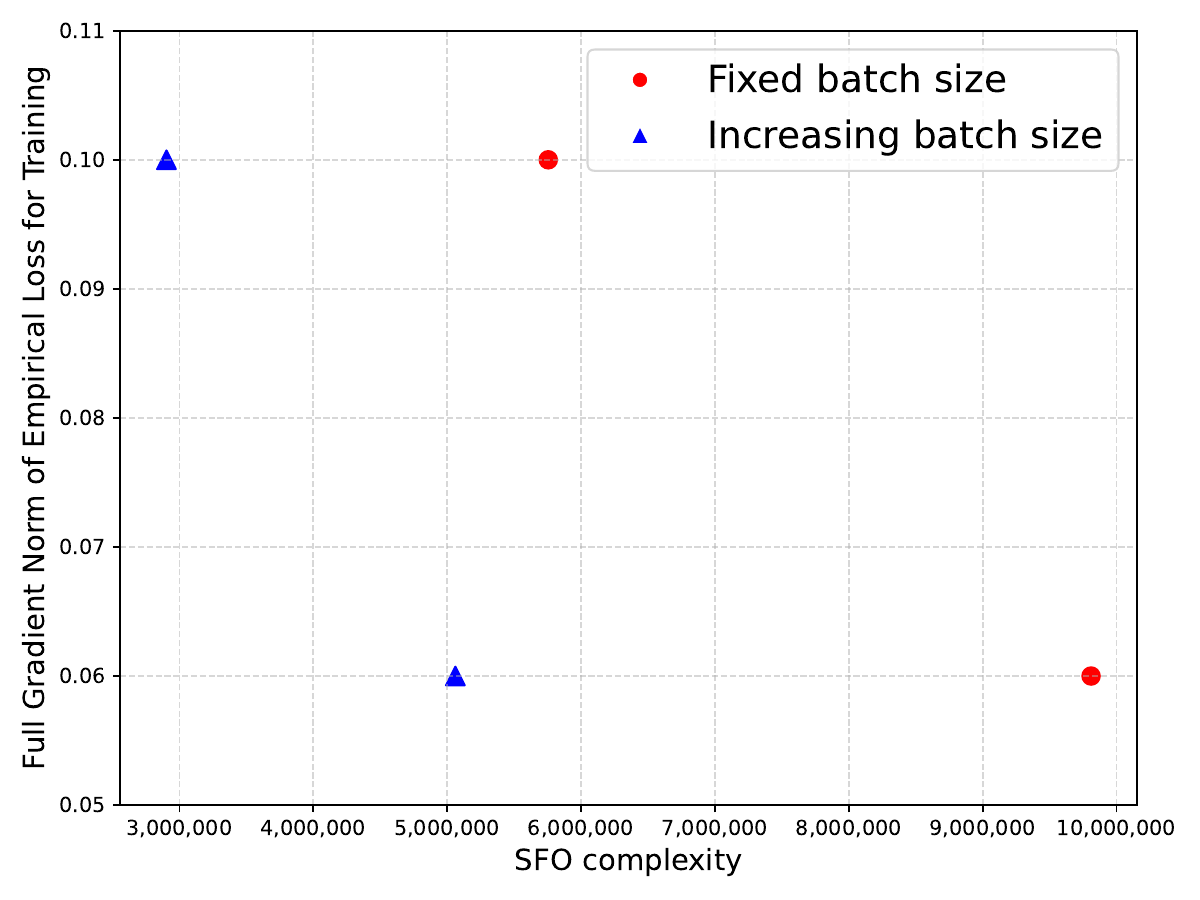}
    } \\
    \subfigure[SFO complexity (Table \ref{table:4})]
{
        \includegraphics[width=0.45\textwidth]{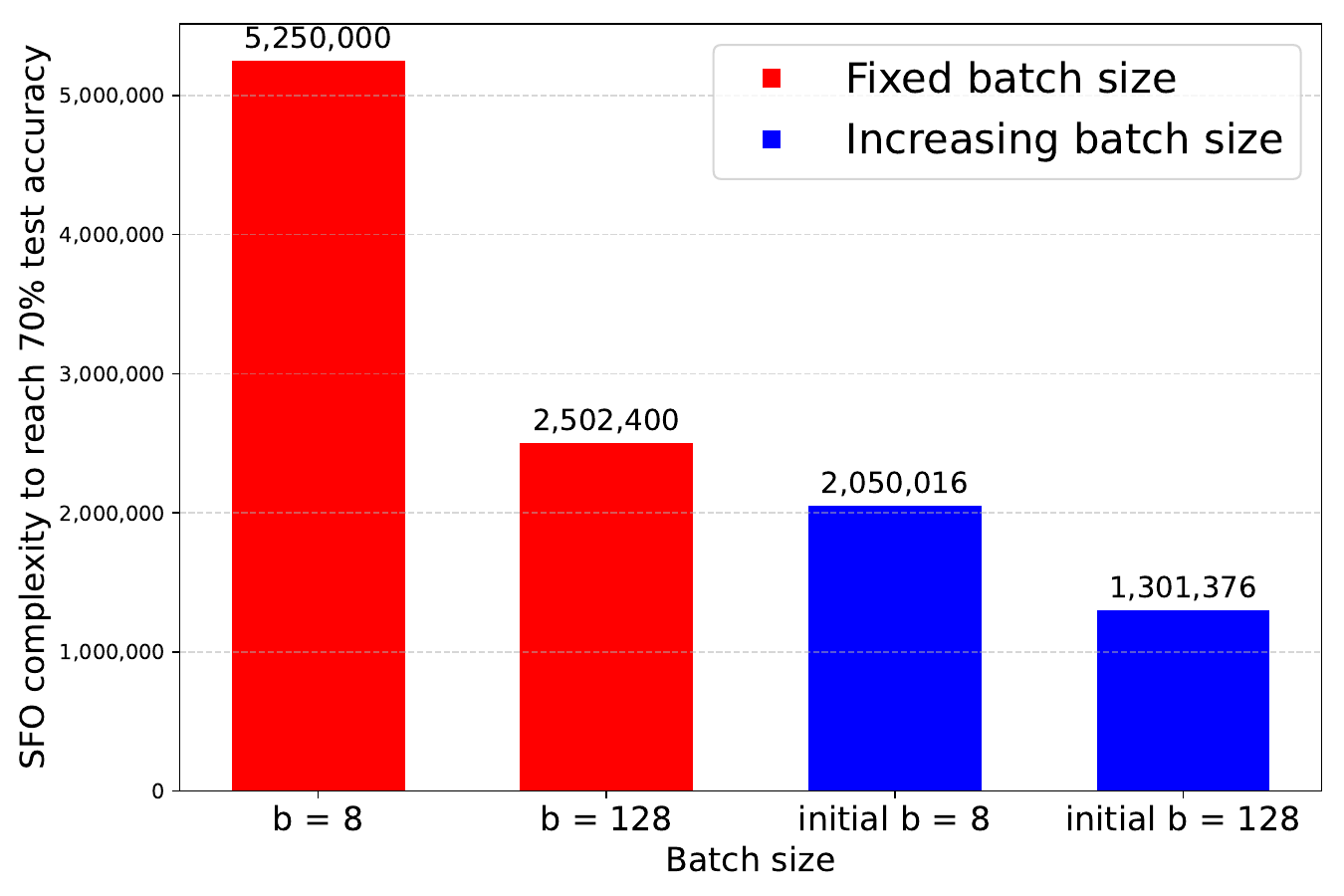}
    }
    \caption{(a) SFO complexity to reach gradient norm thresholds with $b=8$ (corresponding to Table \ref{table:2}), (b) SFO complexity to reach gradient norm thresholds with $b=128$ (corresponding to Table \ref{table:3}), (c) SFO complexity to reach 70\% test accuracy (corresponding to Table \ref{table:4}). Results are obtained using ResNet-18 on the CIFAR-100 dataset.}
\end{figure}

\subsection{Results on wall-clock time}
\label{appendix:A.6}
Since wall-clock time is an important practical metric of optimization efficiency, we additionally measured it under the same experimental settings as in Table \ref{table:2}, using CIFAR-100 with ResNet-18 and NSHB, with the maximum batch size set to 1024.

For the fixed batch size $b=8$ and the increasing batch size starting from $b=8$ and doubling every 20 epochs, the results are as follows:

\begin{table}[h]
\normalsize
\centering
\caption{Wall-clock time to reach gradient norm thresholds ($b=8$)}
\label{table:5}
\resizebox{\textwidth}{!}{%
\begin{tabular}{lcc}
\toprule
Method & Time to reach $\|\nabla f (\bm{\theta}_t) \| < 0.1$ (h:mm:ss) & Time to reach $\|\nabla f (\bm{\theta}_t) \| < 0.05$ (h:mm:ss) \\
\midrule
Fixed batch size ($b=8$) & 2:22:58 & 3:40:05 \\
Increasing batch size (initial $b=8$) & 0:49:50 & 0:54:03 \\
\bottomrule
\end{tabular}
}
\end{table}

All experiments were conducted on the same computing server, using identical hardware and software environments to ensure a fair comparison. These improvements in wall-clock time are consistent with the reductions in SFO complexity reported in the main paper.

\subsection{Assessing the validity of experimental learning rates in light of Theorem \ref{thm:1}}
\label{appendix:A.7}
The learning rate condition in Theorem \ref{thm:1} can be expressed as
\begin{align*}
    \eta \leq \max \left\{ 
    \frac{1-\beta}{2\sqrt{2}\sqrt{\beta+\beta^2 L}}, \,
    \frac{(1-\beta)^2}{(5\beta^2 - 6\beta + 5)L}
    \right\}.
\end{align*}
Substituting $\beta = 0.9$ into the first term yields approximately
\begin{align*}
    \eta \leq \frac{0.0270}{L}.
\end{align*}
This upper bound depends on the smoothness constant $L$, which is typically unknown and difficult to estimate accurately in practice. When $L$ is large, the bound becomes more restrictive, requiring a smaller $\eta$. Conversely, when $L$ is small, the condition allows for a relatively larger $\eta$, under which our empirical choice of $\eta = 0.1$ still appears to fall within a theoretically reasonable range. Consequently, the learning rate setting adopted in our experiments can be regarded as theoretically justified across a wide range of possible values for $L$, and it is consistent with common empirical practices. Although the exact satisfaction of the theoretical bound cannot be rigorously verified, the consistent convergence observed in our experiments suggests that our setting lies within a reasonable range.



\end{document}